%% file: main.tex
\documentclass[english, 11pt]{article}
\usepackage[margin=1in]{geometry}
\usepackage{multirow}
\usepackage{amsthm}
\usepackage{amsmath}
\usepackage{amsfonts}
\usepackage{amssymb}
\usepackage{bm}
\usepackage{natbib}
\usepackage{authblk}
\usepackage{color}
\usepackage[dvipsnames]{xcolor}
\usepackage{caption}
\usepackage{subcaption}
\usepackage{wrapfig}
\usepackage{graphicx}
\usepackage{array}
\usepackage{setspace}
\usepackage{algorithm}
\usepackage{algpseudocode}
\usepackage[bookmarks=false, colorlinks=true, citecolor=MidnightBlue, linkcolor=MidnightBlue, urlcolor=MidnightBlue]{hyperref}
\usepackage{setspace}
\usepackage{bbm}
\usepackage{thm-restate}
\usepackage{comment}
\algnewcommand{\LineComment}[1]{\State \(\triangleright\) #1}

\usepackage{xr}

\makeatletter

\usepackage{verbatim}   

\usepackage{microtype}
\usepackage{mathpazo} 
\usepackage{courier} 
\normalfont
\usepackage[T1]{fontenc}
\renewcommand{\leq}{\leqslant}
\renewcommand{\le}{\leqslant}
\renewcommand{\geq}{\geqslant}
\renewcommand{\ge}{\geqslant}

\newcommand{\KL}{{\rm KL}}

  \theoremstyle{plain}
  \newtheorem{Theorem}{\protect\theoremname}
  \theoremstyle{plain}
  \newtheorem*{Theorem*}{\protect\theoremname}
  \theoremstyle{plain}
  \newtheorem{proposition}{\protect\propositionname}
  \theoremstyle{plain}
  \newtheorem*{prop*}{\protect\propositionname}
  \theoremstyle{plain}
  \newtheorem{lemma}{\protect\lemmaname}
   \theoremstyle{plain}
  \newtheorem*{lemma*}{\protect\lemmaname}  
  \theoremstyle{plain}
  \newtheorem{definition}{\protect\definitionname}
  \theoremstyle{plain}
  \newtheorem{corollary}{\protect\corollaryname}
  \theoremstyle{plain}
  \newtheorem{example}{\protect\examplename}
 \theoremstyle{plain}
\newtheorem{remark}{Remark}
\theoremstyle{plain}

\theoremstyle{plain}
\newtheorem{assumption}{\protect\assumptionname}
\theoremstyle{plain}

 \theoremstyle{plain}
\newtheorem*{model*}{Model}

\makeatother

\usepackage{complexity}     

\newcommand{\ut}{\underline{\mathsf{t}}}

\newcommand{\ind}{\mathbb{1}}
\renewcommand{\E}{\mathbb{E}}

\newcommand{\Prob}{\mathbb{P}}

\newcommand{\Fc}{\mathcal{F}}

\usepackage{babel}
\providecommand{\assumptionname}{Assumption}
\providecommand{\definitionname}{Definition}
\providecommand{\lemmaname}{Lemma}
\providecommand{\propositionname}{Proposition}
\providecommand{\corollaryname}{Corollary}
\providecommand{\examplename}{Example}
\providecommand{\factname}{Fact}
\providecommand{\conditionname}{Condition}
\providecommand{\theoremname}{Theorem}

\DeclareMathOperator*{\argmax}{arg\,max}  
\DeclareMathOperator*{\argmin}{arg\,min}  

\newcommand{\Sto}{\xrightarrow{\mathbb{S}}}

\usepackage{enumitem}   

\usepackage{bm}
\newcommand{\thetabf}{\bm{\theta}}
\newcommand{\varthetabf}{\bm{\vartheta}}
\usepackage{cleveref}

\begin{document}

\title{Optimizing Adaptive Experiments: A Unified Approach to Regret Minimization and Best-Arm Identification} 
\author{Chao Qin and Daniel Russo}
\affil{Columbia University}
\pdfoutput=1
\maketitle
\onehalfspacing
\begin{abstract}
Practitioners conducting adaptive experiments often encounter two competing priorities: maximizing total welfare (or `reward') through effective treatment assignment and swiftly concluding experiments to implement population-wide treatments.  Current literature addresses these priorities separately, with regret minimization studies focusing on the former and best-arm identification research on the latter. This paper bridges this divide by proposing a unified model that simultaneously accounts for within-experiment performance and post-experiment outcomes. We provide a sharp theory of optimal performance in large populations that not only unifies canonical results in the literature but also uncovers novel insights. Our theory reveals that familiar algorithms, such as the recently proposed top-two Thompson sampling algorithm, can optimize a broad class of objectives if  a single scalar parameter is appropriately adjusted. In addition, we demonstrate that substantial reductions in experiment duration can often be achieved with minimal impact on both within-experiment and post-experiment regret.  
\end{abstract}

\section{Introduction}
  A vast body of research on the multi-armed bandit (MAB) problem studies effective adaptive experimentation \citep{lattimore2020bandit}. In a typical model, an experimenter sequentially assigns treatment arms to a large pool of individuals. Feedback on the treatment assigned to one individual is rapidly incorporated and used to improve future treatment assignments. Algorithms from this literature have been widely adopted to drive experiments that optimize features of digital platforms.

 It is common for real-life implementations to recommend running a bandit algorithm for some time, until enough data is gathered to confidently deploy one treatment arm to govern all future interactions. For instance, consider Google Analytics – a platform offering measurement and optimization tools that is used by 55\% of websites worldwide – and Stitch Fix, an online personal styling service. Blog posts from both companies \citep{scott2013google, Amadio2020StitchFix} contrast traditional A/B tests, which allocate a fixed proportion of experimental effort to each treatment arm, with multi-armed bandit experiments, emphasizing that the latter reduce costs and ``waste’’ by adaptively shifting experimental effort away from underperforming arms. Regardless of how  experimental effort is apportioned within the experiment, successful experiments end by deploying a treatment arm to the population. 
 
 Popular bandit algorithms are not designed to reach a deployment decision in this manner. In the academic literature, the typical goal is to design bandit algorithms that maximize total welfare across all treated individuals, or equivalently, to minimize total regret: the degradation in total welfare due to suboptimal treatment assignment. In principle, one could stop and deploy a fixed arm to the remaining population, but doing so is suboptimal. Moreover, it is known that algorithms designed to minimize regret are slow to gather the information needed to reach a deployment decision \citep{bubeck2009pure,villar2015multi}; this is true of the Thompson sampling  \citep{thompson1933likelihood, chapelle2011empirical} algorithm deployed by Google Analytics and Stitch Fix. 
 
 A variant of the MAB problem called the ``best-arm identification problem’’ studies deployment decisions, but ignores welfare (or regret) considerations. A common goal is to minimize the expected experiment length subject to deploying the correct treatment arm to the population with high probability. 
 
 This dichotomy between regret minimization and best-arm identification is reflected in the theoretical foundations of these problems. For regret minimization, the seminal work of \citet{lai1985asymptotically} established sharp asymptotic lower bounds on the expected regret of any policy. Similarly, for best-arm identification, \citet{kaufmann2016complexity} derived sharp asymptotic lower bounds on the expected sample complexity required to identify the best arm with a given confidence level. A compelling feature of this theory is that it often reveals popular algorithms, such as  Thompson sampling, to be asymptotically optimal for their respective objectives. However, these sharp results are specific to each problem formulation in isolation.
 
 A primary goal of our paper is to develop a unified framework that addresses both regret minimization and experiment length considerations while preserving the theoretical sharpness that characterizes each individual problem.  This approach aims to develop academic theory that more closely aligns with common experimental practices in industry, providing insights into the tradeoffs inherent in adaptive experimentation with deployment decisions and into the nature of efficient experimentation policies.

 \subsection{Contributions}\label{subsec:contributions}
This paper introduces a generalized model of bandit experiments. In this model, an experimenter interacts with a large population of individuals. During the experimentation phase, treatment arms are chosen sequentially, and noisy signals of their quality are observed. At any point, the experimenter can choose to conclude the experimentation phase and deploy an arm to the remaining population. Distinct per-period cost functions measure the cost of treatment decisions during the experimentation and deployment phases. We study policies that minimize cumulative expected costs, emphasizing insights into how treatments should be adaptively assigned within the experiment.

In our primary example, costs during each phase reflect the regret from treatment assignment, but costs during the experimentation phase are inflated to encode a desire to rapidly conclude adaptive experimentation and reach a deployment decision. Many of the same insights apply more broadly, so we allow for more abstract cost functions in our theory. Within-experiment costs could, for instance, reflect financial costs of treatment assignment in clinical trials. In simulation experiments \citep{hong2021review}, they may reflect the differentiated time required to simulate the performance of certain arms to the desired accuracy \citep{chick2001new}.

We characterize experimentation procedures that minimize cumulative expected costs when the size of the target population is sufficiently large. This leads to several insights:

 \begin{enumerate}
 	\item {\bf Unification of canonical theory:}  Our findings generalize and unify canonical results in the literature. When the cost functions, both within-experiment and post-experiment, are measured almost exclusively by the regret of treatment assignment, our objective aligns with that of \cite{lai1985asymptotically}. Our theory recovers their famous formula for optimal regret and the optimal division of experimentation effort among the arms. Another special case of our model sets the within-experiment cost of each arm to be the same. In this case, our results recover the sharp theory of optimal best-arm identification presented by \citet{garivier2016optimal}. This unification is noteworthy, given that the formulation of \cite{garivier2016optimal} appears quite different from that of \cite{lai1985asymptotically}, as do various mathematical expressions.
 	\item {\bf The Pareto frontier between experiment length and total regret:}  Section \ref{subsec:frontier}  discusses our theory's application in characterizing the inherent tradeoffs between the expected length of an adaptive experiment and the total regret incurred. The latter measures the cumulative suboptimality of assigned treatments across the whole population, including both within- and post-experiment treatments. Optimal multi-armed bandit algorithms minimize regret (or, equivalently, maximize welfare) by engaging in adaptive, sequential treatment assignments throughout the entire population. Our theory reveals that significant reductions in experiment duration can be achieved with negligible impact on cumulative regret. In regimes where the tradeoff is more substantial, we provide an exact characterization of the Pareto frontier and complement it with interpretable bounds  (Proposition~\ref{prop:Pareto_robustness}).
 	
 	\item {\bf Implications for popular bandit algorithms:}  Our theory suggests that small, careful adjustments to leading bandit algorithms are sufficient to optimize adaptive experimentation in very large populations. We illustrate this through the top-two Thompson sampling algorithm, a modification of Thompson sampling proposed by  \citet{russo2016simple,russo2020simple}. Adjusting a single tuning parameter in this algorithm is enough to attain any point on the Pareto frontier between experiment length and total regret. More surprisingly, adjusting this parameter is sufficient (asymptotically) to optimize any one of our generalized objectives. We also provide a method to apply this algorithm without tuning.
 	\item {\bf Insights into the nature of asymptotically efficient policies:} The surprising results about top-two Thompson sampling serve to illustrate a more general insight into the optimal allocation of experimental effort. We demonstrate that the nature of asymptotically efficient policies is nearly independent of the per-period cost functions. To articulate this precisely, we differentiate two aspects of algorithm design: the choice of an 'exploitation rate,' which determines the fraction of experimentation effort dedicated to the true best arm, and the intelligent division of exploration effort among the other arms. The optimal division of exploration effort turns out to be completely independent of the cost functions. Instead, it is determined by a certain information-balance property, which asserts that the strength of statistical evidence against the optimality of suboptimal alternatives should grow at an equal rate. Leading bandit algorithms like Thompson sampling attain this property ‘automatically.’ Because information balance is a purely statistical property, tailoring experimentation to specific cost considerations is reduced to an appropriate choice of exploitation rate. 
 \end{enumerate}
 
 \paragraph{Limitations of large-population asymptotics.}  The main strengths and the main shortcomings of our paper stem from the asymptotic nature of the results. This approach inherits the sharpness of pioneering works like  \cite{chernoff1959sequential} or \cite{lai1985asymptotically}, lending them a definitive quality. However, our large-population theory implicitly assumes that the benefits of deploying a superior treatment across the population far outweigh the costs of experimentation. In such cases, efficient policies minimize experimentation costs incurred in the process of gathering sufficient information. Our results are not appropriate in cases where the target population is small, experimentation costs are comparatively large, or the differences between treatment arms are nearly undetectable.  In such cases, rapidly deploying a suboptimal treatment arm might be preferable to expensive experimentation.

\section{Numerical teaser}\label{sec:teaser}

We motivate our theoretical study with a simple numerical illustration. We consider an adaptive experiment conducted to select a single treatment arm to deploy in a population of $n$ individuals. The experimenter interacts with individuals in sequence, selecting one of $k$ treatment arms to assign and observing noisy `reward' outcomes. In this experiment, the mean reward of arm $i\in[k]$ is given by a fixed, unknown value $\theta_i$, and reward noise is standard Gaussian. At any point, the experimenter can choose to stop the experiment and deploy a treatment arm to the remaining population.  

We compare policies in terms of their average experiment length and average cumulative regret. Regret measures the total suboptimality in treatments assigned across all individuals in the population, including those treated post-experiment. We imagine that 
committing to the deployment of a single treatment reduces operational costs. Hence, among policies that incur a given level of regret, those that reduce the average length of the adaptive experimentation phase are preferable. 

The problem formulation is described formally in Section \ref{sec:formulation} and  implementation details of the experiments are discussed in Appendix \ref{app:implementation details}.

\paragraph{Policies compared.}

We compare two ways of adaptively selecting treatment arms within an experiment. 

\begin{enumerate} 
	\item \emph{Epsilon-Greedy:} At any round within the experiment, the experimenter identifies the arm with the highest empirical-average reward so far. It plays that arm with probability $1-\epsilon$, exhibiting 'greedy' behavior. With probability $\epsilon$, it 'explores' by selecting an arm uniformly at random. The parameter $\epsilon$ thus balances exploitation of the current best arm with exploration of potentially better arms.   
	\item \emph{Top-two Thompson sampling (Top-two TS):} Thompson sampling (TS) is a prominent bandit algorithm that is widely used in industry and academia. The algorithm maintains a posterior beliefs over the true mean reward of each arm. At decision-time, an arm to play is sampled randomly according the posterior probability an arm offers the highest true mean. Top-two TS \citep{russo2016simple} modifies TS to explore more aggressively within an experiment. To select an arm, it first runs standard TS until two distinct arms are sampled. The first arm, called the leader, tends to concentrate on whichever arm is believed to be best. The second, called the challenger, samples an arm that may plausibly outperform the leader. Applied with tuning parameter $\beta \in (0,1)$, the algorithm plays the leader with probability $\beta$ and the challenger with probability $1-\beta$. If $\beta = 1$, top-two TS is standard TS. If one does not want to tune $\beta$, Section \ref{subsec:ttts} provides a more automatic way to tailor top-two TS to a specific objective. 
\end{enumerate}

In these numerical experiments, we apply both methods with a simple stopping rule based on $Z$-statistics, which stops once there is a high-probability guarantee of having identified the best-arm. Our theory in Section \ref{sec:main-Gaussian} implies this stopping rule is enough to attain a sharp asymptotic efficiency results.

\begin{figure}[!htbp]
	\centering
	\begin{subfigure}[t]{0.9\textwidth}
	\centering
	\includegraphics[width=4in]{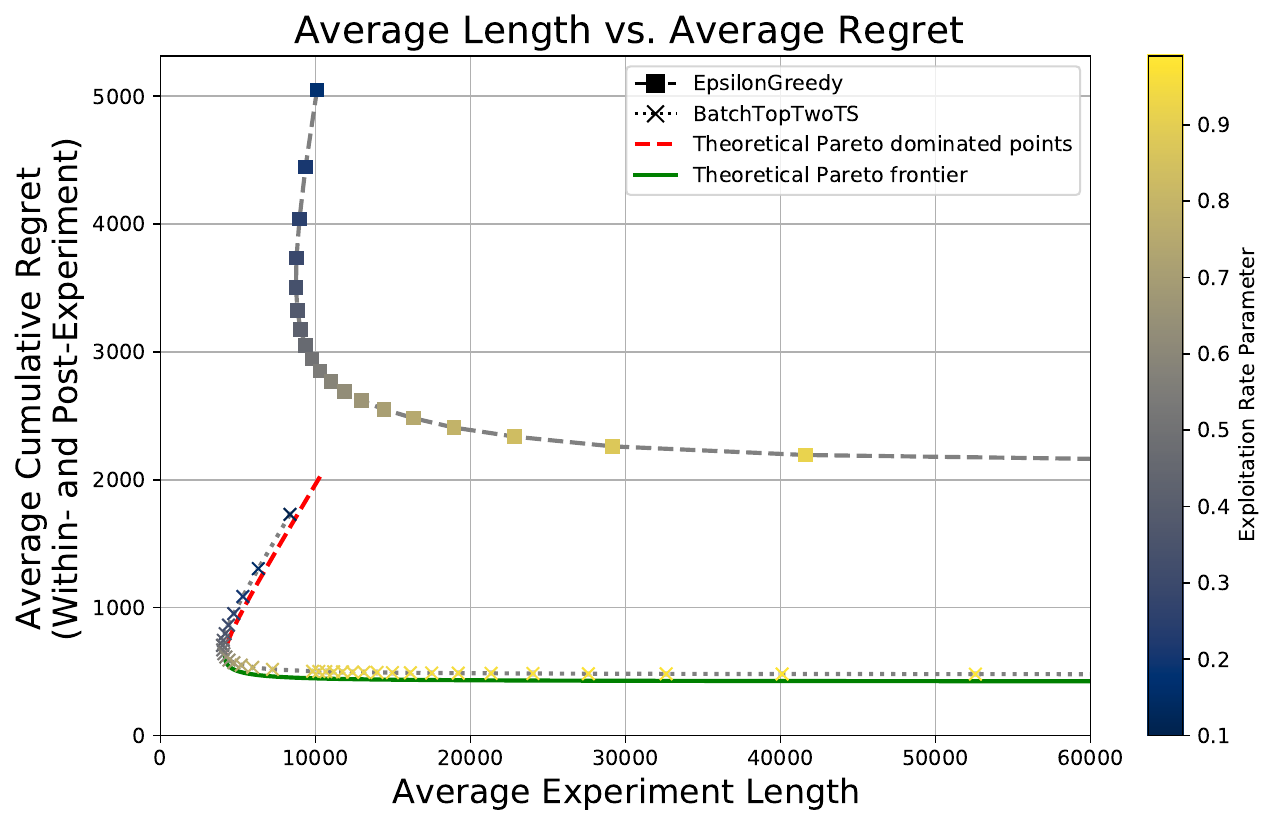}
	\caption{The set of length-regret outcomes attained for various exploitation-rate parameters (given by $\beta$ for top-two TS or $1-\epsilon+\frac{\epsilon}{k}$ for Epsilon-Greedy). The Pareto frontier (among all policies) predicted by our asymptotic theory is plotted in green.}
	\label{fig:teaser_pareto}
	\end{subfigure}
		\vspace{1em} 
	\begin{subfigure}[t]{0.9\textwidth}
		\centering
		\includegraphics[width=6in]{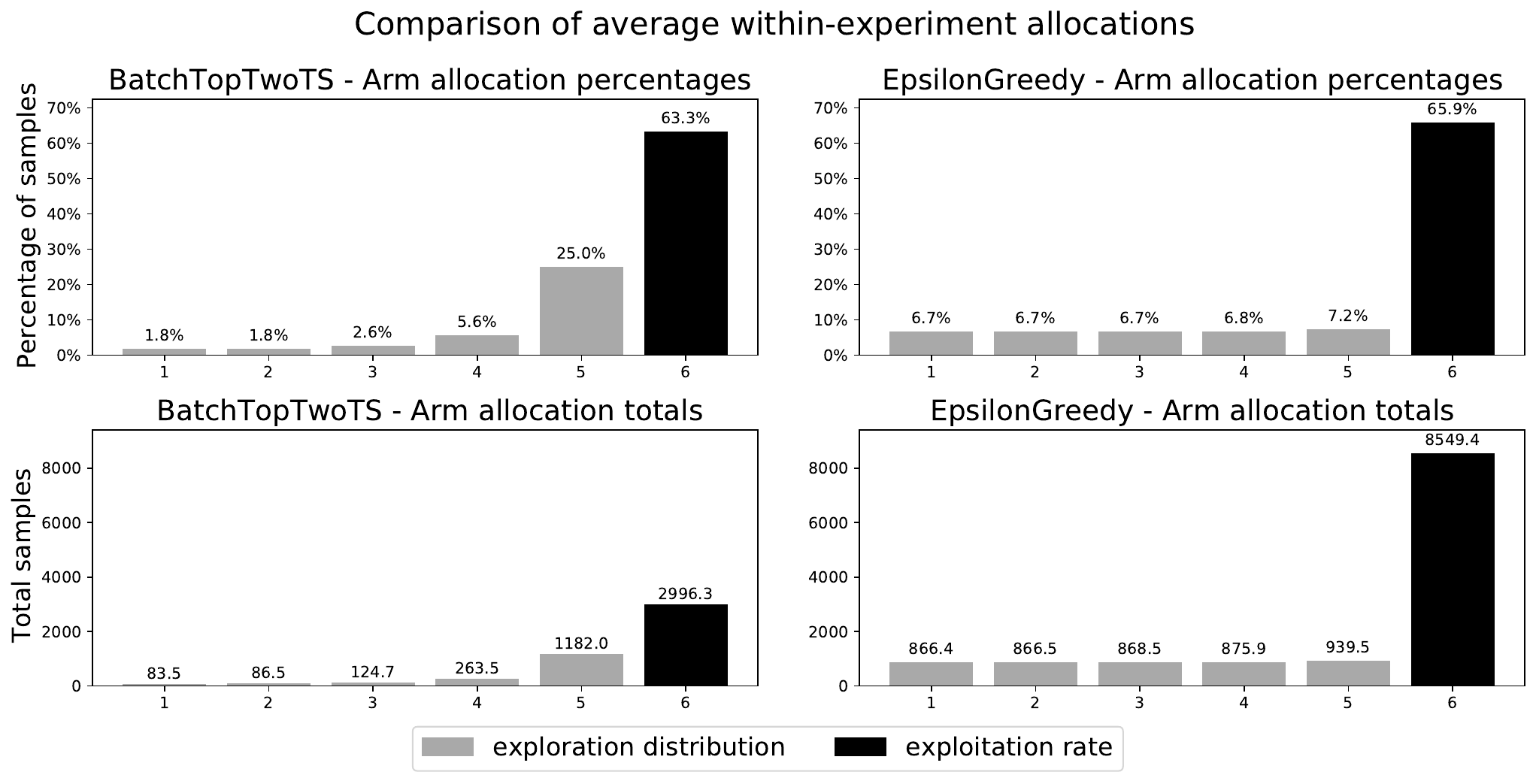}
		\caption{Within-experiment distribution of treatment arm allocations of top-two TS with $\beta=0.7$ and Epsilon-Greedy with $\epsilon=0.4$. The exploitation rate is largely determined by the tuning parameters ($\beta,\epsilon)$. Epsilon-Greedy distributes samples equally among suboptimal arms whereas top-two Thompson sampling shifts most samples to more competitive arms.}
		\label{fig:teaser_play_counts}
	\end{subfigure}
	\caption{Numerical performance of policies on the six arm instance $\thetabf=(\theta_1,\theta_2,\theta_3,\theta_4,\theta_5,\theta_6)=(0,0.2,0.4,0.6,0.8,1)$ with population size $n=10^8$.}
	\label{fig:teaser_combined}
\end{figure}

\paragraph{Findings and observations.}

Figure \ref{fig:teaser_combined} compares the performance of these policies empirically. We make several observations:
\begin{enumerate}
	\item \emph{Dominance of top-two TS over Epsilon-Greedy:}  Figure \ref{fig:teaser_pareto} shows that relative to Epsilon-Greedy, top-two TS results in \emph{both} substantially shorter experiments and a large reduction in cumulative regret.
	\item \emph{Efficient identification and deprioritization of clearly suboptimal arms:} Figure \ref{fig:teaser_play_counts} illustrates why top-two TS offered superior performance. While Epsilon-Greedy allocates equal probability to all non-greedy arm choices, top-two TS strategically shifts measurement away from arms that are clearly inferior and toward arms that have a plausible chance of being optimal.
	\item \emph{Low price of reducing experiment duration:} Standard Thompson sampling is analogous to top-two TS with extremely large exploitation rate parameter $\beta \approx 1$. By reducing $\beta$ and moving away from this standard algorithm, it is possible to massively reduce experiment length while having a negligible impact on the overall quality of treatment decisions (cumulative regret).      
\end{enumerate}

These results highlight interesting performance tradeoffs when applying popular heuristic policies. Rather than focus on specific policies, we aim to answer more fundamental questions about the optimal way to gather information within an experiment. It will turn out to be the case, however, that top-two Thompson sampling gathers information ``optimally'' in an appropriate asymptotic sense, and that this holds for a broad class of objectives. Therefore, this section may provide intuition that is helpful to understanding our broader theory in subsequent sections.

\section{Literature review}
Many papers have provided guarantees on both within-experiment and post-experiment performance in models of adaptive experimentation. Ours is unique in providing a \emph{sharp} theory of asymptotic efficiency (Definition~\ref{def:universal-efficiency}) that unifies and generalizes two parts of the canon of the multi-armed bandit literature: pure regret minimization problems in the style of  \cite{lai1985asymptotically} and (fixed-confidence) best-arm identification problems in the style of \cite{garivier2016optimal}.  An attractive feature of sharp theory, like that of \cite{lai1985asymptotically}, is that it provides a sense in which algorithms such as Thompson sampling or upper confidence bound (UCB) are \emph{exactly} optimal, and not e.g. within a factor of five of optimal. This approach also yields some novel insights, as discussed in Section \ref{subsec:contributions}.

Our analysis  builds on foundational insights by \cite{chernoff1959sequential, anantharam1989asymptotically2, rajeev1989asymptotically1} and \cite{graves1997asymptotically}. These papers all study a min-max optimization problem that characterizes asymptotic limits on cost or regret. In the appendix, we establish our main insights by analyzing this min-max problem as a two-player game and identifying its unique equilibrium.

Our model allows the experimenter to choose when to stop experimenting adaptively as information is revealed.  
Our work is most closely related to papers of \cite{chan2006sequential}, \citet{bui2011committing}, \cite{jamieson2014lilUCB}, \cite{degenne2019bridging}, and \cite{zhang2023fast}, which allow for adaptive stopping and study both within-experiment regret and post-experiment decision quality.  Similar to our length-regret objectives, \cite{bui2011committing} study a modified regret measure that incentivizes early commitment to a single treatment arm. Our results improve on their Theorem 3 and 5, which provide (non-matching) lower and upper bounds on this modified regret measure.  \cite{jamieson2014lilUCB} and \cite{degenne2019bridging} study upper confidence bound algorithms that can ensure that the optimal treatment arm is selected with high probability. \cite{degenne2019bridging} carefully studies how varying tuning parameters trades off regret and probability of incorrect selection, whereas \cite{jamieson2014lilUCB} mentions this possibility informally. While these three papers construct specific policies which guarantee low experiment-length and low regret, we develop a sharp asymptotic theory of optimal experimentation for a broad class of objectives and, as a consequence, characterize the Pareto frontier between experiment-length and regret. Recent work of \citet{zhang2023fast} carefully studies an extreme point of the Pareto frontier; they minimize experiment length subject to the constraint that the asymptotic scaling of regret is exactly minimized.  

Our contribution is most related to the paper of \cite{chan2006sequential}, which has been largely overlooked in the literature. Their results are asymptotically sharp, but different from ours. First, in our formulation there is no constraint on the stopping rule the experimenter can employ. Instead, an overall cost function incentivizes them to stop only once adequate information is gathered.  Second, \cite{chan2006sequential} only study the probability of selecting an arm whose distance from optimality exceeds a pre-specified tolerance.  Our post-experiment cost function can be quite general, encompassing regret measures as a special case. These differences are essential to recovering the regret results of \cite{lai1985asymptotically} and accommodating them introduces challenges in the analysis.\footnote{See Remark \ref{rem:lower_bound_proof} for comments on how the first change impacts the lower bound proof. In terms of modeling, our link to \cite{lai1985asymptotically} results from the ability to set within-experiment and post-experiment cost functions (nearly) equal to the per-period regret. Formulations that constrain the probability of incorrect selection seem to be unable to recover the overall regret objective. } Concurrently with our work, \cite{kanarios2024cost} study a model similar to that of \cite{chan2006sequential} but without requiring an indifference zone. While these papers offers sharp theory, they study different models and do not contain the four main insights mentioned in Section \ref{subsec:contributions}.

Many other works study problems in which the experimenter must stop at a pre-determined time. In these settings, \cite{BUBECK20111832, russo2018learning, krishnamurthy2023proportional, qin2023adaptive, zhong2023achieving} study both within-experiment and post-experiment regret.
There is currently no sharp theory of asymptotically efficient policies in this formulation,
even in simpler models which focus solely on the probability of selecting the best arm. An influential work of \cite{glynn2004large}  used large deviation theory to study optimal static allocations, providing a new perspective on so-called optimal computing budget allocation \citep{chen2000simulation}. Implementing the ideal allocation in \cite{glynn2004large} requires knowledge of the true arm-means. Much effort has gone into designing implementable adaptive algorithms that converge almost surely to the ideal allocation. See \cite{chen2023balancing} and the references therein.  Unfortunately, these do not satisfy the same large-deviations guarantee, and there are questions of whether universally efficient  policies exist at all in problems with a pre-determined stopping time; see \citep{qin2022open,ariu2021policy,degenne2023existence}.

A third kind of asymptotic theory looks at the rate of concentration of a Bayesian posterior distribution as information is acquired. 
\cite{russo2016simple,russo2020simple} constructed asymptotically optimal top-two sampling algorithms according to this measure. Most of the expressions defining asymptotic optimality matched those derived from a purely frequentist perspective by \cite{garivier2016optimal}, whose work appeared in the same conference. Later papers by \cite{qin2017improving} and \cite{shang2020fixed} provided guarantees for these top-two sampling algorithms in the same setting as \cite{garivier2016optimal}, which involves minimizing the expected stopping time subject to a constraint on the probability of identifying the optimal arm. We build on that work in Section \ref{subsec:ttts} to give insight into the asymptotic efficiency of top-two Thompson sampling for our generalized problem formulation. The exploration sampling algorithm of \cite{kasy2021adaptive} is very similar to top-two Thompson sampling applied with an unbiased coin, and exactly equivalent in large populations.  Both top-two Thompson sampling and exploration sampling have been used in recent real-world field experiments \citep{kasy2021adaptive, Rosenzweig2022Conversations}.

Other works have studied within-experiment and post-experiment performance from very different angles. 
A line of work on  information acquisition  concentrates on the exact characterization of optimal strategies tailored to specific objective functions \citep{Wald:1947, arrow1949bayes, lai1980sequential, Fudenberg2018,  morris2019wald, Liang2022, adusumilli2023sample}.
This tends to require studying very specialized models that cannot be easily compared to ours. Other papers derive useful algorithms by directly approximating the dynamic programming problem that defines Bayesian optimal performance \citep{chick2009economic, chick2012sequential, chick2021bayesian}. Additionally, \citet{LiuMBP14, li2022algorithms,  caria2023adaptive, athey2022contextual} contribute to the literature through simulation, numerical analyses, empirical studies or field experiments, encompassing a spectrum of within- and post-experiment objectives.

It is crucial that our work emphasizes the quality of post-experiment decision-making, rather than the goal of accurately estimating the post-experiment performance of all arms or their treatment effects. Adaptive experiments for treatment effect estimation are most often studied in problems with two arms, treatment and control. See \cite{bhat2020near, zhao2023adaptive, dai2023clipogd} and references therein for recent work along these lines.
Interesting work of \cite{adusumilli2022neyman} observes that adaptive experiments for treatment effect estimation are also ideal for identifying the best arm, \emph{if there are just two treatment arms.} 
When there are many arms, as in \cite{erraqabi17a, simchi-levi2023MABExperimentalDesign},  accurately estimating the post-experiment quality of all arms requires measuring clearly inferior treatment arms many times to understand precisely how inferior they are. Multi-armed bandit algorithms derive their efficiency by focusing on post-experiment decision quality; they reduce cost by learning just enough about inferior arms to confidently distinguish them from the best arm. (In other words, the precision with which an arm is related to its optimality gap).

\section{Problem formulation}\label{sec:formulation}

 An adaptive experiment is conducted to select a single treatment arm to deploy to a target population of $n$ individuals, among $k$ options. 
 The \emph{potential} outcome when treatment $i  \in [k] \triangleq \{1,\ldots, k\}$ is assigned to individual $t\in \{0,1,\ldots,n-1\}$ is a scalar random variable $Y_{t,i}$, with larger values indicating more desirable outcomes. Assume the potential outcomes $(Y_{t,i})_{t\in \{0, 1, \ldots, n-1\} } \overset{\text{i.i.d}}\sim P(\cdot \mid \theta_i)$ associated with any treatment $i\in[k]$ are drawn independently across individuals from a distribution $P(\cdot \mid \theta_i)$ with unknown parameter. If these parameters were known, the experimenter would deploy an arm with highest average treatment effect, $\max_{i\in[k]} \intop y P(\mathrm{d}y\mid \theta_i)$.

 Instead, they conduct an adaptive experiment to gather information before deploying an arm to the population. The experimenter interacts sequentially with individuals in the population. Assuming the experiment has not yet stopped, then on basis of the outcomes $Y_{0,I_0}, \ldots, Y_{t-1, I_{t-1}}\in \mathbb{R}$ of treatments $I_0,\ldots, I_{t-1}\in [k]$ assigned to the first $t <n$ individuals, the experimenter can choose either to continue experimental treatment assignments or to stop and deploy a single treatment to the remaining population:  
 \begin{itemize}
 	\item If the experimenter chooses to continue experimenting, they select a treatment arm $I_{t}\in[k]$ and observe the associated outcome $Y_{t,I_t}$.
 	\item  If the experimenter chooses to stop the experiment after observing treatment outcomes for $t$ individuals, they set the stopping time $\tau=t$ and choose an arm $\hat{I}_{\tau}$ to deploy to the remaining individuals throughout timesteps $\{\tau,  \ldots, n-1\}$. In other words, they commit irrevocably to select $I_{\ell} = \hat{I}_{\tau}$ for each $\ell \in \{\tau,  \ldots, n-1\}$, yielding the treatment outcomes $\left\{Y_{\tau, \hat{I}_\tau}, \ldots, Y_{n-1, \hat{I}_\tau}\right\}$. 
 \end{itemize}
 Once $t=n$, the entire population has been treated and the experiment must end by default. In this case, we set the stopping time $\tau=n$. 
  
 A \emph{policy} $\pi$ is a (possibly randomized) rule for making the sequential decisions above. Formally, it is a function that maps the population size $n$, the observations $(I_\ell, Y_{\ell, I_{\ell}})_{\ell \in\{0,\ldots, t-1\}}$ with $t<n$ in a non-stopped experiment, and an exogenous random seed $\zeta$, to a decision $\pi(n, (I_\ell, Y_{\ell, I_{\ell}})_{\ell \in\{0,\ldots, t-1\}}, \zeta)$. A decision is a pair of the form $(\mathtt{Continue}, I_t)$ or $(\mathtt{Stop}, \hat{I}_t)$.  We make explicit that decisions depend on the population size to accommodate asymptotic analysis in which the population size grows.

 The experimenter's goal is to deploy a policy that optimally balances the expected cost of experimenting with a desire to administer high quality treatments. The objective is formulated precisely in Sections~\ref{subsec:lenght-regret-objectives} and~\ref{subsec:general-objectives}. 
 
 \subsection{Restrictions on outcome distributions}
 We study the case where potential outcome distributions are members of a one-dimensional exponential family. The reader may have in mind Gaussian distributions with known variance, Bernoulli distributions, or Poisson distributions. See \cite[Chapter 1]{efron2022exponential} or \cite[Section 4]{cappe2013kullback} for background. Formally, we study outcome distributions that assign to any subset of outcomes $\mathcal{Y} \subset \mathbb{R}$ the probability 
 \begin{equation}
 	\label{eq:exponential-family}
 	P( \mathcal{Y} \mid \theta ) = \int_{\mathcal{Y}} \exp\left(\eta(\theta) y - b(\theta)\right)\rho(\mathrm{d}y).
 \end{equation}
Here $\eta(\cdot)$ is a strictly increasing function that associates any $\theta$ with a unique canonical parameter $\eta(\theta)$, the base measure is denoted by $\rho$, and  $b(\theta) = \log \intop_{\mathbb{R}} \exp\left(\eta(\theta) y\right) \rho(\mathrm{d}y)$ is the cumulant generating function. We assume the exponential family is regular, i.e. that the natural parameter space $\{\theta \in \mathbb{R} : b(\theta) < \infty \}$ is an open interval.  Finally, we assume that $\theta$ denotes the mean of the distribution, i.e. $\theta = \intop y P(\mathrm{d}y \mid \theta)$, which can always be ensured by change of variables.

 We let $\thetabf =(\theta_1, \ldots, \theta_k)$ denote the vector of unknown parameters and call $\thetabf$ a problem instance. We define $I^*(\bm{\theta}) \triangleq \argmax_{i \in [k]} \theta_i$ to be the set of optimal arms.
 In this paper, we focus on the set of problem instances with a unique best arm, formally defined as
 \[
 \Theta \triangleq \left\{\thetabf = (\theta_1,\ldots,\theta_k)\in \mathbb{R}^k \,\, : \,\,  b(\theta_i) < \infty, \, \forall \, i\in [k]\, \text{ and } \,  I^*(\thetabf) \text{ is a singleton set}\right\}.
 \]
 For $\thetabf\in\Theta$, we overload the notation $I^*(\thetabf)$ to be the only element in the singleton set. We often use the shorthand $I^* = I^*(\thetabf)$.

 \subsection{Length-regret objectives}\label{subsec:lenght-regret-objectives}
 This subsection introduces an important special case of our model. For any $\thetabf\in\Theta$ and $c \geq 0$, the overall cost function
 \begin{equation}\label{eq:length-regret-objective}
 	\mathrm{Cost}_{\bm{\theta}}(n, \pi) = c\cdot \mathrm{Length}_{\bm{\theta}}(n, \pi)  + \mathrm{Regret}_{\bm{\theta}}(n, \pi)
 \end{equation}
 assesses a policy's performance through the regret incurred and experimentation time required, defined as  
 \begin{align*}
 	\mathrm{Length}_{\bm{\theta}}(n, \pi) \triangleq \E_{\thetabf}^{\pi}[\tau]  
  \quad \text{and}\quad  
  \mathrm{Regret}_{\bm{\theta}}(n, \pi) \triangleq&\, \E_{\bm{\theta}}^{\pi}\left[ \sum_{t=0}^{n-1} \left(\theta_{I^*} - \theta_{I_t}\right)\right]\\
 	=&\, \E_{\bm{\theta}}^{\pi}\left[ \sum_{t=0}^{\tau-1} \left(\theta_{I^*} - \theta_{I_t}\right) + (n-\tau)\left(\theta_{I^*} - \theta_{\hat{I}_{\tau}} \right)\right].
 \end{align*}
 As a reminder, we follow the convention that $I_{t} =\hat{I}_{\tau}$ for each 
 $t\in \{\tau,\ldots, n-1\}$, reflecting that the experimenter has effectively pre-committed to assigning an arm to future individuals.

 The traditional multi-armed bandit literature focuses solely on minimizing regret, or, equivalently, on maximizing the cumulative reward generated by treatment assignment. 
 The cost measure in \eqref{eq:length-regret-objective} augments regret considerations by assigning a cost to adaptive experimentation itself. 
 Among policies with similar regret, those that quickly cease experimentation and commit to a treatment arm are preferred. 
 When $c$ is large, the experimenter may prefer to prioritize aggressive information gathering early on, incurring extra regret in order to stop experimentation early. 
 
 We will see that, as the population size $n$ tends to infinity, efficient algorithms have the character of multi-armed bandit algorithms if $c$ is a very small constant and that of best-arm identification algorithms if $c$ is a very large constant. 
 Varying $c$ enables us to characterize the nature of optimal policies in between these two extremes and, indirectly, to characterize the Pareto frontier between length and regret. 
 A practitioner may then simulate an optimal policy (e.g. top-two TS in Algorithm~\ref{alg:ttts}) for a range of different $c$'s and select a value that best aligns with their goals.  
 
 \subsection{Generalized objectives}\label{subsec:general-objectives}
 We now introduce a general class of objective functions, of which the length-regret objectives in Section \ref{subsec:lenght-regret-objectives} are a special case. While this may seem to suggest that precise specification of an objective function is the most important design consideration, our main insight is that a simple class of policies (mostly characterized by an information balance condition) is asymptotically optimal for a very broad class of objectives.

 Define the total expected cost of the experiment by
 \begin{equation}\label{eq:general-objective}
 	\mathrm{Cost}_{\bm{\theta}}(n, \pi) \triangleq \E_{\bm{\theta}}^{\pi}\left[ \sum_{t=0}^{\tau-1}C_{I_t}(\thetabf) + (n-\tau)\Delta_{\hat{I}_\tau}(\thetabf) \right],
 \end{equation}
 where for any $i\in[k]$,  $C_i:\mathbb{R}^k \mapsto\mathbb{R}_{\geq  0}$ and $\Delta_i:\mathbb{R}^k \mapsto\mathbb{R}_{\geq 0}$ are positive cost functions; respectively, these assign a per-individual cost to each treatment assignment within- and post-experiment. Length-regret objectives are a special case in which $C_{i}(\thetabf) = c+(\theta_{I^*}-\theta_{i})$ and  $\Delta_{i}(\thetabf)=\theta_{I^*}-\theta_{i}$.  
 Note that while the function forms are known to the experimenter, the function values $C_i(\thetabf)$ and $\Delta_i(\thetabf)$ depend on the unknown parameter $\thetabf$.

 The next example illustrates one natural specification of the generalized cost function. 
 
 \begin{example}[Best-arm identification with differentiated sampling costs] 
 	Consider $\thetabf\in\Theta$. For any $i\in[k]$, let $C_{i}(\thetabf) =c_i$ be an arm specific cost independent of the unknown parameter~$\thetabf$, 
  and $\Delta_{i}(\thetabf) = \ind(I^*\neq i)$ be an indicator of incorrect selection where $I^* = I^*(\thetabf)$.  Then 
 	\[
 	\mathrm{Cost}_{\thetabf}(n,\pi) = \E^{\pi}_{\thetabf}\left[\sum_{t=0}^{\tau-1} c_{I_t} + (n-\tau)\ind\left(\hat{I}_{\tau}\neq I^*\right) \right]. 
 	\] 
 	It will be clear in our theory that efficient policies stop in time $\tau=\tau(n)$ that is just logarithmic in the population size $n$ for optimal policies, so minimizing $\mathrm{Cost}_{\thetabf}(n,\pi)$ is asymptotically equivalent to  minimizing
 	\[
 	\E^{\pi}_{\thetabf}\left[\sum_{t=0}^{\tau-1} c_{I_t} + n\cdot \ind\left(\hat{I}_{\tau}\neq I^*\right) \right].
 	\]
 	When all $c_i$'s are equal, this is the Lagrangian form of the ``fixed-confidence'' best-arm identification problem \citep{kaufmann2016complexity}. 
 	The terms $(c_1, \ldots, c_k)$ might capture the time required to simulate the performance of an arm using a stochastic simulator \citep{chick2001new}. By picking $n$ to be very large, this formulation captures the problem of minimizing average simulation time required to select the optimal alternative.
 \end{example}
 
 We place two assumptions on the cost functions. The first is mild and assumes that deploying the best arm $I^*$ to the population is preferred and (without loss of generality) $\Delta_{I^*}(\thetabf)=0$. 
 \begin{assumption}
 	\label{asm:post-experimentation cost}
 	For any $\thetabf\in\Theta$, zero post-experiment cost is incurred if and only if the best treatment arm is deployed to the population, i.e., 
 	\[
 	\Delta_{I^*}(\thetabf)=0 \quad\text{and}\quad \Delta_{j}(\thetabf)>0, \quad \forall j\neq I^*.
 	\]
 \end{assumption}
 The next assumption is more substantial and indicates that the continued experimentation is strictly more costly than immediate deployment of the optimal arm. For a length-regret objective, the cost $c$ can be arbitrarily close to zero, but never exactly equal to zero. Unsurprisingly, our theory approximates known results in the case of $c=0$ by picking $c$ to be sufficiently small.
 \begin{assumption}
 	\label{asm:sampling cost}
 	For any $\thetabf\in\Theta$, there is a strictly positive cost to continuing to experiment with any treatment arm, i.e., 
 	\[
 	C_{\min}(\thetabf)\triangleq \min_{i\in[k]} C_{i}(\thetabf)>0.
 	\]
 	In addition, for any arm $i\in[k]$, 
  $C_i(\cdot)$ is continuous at~$\thetabf\in\Theta$.
 \end{assumption}

 \subsection{Large population asymptotics}
 We want to be able to argue that a specific policy incurs nearly minimal cost in a regime where the population size is very large. To provide sharp statements, we construct a sequence of problems in which the population size $n$ grows. Given some broad class of policies $\Pi$,  we seek a policy $\pi^*\in \Pi$ that is universally efficient in the large population limit, in the sense that for any other competing policy $\pi\in \Pi$, 
 \begin{equation}
 	\label{eq:universally efficient rule_initial}
 	\sup_{\thetabf \in \Theta} \,  \limsup_{n\to \infty}  \frac{\mathrm{Cost}_{\bm{\theta}}(n, \pi^*)}{\mathrm{Cost}_{\bm{\theta}}(n, \pi)} \leq 1.
 \end{equation}
 This notion of optimality is quite stringent, since it requires that $\pi^*$ is asymptotically optimal simultaneously for all arrangements of the arm means $\thetabf$. We emphasize that this notion of optimality, which is a rewriting of \cite{lai1985asymptotically}, is different from minimax optimality because since the large population limit is taken after fixing $\thetabf \in \Theta$. 
 
 The choice of a reasonable class of policies $\Pi$ is crucial.
 To understand the crux of the issue, consider a policy that stops experimentation after gathering very little information and then ``guesses'' which arm to deploy to the population. At the extreme, it may forego experimentation and deploy a fixed arm immediately. This policy incurs zero cost on instances $\thetabf$ in which it happens to deploy the true best arm --- a benchmark against which no reasonable policy can compete. The next definition excludes brittle policies of this type. Peaking ahead at our theory (see e.g. Theorem~\ref{thm:Lai-Robbins-type formula_general}), we will construct policies whose costs scale with $\ln(n)$, i.e. just logarithmically in the population size. Therefore, following \cite{lai1985asymptotically} and Definition 16.1 in \cite{lattimore2020bandit}, we say  a policy is \emph{consistent} if its total cost never scales \emph{exponentially faster} than $\ln(n)$.
 \begin{definition}[Consistent policy]
 	\label{def:uniformly good rule}
  A policy $\pi$ is said to be consistent if
 	\begin{equation}
 		\label{eq:uniformly good rule}
 		\forall \thetabf\in\Theta:\quad
 		\lim_{n\to \infty} \frac{\ln\left(\mathrm{Cost}_{\bm{\theta}}(n, \pi )\right)}{\ln(n)} = 0.
 	\end{equation}
  	Let $\Pi$ be the set of all consistent policies. 
 \end{definition} 
 Within the broad class of consistent policies, we seek one which is universally efficient. 
 \begin{definition}[Universally efficient policy]
 	\label{def:universal-efficiency}
 	A policy $\pi^*$ is universally efficient if for any other consistent policy $\pi \in \Pi$,
 	\begin{equation}
 		\label{eq:universally efficient rule}
 		\sup_{\thetabf \in \Theta} \limsup_{n\to \infty}  \frac{\mathrm{Cost}_{\bm{\theta}}(n, \pi^*)}{\mathrm{Cost}_{\bm{\theta}}(n, \pi)} \leq 1.
 	\end{equation}
 \end{definition}
 Note that it is easy to show that a policy $\pi^*$ must also be consistent if it satisfies \eqref{eq:universally efficient rule}.

 \paragraph{Optimal cost scaling.}
 As mentioned above, our theory will show that total costs scale just logarithmically in the population size. The following lemma, which we state without proof, connects the notion of universal efficiency to the optimality of an instance-dependent constant that determines the scaling of costs. 
 For two sequences of scalars $\{a_n\}$ and $\{b_n\}$, we write $a_n \sim b_n$ if $a_n/b_n \to 1$. 

 \begin{lemma}
 \label{lem:cost-scaling}
 	Consider a policy $\pi^*$  that satisfies 
 	\begin{equation}
        \label{eq:cost-scaling}
 		\forall\thetabf\in \Theta: \quad \mathrm{Cost}_{\bm{\theta}}(n, \pi^* )\sim \kappa_{\thetabf} \times \ln(n), 
   \quad\text{where}\quad \kappa_{\thetabf}>0.
 	\end{equation}
 	It is universally efficient if and only if, every other consistent policy $\pi \in \Pi$ satisfies the following lower bound:
 	\[ 
 	\forall\thetabf\in \Theta:  \quad \liminf_{n\to \infty} \frac{\mathrm{Cost}_{\bm{\theta}}(n, \pi ) }{\ln(n)} \geq \kappa_{\thetabf}.
 	\]
 \end{lemma}

\section{Main Results, specialized to the Gaussian setting}\label{sec:main-Gaussian}

This section presents our main insights into the nature of asymptotically efficient policies. We restrict to the special case where observations are normally distributed with known variance, i.e $p(y\mid\theta_i) = \mathcal{N}(\theta_i, \sigma^2)$, making the presentation easier to follow. The core theory extends gracefully to one-dimensional exponential family distributions.

Our main insights, previewed first in the introduction, are established here. In particular, Sections~\ref{subsec:info_balace_gaussian} and~\ref{subsec:ttts}  establish the insights into the nature of asymptotically efficient policies. Section~\ref{subsec:frontier} studies the Pareto frontier between regret and experiment length, and Section~\ref{subsec:lai_robbins_gaussian} provides a formula for optimal scaling of total cost that connects to a famous result of \cite{lai1985asymptotically}.

\subsection{An algorithm template}
This subsection provides a general algorithm template that specifies the condition upon which an experiment stops and which arm will be deployed upon stopping, but leaves open how information is gathered within the experiment itself. First we introduce some notation that is needed in the algorithm statement. For some $t\leq \tau$, if the number of samples of arm $i$, denoted by $N_{t,i} \triangleq \sum_{\ell=0}^{t-1} \ind(I_\ell =i)$, is positive, define the empirical mean reward and associated standard error as
\begin{equation}\label{eq:posterior}
	m_{t,i} \triangleq \frac{\sum_{\ell=0}^{t-1} \ind(I_\ell=i) Y_{\ell,i}}{N_{t,i}} \quad \text{and} \quad   s_{t,i} \triangleq \frac{\sigma}{ \sqrt{N_{t,i}} }.
\end{equation}
When $N_{t,i} = 0$, we let $m_{t,i} = 0$ and $s_{t,i} = \infty$.

For any two arms $i, j\in[k]$, the $Z$-statistic for the difference in means
\begin{equation}\label{eq:z-stats}
	 Z_{t,i,j} \triangleq \frac{ m_{t,i} - m_{t,j} }{ \sqrt{ s_{t,i}^2 + s_{t,j}^2 } } = \frac{ m_{t,i} - m_{t,j} }{ \sqrt{\sigma^2/N_{t,i}  + \sigma^2/N_{t,j}  } }
\end{equation}
measures the strength of evidence that arm $i$ outperforms arm $j$ in the population. In non-adaptive experiments, $Z_{t,i,j}\sim \mathcal{N}\left(\frac{ \theta_i - \theta_j }{ \sqrt{\sigma^2/N_{t,i}  + \sigma^2/N_{t,j}  } }, 1\right)$ has a Gaussian distribution with unit variance. 
When $N_{t,i}=0$ or $N_{t,j}=0$, we let $Z_{t,i,j}=0$.

Algorithm  \ref{alg:general-template-gaussian} provides a general template for algorithm design. The experimenter continues experimenting until certain $Z$-statistics are uniformly large, with threshold picked so that the probability of selecting a suboptimal arm is $O\left(\frac{1}{n}\right)$. In this paper, we focus on the following instantiation of such threshold: let $\gamma_0(n) = +\infty$ and for $t=1,\ldots, n-1$, 
\begin{equation}
\label{eq:tight threshold T}
  \gamma_t(n) \triangleq \ln(n) + \ln(k-1) + 6\ln\left(\frac{\ln(n)+\ln(k-1)}{2}+2\right) +  6\ln\left(\ln\left(\frac{t}{2}\right)+1\right) + 14.
\end{equation}
Upon stopping, the experimenter deploys the arm with highest empirical mean to the remaining population. An unspecified ``anytime\footnote{The term ``anytime'' refers to the fact that the allocation rules do not take the population size $n$ as input. They are fixed functions that map histories of arbitrary length to action distributions. This is not a requirement in our lower bounds, but it is sufficient for designing asymptotically optimal rules and simplifies analysis of them.} allocation rule'' governs  adaptive assignment of treatment arms during the experiment. 
This section focuses on the nature of optimal allocation rules.

\begin{algorithm}[H]
	\centering
	\caption{General template for Gaussian distributions}\label{alg:general-template-gaussian}
	\begin{algorithmic}[1]
		\State{\bf Input:}  $\mathtt{AllocationRule}(\cdot)$, a function which takes an input a history $\mathcal{H}_t$ of arbitrary length and returns a probability distribution over $[k]$.
		\State {\bf Initialize:} $\mathcal{H}_0 \gets \{ \}$, $\tau=n$. 
		\For{$t=0,1,\ldots, n-1$}
		\If{$t < \tau$}
            \LineComment{\emph{\textcolor{blue}{If experimentation has not yet stopped, ...}}}
		\State Obtain $\hat{I}_t \in \argmax_{i \in [k]} m_{t,i}$ and $\gamma_t(n)$ as in \eqref{eq:tight threshold T}.
		\If {$\frac{1}{2}\min_{j\neq \hat{I}_t }Z^2_{t, \hat{I}_{t}, j} \geq   \gamma_t(n)$} 
            \LineComment{\emph{\textcolor{blue}{If all $Z$-statistics are large, stop experimenting. }}}
		\State Set $\tau=t$ and $\hat{I}_{\tau} \in \argmax_{i \in [k]} m_{\tau,i}$.
            \State Assign arm $I_t = \hat{I}_{\tau}$.
            \Else 
            \LineComment{\emph{\textcolor{blue}{Otherwise, continue experimenting using the allocation rule.}}} 
		\State Assign arm $I_t \sim \mathtt{AllocationRule}(\mathcal{H}_t)$.
		\State Observe $Y_{t,I_t}$ and update history $\mathcal{H}_{t+1} \gets \mathcal{H}_t\cup \{(I_t, Y_{t,I_t}) \}$.
            \EndIf
            \Else
            \LineComment{\emph{\textcolor{blue}{If experimentation has stopped, ...}}}
		\State Assign arm $I_t = \hat{I}_{\tau}$.
		  \EndIf
		\EndFor 
	\end{algorithmic}
\end{algorithm}

\begin{remark}[A constraint on the probability of incorrect select]
	The threshold $\gamma_{t}(n)$ is adapted from \cite{Kaufmann2021martingale}, and ensures that Algorithm \ref{alg:general-template-gaussian} obeys the following constraint on the probability of incorrect selection:
	\begin{equation}\label{eq:pcs_constraint_initial} 
	\sup_{\varthetabf \in \Theta}  \Prob_{\varthetabf}\left( \hat{I}_{\tau}  \neq  I^*(\varthetabf) \wedge \tau < n \right) \leq \frac{1}{n}.
	\end{equation} 
	See Appendix \ref{app:post-experiment-cost}. However, unlike much of the literature on (fixed-confidence) best-arm identification, we do not impose this as a constraint. In our formulation, the experimenter is free to stop at any time, and it is a surprising consequence of our theory that universally asymptotically efficient rules can be constructed while obeying \eqref{eq:pcs_constraint_initial}. See Remark \ref{rem:lower_bound_proof} in the appendix for discussion of technical challenges in the lower bound proof.
\end{remark}

\subsection{Model for sample path asymptotics of Algorithm \ref{alg:general-template-gaussian}}

 For asymptotic analysis of the allocation rule supplied in Algorithm \ref{alg:general-template-gaussian}, it is convenient to ignore the stopping rule altogether and imagine the allocation rule were applied indefinitely. For this purpose, we extend the array of latent potential outcomes associated with arm $i\in [k]$ to the infinite array $(Y_{t,i})_{t \in \mathbb{N}_0}$. An anytime allocation rule can be applied without stopping to produce an infinite sequence of actions  $(I_{t})_{t \in \mathbb{N}_0}$ with 
\begin{equation}\label{eq:allocaiton-only-sample}
I_t \mid \mathcal{H}_t \sim \mathtt{AllocationRule}(\mathcal{H}_t) \quad \text{where} \quad \mathcal{H}_t = (I_0, Y_{0,I_0}, \ldots, I_{t-1}, Y_{t-1, I_{t-1}}). 
\end{equation}
We call the infinite sequence of within-experiment outcomes $(I_t, Y_{t,I_t})_{t\in \mathbb{N}_0}$ an \emph{indefinite-allocation sample path.}

Notice that the population size $n$ is not relevant to the above construction. It plays a role only in the stopping time 
\begin{equation}\label{eq:stopping_time}
\tau  = \min\left\{ t \leq n  \,:\,  \min_{j\neq \hat{I}_t }Z_{t, \hat{I}_{t}, j} \geq   \gamma_t(n)  \right\},
\end{equation}
which is a function of the indefinite-allocation sample path  and the specified population size. 

We emphasize that this construction is used only to analyze policies that fit the template of Algorithm \ref{alg:general-template-gaussian}. Our lower bound proofs do not use such a construction.

\subsection{Properties of an optimal allocation: information balance with a cost-aware exploitation rate}\label{subsec:info_balace_gaussian}

We offer a sufficient (and almost necessary) condition on allocation rules for Algorithm \ref{alg:general-template-gaussian} to be universally efficient.  The result references a custom notion of convergence for random variables that is slightly more stringent than almost sure convergence. We defer discussion of this technical detail until to Section \ref{subsec:strong-convergence}. Throughout, 
we denote the empirical allocation vector by 
\begin{equation}
\label{eq:empirical allcoation}
{\bm p}_t = \left(p_{t,1}, \ldots,  p_{t,k}\right) 
\text{ with }
p_{t,i} \triangleq \frac{N_{t,i}}{t}
\text{ being the fraction of measurements allocated to arm $i\in [k]$.}
\end{equation}

\begin{Theorem}[Optimality condition of allocation rules]
\label{thm:efficient-p-gaussian} 
Algorithm~\ref{alg:general-template-gaussian} is universally efficient (Definition \ref{def:universal-efficiency}) if 
the allocation rule satisfies the following condition: for any $\thetabf\in\Theta$,  under any produced indefinite-allocation sample path~\eqref{eq:allocaiton-only-sample}, the empirical allocation ${\bm p}_t$ converges strongly (Definition \ref{def: strong convergence}) to a unique probability vector $\bm{p}^* = \left(p_1^*,\ldots,p_k^*\right) > \bm{0}$ satisfying the information balance condition
	\begin{equation}
        \label{eq:info-balance-gaussian}
		\frac{ (\theta_{I^*} - \theta_{i})^2}{ \sigma^2/p^*_{I^*}   + \sigma^2/p^*_{i}} = \frac{ (\theta_{I^*} - \theta_{j})^2}{ \sigma^2/p^*_{I^*}   + \sigma^2/p^*_{j}},  \quad \forall  i,j \neq I^*
	\end{equation}
	and (cost-aware) exploitation rate condition 
	\begin{equation} 
        \label{eq:exploitation-rate-gaussian}
	p^*_{I^*} = \sqrt{\sum_{j\neq I^*} (p_j^*)^2 \frac{C_j(\thetabf)}{C_{I^*}(\thetabf)} }.
	\end{equation}
By contrast, Algorithm~\ref{alg:general-template-gaussian} is not universally efficient 
if there exists $\thetabf\in\Theta$ such that the allocation rule's empirical allocation ${\bm p}_t$ converges strongly to a probability vector other than ${\bm p}^*$.
\end{Theorem}

It is helpful to think of the long-run allocation as having two components: a scalar \emph{exploitation rate} $p^*_{I^*}$ and a vector allocation of residual \emph{exploration rates} $(p_j^*)_{j\neq I^*} \in \mathbb{R}^{k-1}$. Empirical analogues of these rates were depicted by black and gray bars, respectively, in Section \ref{sec:teaser} -- Figure \ref{fig:teaser_play_counts}. Theorem~\ref{thm:efficient-p-gaussian} reveals the following insight: 
\begin{enumerate}
	\item The optimal choice of scalar exploitation rate $p^*_{I^*}$ depends on the within-experiment cost functions $(C_i(\thetabf))_{i\in [k]}$. 
	\item Given an exploitation rate, the vector of exploration rates is uniquely determined by the information balance condition \eqref{eq:info-balance-gaussian}, independently of the cost functions $(C_i(\thetabf), \Delta_i(\thetabf))_{i\in [k]}$. 
\end{enumerate}
As the name suggests, one can interpret information balance~\eqref{eq:info-balance-gaussian}  as balancing the weight of evidence against suboptimal arms. From the formula for the (squared) $Z$-statistics in \eqref{eq:z-stats}, information balance is easily seen to be equivalent to the property   
\begin{equation}\label{eq:Z-growth-rate}
	Z_{t,I^*,i}^2 \sim Z_{t, I^*, j}^2, \quad \forall i,j\neq I^*, 
\end{equation}
i.e. $	Z^2_{t,I^*,i}/Z^2_{t, I^*, j} \to 1$.  
Attaining information balance  requires sampling grossly suboptimal arms less frequently and allocating more samples to more competitive arms. Popular bandit algorithms are designed to somewhat `automatically' meet such a constraint. The next remark provides some intuition for this result.

\begin{remark}
	Theorem \ref{thm:efficient-p-gaussian} may appear strange. How could optimal long-run behavior depend so little on the within-experiment and post-experiment cost functions? The first observation is that \eqref{eq:Z-growth-rate} itself does not depend on the cost functions and yet nearly determines the optimal proportions ${\bm p}^*$. Intuition for \eqref{eq:Z-growth-rate} can be given through hand-wavy calculations. Under non-adaptive allocation rules\footnote{This first inequality would hold if the allocation rule were not adaptive.  \cite{russo2020simple} analyzes adaptive allocations and provides rigorous expressions for posterior probabilities that are of this type. Be careful, however, as purely frequentist large deviations analysis of this type breaks down when allocation rules are adaptive. }, the probability of incorrectly selecting arm $i\neq I^*$ roughly behaves as
	\begin{equation}
	\Prob( m_{t,i} \geq m_{t, I^*} ) \lessapprox  \Phi( - Z_{t,I^*, i}) \approx  \exp\left\{ - \frac{Z_{t,I^*, i}^2}{2}   \right\}.
	\end{equation}
	The information balance property in \eqref{eq:Z-growth-rate} roughly translates to the property that $\ln 	\Prob( m_{t,i} \geq m_{t, I^*} ) \sim 	\ln  \Prob( m_{t,j} \geq m_{t, I^*} )$ for each $i,j \neq I^*$. 	Correspondingly, information imbalance roughly translates to the event  of incorrectly selecting one arm $i$ being exponentially rare relative to the event that arm $j$ is incorrectly chosen:  
	\begin{equation}\tag{Information-Imbalance}\label{eq:information-imbalence}
		\Prob( m_{t,i} \geq m_{t, I^*} ) \lessapprox \exp\{ - \alpha t\}  \times \Prob( m_{t,j} \geq m_{t, I^*} )
	\end{equation} 
	for some $\alpha>0$. See \cite{russo2016simple} for Bayesian arguments of this type.  Information imbalance is wasteful. If \eqref{eq:information-imbalence} holds, then for large $t$ the experimenter knows that mistakes, if they occur at all, are exponentially more likely to come from arm incorrect selections of  $j$ and not $i$. It would be better to shift experimentation effort away from arm $i$ and reallocate it to reducing uncertainty about $j$, roughly rebalancing the information gathered about those two arms.
\end{remark}

We close this section with a remark on linking ${\bm p}^*$ to segments of the literature that study best-arm identification and regret-minimization. 
\begin{remark}[Connections to existing allocations]\label{rem:connections} Consider objective in Section \ref{subsec:lenght-regret-objectives}, where
	$C_{j}(\thetabf) = c+ (\theta_{I^*} - \theta_j)$. As $c\to \infty$, $\frac{C_j(\thetabf)}{C_{I^*}(\thetabf)} \to 1$ and the formula for ${\bm p}^*$ recovers the optimal allocation in the best-arm identification literature \citep{kaufmann2016complexity, russo2020simple}. 
	
	As $c\to 0$, $\frac{C_j(\thetabf)}{C_{I^*}(\thetabf)} \to \infty$ and one must have $p_{I^*}^* \to 1$. We observe that information balance condition~\eqref{eq:info-balance-gaussian} requires 
	\[ 
	\frac{p_{j}^*}{1-p_{I^*}^*}  \sim  
 \frac{(\theta_{I^*} - \theta_{j})^{-2}}{ \sum_{i\neq I^*}  (\theta_{I^*} - \theta_{i})^{-2} } \quad \text{as}\quad p_{I^*}^*\to 1.
	\]
 Suboptimal arms are sampled with frequency inversely-proportional to their squared optimality gap, a defining property of optimal algorithms for regret minimization \citep{lai1985asymptotically, lattimore2020bandit}.   
\end{remark}
\begin{remark}[Connections to optimal computing budget allocations]
	When $C_{i}(\thetabf)=1$ for any $i\in[k]$, the allocation $\bm{p}^*$ matches that derived in the simulation optimization literature by \cite{glynn2004large} as a refinement and justification of \cite{chen2000simulation}. This should be interpreted with caution, however, since the derivation is quite different and the formula does not match in non-Gaussian cases.\footnote{The allocation $\bm{p}^*$ in Theorem~\ref{thm:efficient-p-general} only matches that in \cite{glynn2004large} for Gaussian distributions where the KL divergence is symmetric.} See the literature review for further discussion and open questions in that theory. 
\end{remark}

\subsection{A Lai-Robbins-type formula for the optimal cost scaling}\label{subsec:lai_robbins_gaussian}
The pioneering work of \cite{lai1985asymptotically} identified that a policy $\pi$ has asymptotically optimal regret scaling if and only if, 
\begin{equation}
	\label{eq:lai-robbins-regret}	\forall\thetabf\in\Theta:\quad\mathrm{Regret}_{\thetabf}(n, \pi) \sim  \sum_{j\neq I^*} \frac{\theta_{I^*} -\theta_j }{ {\rm KL}(\theta_j , \theta_{I^*})  }  \times \ln(n) \quad\text{as}\quad n\to \infty.
\end{equation}
The notation $f(n) \sim g(n)$ means that $f(n)/g(n) \to 1$ and $ {\rm KL}(\theta, \theta')$ denotes the KL-divergence between the distributions $p(y\mid\theta)$ and $p(y\mid\theta')$ parameterized by mean values $\theta$ and $\theta'$, respectively.

The next theorem provides a similar result for our generalized problem formulation.  
\begin{Theorem}[Lai-Robbins-type formula]
	\label{thm:Lai-Robbins-type formula}
	A policy $\pi$ is universally efficient if and only if, 
	\begin{equation}
		\label{eq:sufficient and necessary condition of universal efficiency Gaussian}
		\forall\thetabf\in\Theta:\quad
		\mathrm{Cost}_{\thetabf}(n, \pi) \sim  \sum_{j\neq I^*}  \frac{C_{j}(\thetabf) }{ {\rm KL}(\theta_j  ,  \bar{\theta}^*_{I^*,j})  }  \times \ln(n) 
		\quad\text{as}\quad n\to \infty,
	\end{equation}
	where  $\bar{\theta}^*_{I^*,j} \triangleq \frac{p^*_{I^*}\theta_{I^*}+p_j^* \theta_j}{p^*_{I^*}+p^*_j}$ with $\bm{p}^* = (p_1^*,\ldots,p_k^*)$ identified in Theorem~\ref{thm:efficient-p-gaussian}. 
\end{Theorem}

Consider the length-regret objectives in Section \ref{subsec:lenght-regret-objectives}, where $C_{i}(\thetabf) = c+(\theta_{I^*}-\theta_i)$. As $c \to 0$, the objective is  equivalent to the regret objective in \cite{lai1985asymptotically}. In that limit, \eqref{eq:sufficient and necessary condition of universal efficiency Gaussian} recovers \eqref{eq:lai-robbins-regret}, since $C_{i}(\thetabf) \to \theta_{I^*}-\theta_i$ and $p^*_{I^*}\to 1$. See also Remark \ref{rem:connections}. Although this expressions look very different, \eqref{eq:sufficient and necessary condition of universal efficiency Gaussian} also matches one from \cite{garivier2016optimal} in best-arm identification problems where $C_{i}(\thetabf)=1$. (See Theorem \ref{thm:equilibrium} for simplification of the min-max problem studied in \cite{garivier2016optimal}.)

\subsection{Automatic information balance via top-two Thompson sampling}\label{subsec:ttts}

It is not difficult to engineer a universally efficient allocation rule that mimics the long-run proportions ${\bm p}^*$. See Algorithm \ref{alg:D-Tracking} in Appendix \ref{app:tracking}. Here we demonstrate that, after simple modifications, leading bandit algorithms can also be universally efficient. 

We focus on a modification of Thompson sampling (TS) \citep{thompson1933likelihood}, a leading bandit algorithm that is widely used in both academia and industry. Pseudocode for standard TS is presented in Algorithm \ref{alg:ts}. Rather than play the arm with the highest estimated mean, it plays the arm with the highest posterior sample; the variance in this sampling drives exploration. Another interpretation of this method is that it samples an arm according the posterior probability it is the optimal arm \citep{scott2010modern}. Algorithm \ref{alg:ts} is presented using an improper prior.\footnote{With an improper prior, $s_{t,i}=\infty$ if an arm has never been played before. In that case, $\tilde{\theta}_i$ takes on either positive or negative infinity, with equal probabilities.} If related data is available prior to the experiment, encoding this in an informed prior can have a substantial impact on practical performance. 

As mentioned in Section \ref{sec:teaser}, Thompson sampling can perform poorly when the length of the experimentation phase is a key consideration (see Section~\ref{subsec:lenght-regret-objectives}). As the experimenter gains confidence about the identity of the optimal arm, it plays that arm almost exclusively. Therefore, it is slow to gather information about competing arms. 
Algorithm \ref{alg:ttts} presents pseudocode for the top-two sampling \citep{russo2020simple} variant of Thompson sampling discussed previously in Section \ref{sec:teaser}. It overcomes this shortcoming of regular Thompson sampling by exploring competing arms more frequently. It runs Thompson sampling until  two distinct arms are drawn --- a leader $I_t^{(1)}$ and a challenger $I_t^{(2)}$ --- and flips a biased coin to pick among them. As the posterior concentrates, the same arm is consistently picked as the leader, and the challenger is picked randomly according the posterior probability it is is optimal given that the leader is not.

\begin{figure}[t]
	\begin{minipage}[t]{0.46\textwidth}
		\begin{algorithm}[H]
			\centering
			\caption{Thompson sampling (TS)}\label{alg:ts}
			\begin{algorithmic}[1]
				\State {\bf Input:} History $\mathcal{H}_t$.
				\State  Form posterior $\nu_t = (m_{t,i}, s_{t,i}^2)_{i\in [k]}$ as in~\eqref{eq:posterior}.
				\State Sample $\tilde{\theta}_i \sim \mathcal{N}(m_{t,i}, s_{t,i}^2)$ \,\, for each $i\in [k]$.
				\State \Return $\argmax_{i \in [k]} \tilde{\theta}_{i}$.
			\end{algorithmic}
		\end{algorithm}
	\end{minipage}
	\begin{minipage}[t]{0.46\textwidth}
		\begin{algorithm}[H] 
			\centering
			\caption{Top-two Thompson sampling with cost-aware selection rule}\label{alg:ttts}
			\begin{algorithmic}[1]
				\State {\bf Input:} History $\mathcal{H}_t$.
				\State Sample $I_t^{(1)} \sim \mathrm{TS}(\mathcal{H}_t)$ using Algorithm \ref{alg:ts}.
				\Repeat 
				\State Sample $I_t^{(2)} \sim \mathrm{TS}(\mathcal{H}_t)$ using Algorithm \ref{alg:ts}.
				\Until{$I_t^{(2)} \neq I_t^{(1)}$}
				\State Determine coin bias $h_t$ via~\eqref{eq:cost-aware-IDS}.
				\State \Return $I_t^{(1)}$ w/ prob $h_t$, $I_{t}^{(2)}$ otherwise. 
			\end{algorithmic}
		\end{algorithm}
	\end{minipage}
	\vspace{1mm}
\end{figure}

\cite{russo2016simple,russo2020simple} showed that asymptotic efficiency in best-arm identification problems could be attained by ``tuning" the coin's bias, but suggested that simple constant choices (like $h_t=1/2$) already performed well. \cite{qin2023dualdirected} shows that an extremely simple selection rule attains asymptotic optimality in best-arm identification problems (though subtle mathematics is involved in proving this). Building on their proposal, we suggest the rule:
\begin{equation}
\label{eq:cost-aware-IDS}
	h_{t} =  h_{t, I_t^{(1)}, I_{t}^{(2)}}   \quad \text{where} \quad  h_{t,i,j} \triangleq \frac{ \frac{1}{p_{t,i} C_{i}( \bm{m}_t)} }{\frac{1}{p_{t,i} C_{i}( \bm{m}_t)} + \frac{1}{p_{t,j} C_{j}( \bm{m}_t)} },
\end{equation} 
where $p_{t,i} = \frac{N_{t,i}}{t}$ is the proportion of time arm $i$ has been played in the past. 
The bias $h_t$ is large, meaning arm $I_t^{(1)}$ is more likely to be chosen, when this arm has been played less often or has lower estimated experimentation cost.  If either the leader or challenger has never been sampled before, we use the default choice $h_t=1/2$. 

 The next result (nearly) establishes that top-two TS is a universally efficient allocation rule. The proof adapts results on top-two TS from \cite{qin2017improving, russo2020simple, shang2020fixed} and \cite{qin2023dualdirected} to establish (strong) convergence of empirical sampling proportions $\bm{p}_t$ to the optimal allocation ${\bm p}^*$. The analysis on which we build is quite intricate, and imposes some simplifying assumptions. For this reason, Proposition \ref{prop:TTTS} does not cover instances in which two arms have identical performance, and requires a more stringent regularity condition (Assumption~\ref{asm:sampling cost stronger}). These restrictions relate to challenges\footnote{These technical restrictions are used in an initial part of the analysis that, essentially, excludes the possibility that top-two TS neglects to explore an arm as the number of interactions $t$ goes to infinity. See Section \ref{subsec:app_sufficient_exploration}. Uniform bounds on $C_{i}(\cdot)$ play a role in analyzing \eqref{eq:cost-aware-IDS} during this initial phase, which otherwise could blow up if the empirical mean vector $\bm{m}_t$ takes very extreme values. A more thorough proof might apply concentration inequalities to control $\bm{m}_t$ rather than assume $C_i(\cdot)$ is uniformly bounded.} of analyzing the transient performance of top-two TS while dealing with the bespoke converge notion in Definition \ref{def: strong convergence}.

\begin{restatable}[Optimality of Algorithm \ref{alg:ttts}]{proposition}{TTTS}
	\label{prop:TTTS}
	Let observations follow the Gaussian model where for any $i\in[K]$, $p(y\mid\theta_i) = \mathcal{N}(\theta_i, \sigma^2)$. 
	Let $\widetilde\pi$ implement Algorithm \ref{alg:general-template-gaussian} with Algorithm~\ref{alg:ttts} serving as the allocation rule. Under Assumption~\ref{asm:sampling cost stronger}, for any consistent policy $\pi \in \Pi$, 
	\[
	 \sup_{\thetabf \in \widetilde{\Theta}} \, 
	\limsup_{n\to \infty} \,  \frac{\mathrm{Cost}_{\bm{\theta}}(n, \widetilde\pi)}{\mathrm{Cost}_{\bm{\theta}}(n, \pi)} \leq 1,
	\]
	where $\widetilde{\Theta} \triangleq \left\{\thetabf=(\theta_1,\ldots,\theta_k)\in\Theta \,:\,  \theta_i\neq \theta_j, \, \forall i,j\in[k] \right\}$ is the set of instances with distinct arm means.
\end{restatable}

\begin{assumption}[A stronger assumption than Assumption~\ref{asm:sampling cost}]
\label{asm:sampling cost stronger}
For any arm $i\in[k]$, its within-experiment cost function $C_i(\cdot)$ is uniformly bounded away from~$0$ and~$\infty$. 
That is, there exist $c_{\min}, c_{\max}\in(0,\infty)$ such that
\[
c_{\min} \leq C_i(\thetabf) \leq c_{\max}, \quad\forall \thetabf\in\mathbb{R}^k,\,\, \forall i\in[k].
\]
In addition, 
$C_i(\cdot)$ is Lipschitz continuous at~$\thetabf\in\Theta$.
\end{assumption}

  Past theory gives some insight into the inner working of top-two Thompson sampling. It suggests information balance~\eqref{eq:info-balance-gaussian} is somehow attained \emph{automatically} through posterior sampling. Intuitively, if the weight of evidence against a suboptimal arm is especially large, that arm is very unlikely to be chosen as the challenger $I_t^{(2)}$. As information about the other arms accumulates, information imbalances correct.  This result leaves open how to adjust the long-run exploitation rate; one can interpret \eqref{eq:cost-aware-IDS} as a resolution to that problem.
  \begin{remark}[\cite{russo2020simple, shang2020fixed}] 
  \label{remark:TTTS beta}
  Suppose top-two Thompson sampling (Algorithm \ref{alg:ttts}) is applied with a coin bias $h_t$ that converges almost surely to some fixed $\beta>0$. Then, with probability 1, ${\bm p}_{t} \to {\bm p}^{(\beta)}$ as $t\to \infty$, where ${\bm p}^{(\beta)} = \left(p^{(\beta)}_1,\ldots,p^{(\beta)}_k\right)$ is the unique probability vector satisfying $p^{(\beta)}_{I^*} = \beta$ and
  	\begin{equation}\label{eq:preview-info-balance}
  		\frac{ \theta_{I^*} - \theta_{i}}{ \sigma^2/\beta   + \sigma^2/p^{(\beta)}_{i}} = \frac{ \theta_{I^*} - \theta_{j}}{ \sigma^2/\beta   + \sigma^2/p^{(\beta)}_{j}},  \quad \forall  i,j \neq I^*.
  	\end{equation}
  \end{remark}

\subsection{Tracing the Pareto frontier between length and regret by adjusting the exploitation rate}
\label{subsec:frontier}

We have seen in Theorem~\ref{thm:efficient-p-gaussian} that the optimal long-run proportions depend on the particulars of the cost-functions only through a single scalar: the optimal long-run exploitation rate. This structure makes it easy to study the tradeoffs between different measures of performance. 

As in Section~\ref{subsec:lenght-regret-objectives}, our presentation here focuses on the tension between two important performance measures: the cumulative regret in treatment decisions applied to the population and the length of the experimentation phase, which are induced by a special case of our model in which for any treatment arm $i\in[k]$, $C_{i}(\thetabf) = c+(\theta_{I^*}-\theta_{i})$ and  $\Delta_{i}(\thetabf)=\theta_{I^*}-\theta_{i}$ with $c > 0$ penalizing long experiment.
With Lemma \ref{lem:consistency independent of cost functions}, we define the (asymptotic) attainable region as follows:
\begin{definition}[Attainable region]
\label{def:attainable region}
	Fix some $\thetabf \in \Theta$. A (normalized) length-regret pair $(L,R) \in \mathbb{R}^2$ is asymptotically attainable if there exists some policy $\pi \in \Pi$ satisfying 
	\[ 
	 \limsup_{n\to \infty} \frac{\mathrm{Length}_{\thetabf}(n,\pi)}{\ln(n)} \leq L \quad \mathrm{and} \quad  \limsup_{n\to \infty} \frac{\mathrm{Regret}_{\thetabf}(n,\pi)}{\ln(n)} \leq R.
	\]
	The set of all attainable length-regret pairs is denoted $\mathcal{F}_{\thetabf}$.
\end{definition}
For interpretation of the attainable region, it is important that the set of consistent policies $\Pi$, introduced in Definition~\ref{def:uniformly good rule}, does not depend on the particular cost functions. As a result, the set $\mathcal{F}_{\thetabf}$ characterizes which pairs are attainable, independently of whether they are desirable under specific cost functions. This is ensured by the next lemma.
\begin{lemma}[Consistency is determined independently of the cost functions]
	\label{lem:consistency independent of cost functions}
	A policy $\pi$ is an element of $\Pi$ if and only if, for each fixed $\thetabf \in \Theta$, as $n\to \infty$, 
	\[
	\ln\left(\E_{\thetabf}^{\pi}\left[ \tau  \right]\right) = o_{\thetabf}(\ln(n))
	\quad\text{and}\quad
	\mathbb{P}^{\pi}_{\thetabf}\left( \hat{I}_{\tau} \neq I^*(\thetabf) \right)= o_{\thetabf}\left(\frac{\ln(n)}{n}\right).
	\]
\end{lemma}

Figure \ref{fig:Pareto-frontier} displays the (asymptotic) attainable region. The  Pareto frontier between these measures the fundamental tradeoffs between experiment length and regret across the population.  
The tradeoff between the two metrics is surprisingly mild, reflected in the nearly rectangular shape of the feasible region.
Varying the exploiting rate (or coin bias) in top-two Thompson sampling provides one way to trace the Pareto frontier. 
This can be done either by adjusting the fixed coin bias, as in~\cite{russo2020simple}, or applying~\eqref{eq:cost-aware-IDS}  for a length-regret objective~\eqref{eq:length-regret-objective} and varying the parameter $c>0$.
Appendix~\ref{app:frontier} provides details on how to plot the attainable region and Pareto frontier. More specifically, refer to Proposition \ref{prop:frontier}.

\begin{figure}[ht]
\centering
	\caption{The set of length-regret outcomes that are attainable by a consistent policy. The plot considers the 6-arm instance $\thetabf=(\theta_1,\theta_2,\theta_3,\theta_4,\theta_5,\theta_6)=(0,0.2,0.4,0.6,0.8,1)$. All points on the Pareto frontier are attained by adjusting the exploitation rate $\beta$ of top-two Thompson sampling. }
	\label{fig:Pareto-frontier}
	\includegraphics[width=3in]{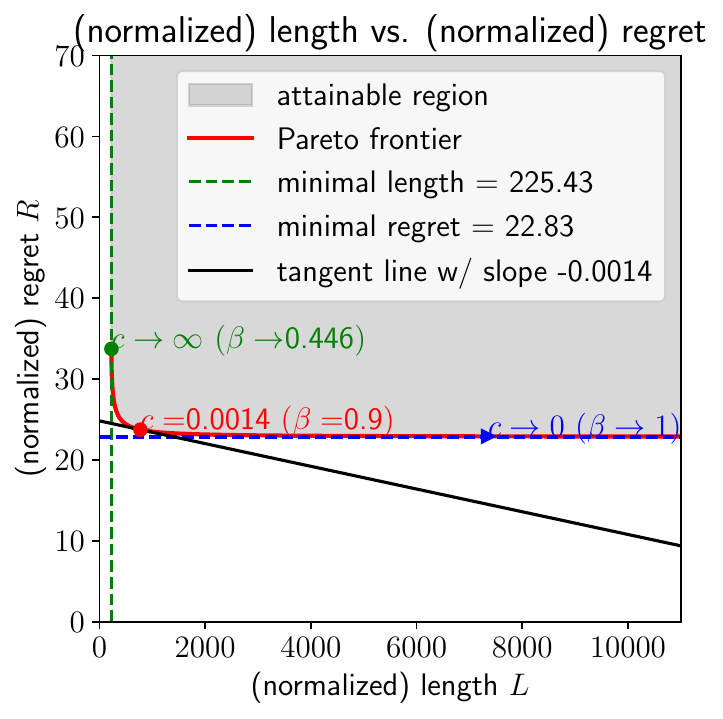}
\end{figure}

Theoretical results give some insights into the somewhat benign tradeoff between length and regret and the nearly rectangular shape of the feasible region. Denote the minimal attainable (normalized) length and regret by 
\[ 
L^*_{\thetabf} \triangleq \inf\{L : (L,R) \in \mathcal{F}_{\thetabf} \}  \qquad \text{and} \qquad  R^*_{\thetabf} \triangleq \inf\{R : (L,R) \in \mathcal{F}_{\thetabf} \}. 
\]
On Figure \ref{fig:Pareto-frontier}, $L^*_{\thetabf}$ is the leftmost boundary of the feasible region, and $R^*_{\thetabf}$ is the bottom boundary. The next lemma shows that approaching minimal regret requires infinite normalized length. Intuitively, stopping and committing to deploy a treatment always leads to \emph{some} degradation in expected treatment quality, since it forgoes the possibility of adapting based on future information.   

\begin{lemma}[Length cost of attaining optimal regret]
\label{lem:infinite_length} 
Define $L^*_{\thetabf}(R) \triangleq \inf\{ L' : (L', R') \in \mathcal{F}_{\thetabf}, \, R'\leq R \}$. Then 
\[
	\lim_{R\downarrow R^*_{\thetabf}} L^*_{\thetabf}(R) = \infty.
\] 
\end{lemma}
The next proposition reveals that the price of early commitment is mild. For instance applying this result at $\beta = 0.9$ shows that a 10\% increase in normalized regret is associated with an enormous reduction in normalized length --- from the infinite value in Lemma \ref{lem:infinite_length} to a finite one. Plugging in $\beta = 1/2$ shows that it is possible to run an experiment that is no more than twice as long as the shortest possible and incurs regret no more than twice as large as the minimal. 
\begin{proposition}
\label{prop:Pareto_robustness}
	For any $\beta \geq 1/3$,
 there exits an attainable point $\left(L^{(\beta)}_{\thetabf}, R^{(\beta)}_{\thetabf}\right) \in \mathcal{F}_{\thetabf}$ such that 
	\[
	L^{(\beta)}_{\thetabf} \leq \frac{L^*_{\thetabf}}{1-\beta}  \qquad \text{and} \qquad R^{(\beta)}_{\thetabf} \leq \frac{R^*_{\thetabf}}{\beta}.
	\]
\end{proposition}
Our proofs go further, and also reveal which algorithms attain the point $\left(L^{(\beta)}_{\thetabf}, R^{(\beta)}_{\thetabf}\right)$. Algorithm~\ref{alg:general-template-gaussian} will attain that point if applied with an allocation rule under which the empirical allocation converges strongly to the proportions $\bm{p}^{\beta}$ identified in Remark~\ref{remark:TTTS beta}. For example, top-two Thompson sampling with fair coin (bias $h_t=1/2$) simultaneously incurs regret no more than $2R^*_{\thetabf}$ and length no more than $2L^*_{\thetabf}$ for $\thetabf\in \widetilde{\Theta}$. This follows from Lemma \ref{lem:length-regret-under-pi-beta}, given in Appendix \ref{app:frontier}, together with \citet{shang2020fixed}.

\subsection{Technical subtleties and almost sure vs. `strong' convergence}\label{subsec:strong-convergence}

A segment of the literature focuses on establishing almost sure convergence of ${\bm p}_{t}$ to ${\bm p^*}$. Unfortunately, almost sure convergence does not guarantee any reasonable scaling of total \emph{expected} costs. 
\begin{example}[Failure of almost sure convergence]\label{ex:failure-of-almost-sure}
	Consider the following randomized allocation rule. Before the experiment begins, it samples a random time $T$. For the initial $T$ individuals, it samples arm 1.  Those observations are then discarded and ignored. Starting with individual $T+1$, the experiment proceeds according to Algorithm \ref{alg:general-template-gaussian} with top-two Thompson sampling (Algorithm \ref{alg:ttts}) serving as the allocation rule. One can show that ${\bm p}_t \to  {\bm p}^*$ almost surely. (It's as if the first $T$ samples never occurred.)

	But if $T$ has Pareto-type tails with $\alpha < 1$: $\Prob(T>x) = x^{-\alpha}$ for $x\geq 1$, then ${\rm Cost}_{\thetabf}(n, \pi) = \Omega(n^{1-\alpha})$, since expected costs are lower bounded as
	\[
	{\rm Cost}_{\thetabf}(n, \pi)  \geq C_{\min}(\thetabf)\cdot\E\left[T \wedge n\right] \geq  C_{\min}(\thetabf)\cdot n \cdot \Prob\left(T>n\right) = n^{1-\alpha},
	\]
	where $C_{\min}(\thetabf)=\min_{i \in [k]}C_{i}(\thetabf)$ is strictly positive by Assumption \ref{asm:sampling cost}. 
\end{example}
We define a custom notion of convergence for random variables which rules out the degenerate behavior in the example above. To understand this, it is helpful to start with the usual definition of almost sure convergence. For a sequence of random variables $\{X_\ell\}_{\ell\in\mathbb{N}_1}$ on a probability space $(\Omega, \Fc, \mathbb{P})$, we say $X_\ell \to x$ almost surely if 
\[
\Prob\left(  \omega :   \lim_{\ell\to\infty} X_\ell(\omega) =x    \right)=1.
\]
That is, with probability 1, for any $\epsilon>0$ there exists  $L(\omega)<\infty$ such that  $|X_{\ell}(\omega) - x| \leq \epsilon$ holds for every $\ell\geq L(\omega)$. Eventually, after the random amount of time $L$, the random variables $X_\ell$ stay in the neighborhood of $x$. The issue in Example \ref{ex:failure-of-almost-sure} was that the expected time one needs to wait could be infinite (that is, one could have $\E\left[L\right]=\infty$). To bound quantities like the expected stopping time of our algorithms, we rely on the following stronger notion of convergence. 
 
Let the space $\mathcal{L}^1$ consist of all measurable random variables $L$ with $\E[|L| ]  <\infty$.

\begin{definition}[Strong convergence]
\label{def: strong convergence}
	For a sequence of real valued random variables $\{X_\ell\}_{\ell\in \mathbb{N}_1}$ and a scalar $x\in \mathbb{R}$,
	we say $X_\ell$ converges strongly to $x$, denoted by $X_\ell\Sto x$, if 
	\[
	\text{for all } \epsilon>0 \, \text{there exists } \, L \in \mathcal{L}^1 \, \text{ such that for all } \ell\geq L,\,\,   |X_\ell - x| \leq \epsilon.
	\]
	For a sequence of random vectors $\{\bm{X}_\ell\}_{\ell\in \mathbb{N}_1}$ taking values in $\mathbb{R}^d$ and a vector $\bm{x}\in \mathbb{R}^d$, we say $\bm{X}_\ell$ converges strongly to  $\bm{x}$, denoted by $\bm{X}_\ell\Sto \bm{x}$, if for all $i\in [d]$,  $X_{\ell,i}$ converges strongly to  $x_{i}$ .
\end{definition}

\section{Generalization to exponential family distributions }

This section generalizes insights derived in the Gaussian case. 
The Lai-Robbins formula in Theorem~\ref{thm:Lai-Robbins-type formula} holds without modification. Again, the optimal allocation of experimental effort is mostly determined by a purely statistical \emph{information-balance} constraint. A scalar variable, the long-run \emph{exploitation rate} is what must vary depending on the experimenter's goal (i.e. the cost functions). The substantial change is that the $Z$-statistics used in the Gaussian case need to be replaced with a more subtle notion of statistical separation between reward distributions.

\subsection{A similar algorithm template}

To address problems in which rewards are drawn from general one-dimensional exponential family distributions, we slightly modify Algorithm  \ref{alg:general-template-gaussian}. Rather than stop when certain $Z$-statistics are uniformly large, as in Algorithm  \ref{alg:general-template-gaussian},  under Algorithm \ref{alg:general-template-exp-family} the experimenter stops when  Chernoff-information statistics are uniformly large.  The form of this stopping rule is the same as a generalized likelihood ratio statistic rule used in \citep{chernoff1959sequential, chan2006sequential}, with a threshold that is derived based on  \cite{Kaufmann2021martingale}.

\begin{algorithm}[H]
	\centering
	\caption{General template}\label{alg:general-template-exp-family}
	\begin{algorithmic}[1]
		\State{\bf Input:}  $\mathtt{AllocationRule}(\cdot)$, a function which takes an input a history $\mathcal{H}_t$ of arbitrary length and returns a probability distribution over $[k]$.
		\State {\bf Initialize:} $\mathcal{H}_0 \gets \{ \}$, $\tau=n$. 
		\For{$t=0,1,\ldots, n-1$}       
		\If{$t < \tau$}
            \LineComment{\emph{\textcolor{blue}{If experimentation has not yet stopped, ...}}}
		\State Obtain $\hat{I}_t \in \argmax_{i \in [k]} m_{t,i}$ and $\gamma_t(n)$ as in \eqref{eq:tight threshold T}.
		\If {$t \cdot \min_{j\neq \hat{I}_t} D_{{\bm m}_t , \hat{I}_t,j}({p}_{t,\hat{I}_t}, {p}_{t,j}) \geq   \gamma_t(n)$}
            \LineComment{\emph{\textcolor{blue}{If all Chernoff information statistics are large, stop experimenting. }}}
		\State Set $\tau=t$ and $\hat{I}_{\tau} \in \argmax_{i \in [k]} m_{\tau,i}$.
            \State Assign arm $I_t = \hat{I}_{\tau}$.
            \Else 
            \LineComment{\emph{\textcolor{blue}{Otherwise, continue experimenting using the allocation rule.}}} 
		\State Assign arm $I_t \sim \mathtt{AllocationRule}(\mathcal{H}_t)$.
		\State Observe $Y_{t,I_t}$ and update history $\mathcal{H}_{t+1} \gets \mathcal{H}_t\cup \{(I_t, Y_{t,I_t}) \}$.
            \EndIf
            \Else
            \LineComment{\emph{\textcolor{blue}{If experimentation has stopped, ...}}}
		\State Assign arm $I_t = \hat{I}_{\tau}$.
		  \EndIf
		\EndFor 
	\end{algorithmic}
\end{algorithm}

\subsection{Chernoff information as a generalization of $Z$-statistics}

We define a weighted version of the Chernoff information between, a fundamental quantity in the theory of hypothesis testing \citep{cover2006elements} which as appeared previously in the pure exploration literature \citep{garivier2016optimal,russo2020simple}.  Consider an instance $\thetabf$ and its two entries $\theta_i \geq \theta_j$.
For any weights $p_{i},p_j\geq0$, define
\begin{align}
	D_{\thetabf, i,j}(p_{i}, p_j) &\triangleq \min_{\vartheta_j \geq \vartheta_i}  \, p_{i} \cdot {\rm KL}\left(\theta_{i} , \vartheta\right) + p_j \cdot {\rm KL}\left(\theta_{j} , \vartheta\right) \label{eq:chernoff-info}\\ 
	&= p_{i} \cdot {\rm KL}(\theta_{i} ,  \bar{\theta}_{i,j}) + p_j \cdot {\rm KL}(\theta_{j} , \bar{\theta}_{i,j}),\label{eq:chernoff-info-minimizer}
\end{align}
where the weighted average
\begin{equation}
\label{eq:weighted average}
\bar{\theta}_{i,j} \triangleq \frac{p_{i} \theta_{i}+ p_j\theta_j}{ p_{i}+p_j}.
\end{equation}
is the unique minimizer to~\eqref{eq:chernoff-info} if $\max\{p_i,p_j\} > 0$. We let $\bar{\theta}_{i,j}\triangleq \theta_j$ if $p_i=p_j=0$.

This weighted version of the Chernoff information measures the strength of evidence distinguishing $\theta_{i}$ and $\theta_j$ under the weights $p_{i}$ and $p_j$. 
It is hardest to rule out the alternative state of nature with both arms $i$ and $j$ having mean $\bar{\theta}_{i,j} \in [\theta_j, \theta_{i}]$.
\begin{example}[Connection to Z-statistics]\label{ex:chernoff-to-z}
	If observations are Gaussian with $p(y\mid\theta_i) =\mathcal{N}(\theta_i, \sigma^2)$, then 
	\[ 
	D_{\thetabf, i,j}(p_{i}, p_j) = \frac{ (\theta_{i}-\theta_j)^2 }{2\sigma^2 \left( \frac{1}{p_{i}}  + \frac{1}{p_{j}}\right)}.
	\]
	In particular, for the empirical mean reward ${\bm m}_t$ and empirical allocation $\bm{p}_t$,
	\[
	D_{{\bm m}_t , i,j}({p}_{t,i}, {p}_{t,j}) =   \frac{ (m_{t,i}-m_{t,j})^2 }{2\sigma^2 \left( \frac{1}{p_{t,i}}  + \frac{1}{p_{t,j}}\right)} =   \frac{Z_{t, i, j}^2}{2t},
	\]
 where $Z_{t, i, j}$ is the Z-statistic for the mean difference defined in~\eqref{eq:z-stats}.
\end{example}

\subsection{Properties of an efficient allocation: information-balance with a cost-aware exploitation rate}

The subsection generalizes Theorem~\ref{thm:efficient-p-gaussian} to one-dimensional exponential family distributions. Again we show The information balance condition introduced in~\eqref{eq:info-balance-general} below generalizes the one for Gaussian distributions in~\eqref{eq:info-balance-gaussian}. 
As before, this purely statistical constraint largely determines the optimal limiting proportions ${\bm p}^*$. For any fixed exploitation rate $\beta\in[0,1]$, information balance condition~\eqref{eq:info-balance-general} uniquely determines the residual exploration rates $\left(p_j^{(\beta)}\right)_{j\neq I^*}\in \mathbb{R}^{k-1}$. Optimal performance requires adjusting the exploitation rate $\beta$ based on the within-experiment cost functions so that (cost-aware) exploitation rate condition~\eqref{eq:exploitation-rate-general} holds. Denote the optimal exploitation rate by $\beta^*$, and then the optimal limitation proportions $\bm{p}^* = (p^*_1,\ldots,p^*_k)$ has entries: $p^*_{I^*} = \beta^*$ and $p^*_{j} = p_{j}^{(\beta^*)}$ for any $j\neq I^*$.

\begin{restatable}[Optimality condition of allocation rules]{Theorem}{SufficientCondition}
\label{thm:efficient-p-general} 
Algorithm~\ref{alg:general-template-exp-family} is universally efficient (Definition \ref{def:universal-efficiency}) if 
the allocation rule satisfies the following condition:
for any $\thetabf\in\Theta$, under any produced indefinite-allocation sample path~\eqref{eq:allocaiton-only-sample}, the empirical allocation $\bm{p}_t$ converges strongly (Definition \ref{def: strong convergence}) to a unique probability vector $\bm{p}^* = \left(p_1^*,\ldots,p_k^*\right) > \bm{0}$ satisfying the information balance condition
	\begin{equation}
 \label{eq:info-balance-general}
		 D_{\thetabf, I^*, i}(p^*_{I^*}, p_i^*) = D_{\thetabf, I^*, j}(p^*_{I^*}, p_j^*), \quad \forall  i,j \neq I^*
	\end{equation}
and (cost-aware) exploitation rate condition
\begin{equation}
	\label{eq:exploitation-rate-general}
	\sum_{j\neq I^*}\frac{{\rm KL}(\theta_{I^*},\bar{\theta}^*_{I^*,j})/C_{I^*}(\thetabf)}{{\rm KL}(\theta_{j},\bar{\theta}^*_{I^*,j})/C_j(\thetabf)} = 1
 \quad\text{with}\quad
 \bar{\theta}^*_{I^*,j} \triangleq \frac{p_{I^*}^*\theta_{I^*} + p^*_j\theta_j}{p_{I^*}^* + p_j^*}.
\end{equation} 
By contrast, Algorithm~\ref{alg:general-template-exp-family} is not universally efficient if there exists $\thetabf\in\Theta$ such that the allocation rule's empirical allocation ${\bm p}_t$ converges strongly to a probability vector other than ${\bm p}^*$.
\end{restatable}

\subsection{Is top-two Thompson sampling still universally efficient?}
Given the developments above, it is not hard to directly construct a universally efficient allocation rule. In Appendix \ref{app:tracking}, we do this by building on the tracking algorithm of \cite{garivier2016optimal}; see Algorithm \ref{alg:D-Tracking} there. Essentially, we use a number of samples that grows slowly in the population size to estimate $\thetabf$, solve for the optimal long-run proportions assuming our estimate were correct, and allocate all future measurements according to this. Similar ideas date back as far as the work of \cite{chernoff1959sequential}.  

It is not difficult to generalize top-two Thompson sampling to non-Gaussian reward distributions. The original paper \citep{russo2020simple} studied the algorithm under exponential family distributions (with general priors) and showed that its long-run measurement proportions satisfy information balance condition. However, two gaps in the existing literature prevent us from easily establishing that asymptotic efficiency in our model: 
\begin{enumerate}
	\item Our theory requires strong convergence to the optimal proportions (Definition \ref{def: strong convergence}) whereas \cite{russo2020simple} established only almost sure convergence. Strong convergence was established  for Gaussian distributions by \cite{shang2020fixed} and bounded distributions (e.g. Bernoulli) by \cite{jourdan2022top}, but not for general exponential family distributions.  
	\item  The simple rule for setting the coin bias is adapted from \cite{qin2023dualdirected}, who provided theory only in the Gaussian case. It's worth pointing out that there is a natural generalization to exponential family distributions which sets 
	 \begin{equation*} 
	h_{t} =  h_{t, I_t^{(1)}, I_{t}^{(2)}}   \quad \text{where} \quad  h_{t,i,j}\triangleq
 \frac{ \frac{p_{t,i}\KL(m_{t,i}, \bar{m}_{t,i,j})}{C_{i}( \bm{m}_t)} }{ \frac{p_{t,i}\KL(m_{t,i}, \bar{m}_{t,i,j})}{C_{i}( \bm{m}_t)} + \frac{p_{t,j}\KL(m_{t,j}, \bar{m}_{t,i,j})}{C_{j}( \bm{m}_t)} }
 \quad\text{with}\quad
 \bar{m}_{t,i,j} \triangleq \frac{p_{t,i}m_{t,i} + p_{t,j}m_{t,j}}{p_{t,i} + p_{t,j}},
	\end{equation*} 
	but it is an open question to provide rigorous convergence guarantees under this choice. One can set the coin bias equal to the estimated optimal exploitation rate, formed solving an empirical analogue of the \eqref{eq:info-balance-general} and \eqref{eq:exploitation-rate-general}. But the formula above seems much simpler. 
\end{enumerate}

\subsection{The Lai-Robbins formula continues to hold}
The next result confirms that the result in Theorem~\ref{thm:Lai-Robbins-type formula} holds even for non-Gaussian distributions.

\begin{restatable}[Lai-Robbins-type formula]{Theorem}{LaiRobbins}
\label{thm:Lai-Robbins-type formula_general}
	A policy $\pi$ is universally efficient if and only if, 
	\begin{equation}
 \label{eq:sufficient and necessary condition of universal efficiency}
 \forall\thetabf\in\Theta:\quad
	\mathrm{Cost}_{\thetabf}(n, \pi) \sim  \sum_{j\neq I^*}  \frac{C_{j}(\thetabf) }{ {\rm KL}(\theta_j , \bar{\theta}^*_{I^*,j})  }  \times \ln(n)   \quad\text{as}\quad n\to\infty,
	\end{equation}
	where $\bar{\theta}^*_{I^*,j}=\frac{p^*_{I^*}\theta_{I^*}+p_j^* \theta_j}{p^*_{I^*}+p^*_j}$ is identified in Theorem~\ref{thm:efficient-p-general}.
\end{restatable}

\section{Characterizing asymptotic efficiency by the ``Skeptic's Standoff'' game}
\label{sec:game}

As discussed in the literature review, the analysis leading to our results in Theorems~\ref{thm:efficient-p-general} and \ref{thm:Lai-Robbins-type formula_general} builds on an approach that dates back to \cite{chernoff1959sequential}. 
We introduce a two-player zero-sum game whose equilibrium characterizes both the limits of attainable costs and the nature of universally asymptotically efficient policies. We call this game  \emph{the Skeptic's Standoff}. 
This two-player, zero-sum, simultaneous-move game occurs between an experimenter and a skeptic. 
The experimenter believes $\thetabf= (\theta_1,\ldots,\theta_k)\in\Theta$ is the true state of nature, and tries to gather convincing evidence that $I^*(\thetabf)$ is the unique optimal treatment arm by
choosing proportional allocation of measurement effort $\bm{p}=(p_1,\ldots,p_k)\in\Sigma_k$; here $\Sigma_k$ denotes the $(k-1)$ dimensional probability simplex. 
The skeptic picks some alternative state of nature $\varthetabf=(\vartheta_1,\ldots,\vartheta_k)$ under which a different arm can be optimal. 
Specifically, the skeptic chooses from the set
\[
\overline{\rm Alt}(\thetabf) \triangleq 
{\rm Closure}\left({\rm Alt}(\thetabf)\right),
\quad\text{where}\quad {\rm Alt}(\thetabf)\triangleq\{\varthetabf\in \Theta \,:\, I^*(\varthetabf) \neq I^*(\thetabf)\}
\]
is the set of parameters under which $I^*(\thetabf)$, the optimal arm under $\thetabf$, is suboptimal; the closure $\overline{\rm Alt}(\thetabf)$ also includes the parameters with multiple optimal arms with $I^*(\thetabf)$ being one of them. 
Since $I^*(\thetabf)$ is the unique optimal arm, it must be strictly optimal, so $\thetabf\notin \overline{\rm Alt}(\thetabf)$.

Roughly speaking, the experimenter hopes to gather evidence that rules out whatever alternative the skeptic raises, and hopes to do so cheaply.  Define the payoff function for the experimenter as
\begin{equation}\label{eq:payoff-function}
	\Gamma_{\bm{\theta}}(\bm{p},\bm{\vartheta}) \triangleq \frac{\sum_{i\in[k]} p_i {\rm KL}(\theta_i,\vartheta_i)}{\sum_{i\in[k]}p_iC_i(\thetabf)},
	\quad \forall (\bm{p},\varthetabf)\in \Sigma_k\times \overline{\rm Alt}(\thetabf).
\end{equation}
Since the game is zero-sum, $-\Gamma_{\bm{\theta}}(\bm{p},\bm{\vartheta})$ is the payoff function for the skeptic.  
The denominator in \eqref{eq:payoff-function} measures the average per-period cost the experimenter incurs during experimentation, while the numerator measures the average amount of discriminative information acquired against the skeptic's alternative. The payoff function measures  bits of discriminative information acquired per unit cost.

To seek a Nash equilibrium of this game, we allow the skeptic to play a mixed strategy
$\bm{q}\in \mathcal{D}\left( \overline{\rm Alt}({\bm \theta}) \right)$, where $\mathcal{D}\left( \overline{\rm Alt}({\bm \theta}) \right)$ denotes the set of distribution over the set of alternative states of nature $\overline{\rm Alt}({\bm \theta})$. We overload the definition of the payoff function to 
\[	
\Gamma_{\bm{\theta}}( {\bm p} ,  {\bm q}  ) = \E_{\bm{\vartheta}\sim \bm{q} }\left[ \Gamma_{\bm{\theta}}(\bm{p},\bm{\vartheta})\right].
\]
A pair of  strategies $(\widetilde{\bm p}, \widetilde{\bm q})\in\Sigma_k\times \mathcal{D}\left( \overline{\rm Alt}({\bm \theta}) \right)$ forms an equilibrium if it satisfies
\[
	\inf_{\bm{q}\in \mathcal{D}\left(\overline{\rm Alt}(\thetabf)\right)}\Gamma_{\thetabf}(\widetilde{\bm p}, \bm{q})
 =\Gamma_{\thetabf}(\widetilde{\bm p}, \widetilde{\bm q})
 =\sup_{\bm{p}\in\Sigma_k}\Gamma_{\thetabf}(\bm{p},\widetilde{\bm q}).
\]

\paragraph{Skeptic's strategy over hard alternative states of nature.}
We introduce a specific mixed strategy $\bm{q}^*$, which is shown to be the skeptic's unique equilibrium strategy. See Theorem~\ref{thm:equilibrium}, which will be presented shortly. The mixed strategy $\bm{q}^*$ is supported on
$(k-1)$ alternative states of nature $\{\varthetabf^{*j}\}_{j\neq I^*}\subset\overline{\rm{Alt}}(\thetabf)$, where $I^*=I^*(\thetabf)$ is the optiaml arm under $\thetabf$. The $j$-th alternative $\varthetabf^{*j}=(\vartheta^{*j}_1,\ldots,\vartheta^{*j}_k)$ satisfies
\begin{equation}
	\label{eq:hard instance}
	\vartheta^{*j}_{I^*} = \vartheta^{*j}_j =  \bar{\theta}^*_{I^*,j} 
	\quad\text{and}\quad
	\vartheta^{*j}_i = \theta_i, \quad\forall i\notin \{I^*,j\},
\end{equation}
where $\bar{\theta}^*_{I^*,j}=\frac{p^*_{I^*}\theta_{I^*}+p_j^* \theta_j}{p^*_{I^*}+p^*_j}$
is the weighted average of $\theta_{I^*}$ and $\theta_j$ with respect to the unique probability vector~$\bm{p}^*$ satisfying information balance condition~\eqref{eq:info-balance-general} and (cost-aware) exploitation rate condition~\eqref{eq:exploitation-rate-general}. Under the alternative $\varthetabf^{*j}$, it is hard for the experimenter to distinguish the treatment arms $I^*$ and $j$.
The probability of playing these alternatives $\varthetabf^{*j}$ is
\begin{equation}
	\label{eq:optimal q}
	\bm{q}^*(\varthetabf^{*j}) \propto \frac{C_j(\thetabf)}{{\rm KL}(\theta_{j},\bar{\theta}^{*}_{I^*,j})};
\end{equation}
the RHS equals the $j$-th term of the optimal scaling in Lai-Robbin-type formula in Theorem~\ref{thm:Lai-Robbins-type formula_general}.  

Now we formally present Theorem~\ref{thm:equilibrium}, which identifies the unique equilibrium strategies $({\bm p}^*, {\bm q}^*)$ and the equilibrium value to the Skeptic's Standoff game.

\begin{Theorem}[Equilibrium]
\label{thm:equilibrium}
	The Skeptic's Standoff game has a unique pair of equilibrium strategies $({\bm p}^*, {\bm q}^*)$, and
the equilibrium value satisfies the formula\footnote{This formula indicates that the probability vector $\bm{q}^*$, defined in \eqref{eq:optimal q}, has the explicit expression: for any $j\neq I^*$,
$\bm{q}^*(\varthetabf^{*j}) = \frac{C_j(\thetabf)}{{\rm KL}(\theta_{j},\bar{\theta}^{*}_{I^*,j})}\Gamma_{\bm{\theta}}(\bm{p}^*,\bm{q}^*)$.} 
	\begin{equation}
 \label{eq:equilibrium value's formula}
	\frac{1}{\Gamma_{\bm{\theta}}(\bm{p}^*,\bm{q}^*)} =  \sum_{j\neq I^*} \frac{C_j(\thetabf)}{ {\rm KL}(\theta_j , \bar{\theta}^*_{I^*,j})  },
 \quad\text{where}\quad \bar{\theta}^*_{I^*,j}=\frac{p^*_{I^*}\theta_{I^*}+p_j^* \theta_j}{p^*_{I^*}+p^*_j}.
	\end{equation}
\end{Theorem}

Theorem~\ref{thm:equilibrium} provides a new interpretation of some of our most important results.  
Since the experimenter's unique equilibrium strategy in the game corresponds to the 
optimal long-run sampling proportions $\bm{p}^*$ identified in Theorem~\ref{thm:efficient-p-general}, 
we can restate Theorem~\ref{thm:efficient-p-general} as follows. 
\begin{corollary}[Restatement of Theorem~\ref{thm:efficient-p-general}]
	  Algorithm  \ref{alg:general-template-exp-family} is universally efficient (Definition \ref{def:universal-efficiency}) if the input allocation rule satisfies the following property: for any $\thetabf\in\Theta$, the empirical allocation ${\bm p}_t$ converges strongly (Definition \ref{def: strong convergence}) to the experimenter's unique equilibrium strategy~${\bm p}^*$.
   
   By contrast, Algorithm~\ref{alg:general-template-exp-family} is not universally efficient if there exists $\thetabf\in\Theta$ such that the allocation rule's empirical allocation ${\bm p}_t$ converges strongly to a probability vector other than the experimenter's unique equilibrium strategy~${\bm p}^*$.
\end{corollary}

Theorem~\ref{thm:equilibrium} also shows that the equilibrium value of the game is equal to the Lai-Robbins-style formula in Theorem~\ref{thm:Lai-Robbins-type formula_general},
so we can reinterpret Theorem~\ref{thm:Lai-Robbins-type formula_general} as follows.
\begin{corollary}[Restatement of Theorem~\ref{thm:Lai-Robbins-type formula_general}]
		A policy $\pi$ is universally efficient if and only if,
	\[
 \forall\thetabf\in\Theta:\quad
	\mathrm{Cost}_{\thetabf}(n, \pi) \sim    \frac{1}{\Gamma_{\bm{\theta}}(\bm{p}^*,\bm{q}^*)}  \times \ln(n)
 \quad\text{as}\quad n\to\infty.
	\]
\end{corollary}

Our analysis in the appendix starts with the proof of Theorem~\ref{thm:equilibrium} and subsequently proves these restated results of Theorems~\ref{thm:efficient-p-general} and \ref{thm:Lai-Robbins-type formula_general}. The alternative presentation in the main paper is meant to allow a reader to understand the conceptual takeaways without digesting the two-player game.  

\paragraph{Novelty in the equilibrium analysis.} 
A lot of the core insights in the paper follow from the explicit characterization of the equilibrium above. Relative to the literature, one uncommon feature is that we also characterize the skeptic's equilibrium strategy ${\bm q}^*$. 
A more common approach in the best-arm identification literature only allows the skeptic to play pure strategies \citep{garivier2016optimal}, in which case no equilibrium exists.  
Subsequent papers \citep{degenne2019MultipleCorrectAnswers, degenne2019Non-Asymptotic, degenne2020structure} adopt a similar game-theoretic view as ours, but do not characterize ${\bm q}^*$.
Identifying the skeptic's equilibrium strategy ${\bm q}^*$ is crucial to overcoming technical challenges in proving our lower bound (see Remark \ref{rem:lower_bound_proof} in Appendix \ref{app:lower bound proof}) and to deriving the form of the equilibrium value in Theorem~\ref{thm:equilibrium}, which ties the connection to a famous formula of \cite{lai1985asymptotically}.

\section{Conclusion} 
This paper offers a model and analysis that unifies and generalizes two strands of the multi-armed bandit literature: pure regret minimization problems in the style of  \cite{lai1985asymptotically} and (fixed-confidence) best-arm identification problems in the style of \cite{garivier2016optimal}. Beyond the unification, the theory seems to offer interesting insights about the nature of asymptotically efficient policies and the tradeoffs between experiment length and total regret. It would be interesting to generalize the asymptotic results to multi-parameter exponential family distributions, and to relax the conditions needed to analyze top-two sampling algorithms. In addition, there is lots of room to design algorithms with desirable finite-time performance in this problem.

\newpage
\singlespacing
{\footnotesize 
	\setlength{\bibsep}{2pt plus 0.4ex}
	\bibliographystyle{plainnat}
	\bibliography{references}
}

\newpage
\appendix

\input{app_outline}

\input{app_implementation_details}
\input{app_p_properties}
\input{app_games_and_equilibrium_structures}

\input{app_lower_bound}
\input{app_if_allocation}
\input{app_iff_LaiRobbins}

\input{app_onlyif_allocation}

\input{app_TTTS}

\input{app_Pareto_frontier_New}

\end{document}

%% file: app_outline.tex
\section*{Outline of appendix}

The appendix is structured as follows.

\begin{enumerate}
    \item Appendix~\ref{app:implementation details} provides the implementation details for Section~\ref{sec:teaser}.
    \item Appendix~\ref{app:properties of p*} proves the existence and properties of probability vector $\bm{p}^*$. The uniqueness of $\bm{p}^*$ is needed for the proof of unique equilibrium in Theorem~\ref{thm:equilibrium}.
Both uniqueness and strict positivity of $\bm{p}^*$ is required for showing Theorems~\ref{thm:efficient-p-general} and~\ref{thm:Lai-Robbins-type formula_general}. 
    \item Appendix~\ref{app:equilibrium verification} proves that $(\bm{p}^*,\bm{q}^*)$ forms an equilibrium and the equilibrium value's formula in Theorem~\ref{thm:equilibrium}.
    \item Appendix~\ref{app:unique equilibrium proof} shows that $(\bm{p}^*,\bm{q}^*)$ forms the unique equilibrium, stated in Theorem~\ref{thm:equilibrium}.
    \item Appendix~\ref{app:lower bound proof} completes the proof of lower bound in Proposition~\ref{prop:lower bound}.
    \item Appendix \ref{app:proof of sufficient condition} proves the first statement in Theorem~\ref{thm:efficient-p-general}.
    \item Appendix~\ref{app:tracking satisfies sufficient condition} establishes Theorem~\ref{thm:Lai-Robbins-type formula_general} by constructing an allocation rule satisfying the sufficient condition in Theorem~\ref{thm:efficient-p-general}.
    \item Appendix~\ref{app:proof of sufficient condition being almost necessary} completes the proof of Theorem~\ref{thm:efficient-p-general} by showing the second statement therein.
    \item Appendix~\ref{app:TTTS} proves the optimality of TTTS in Proposition~\ref{prop:TTTS}.
    \item Appendix~\ref{app:frontier} completes the proofs of the results presented in Section~\ref{subsec:frontier}.
\end{enumerate}

%% file: app_implementation_details.tex
\section{Implementation details for Section \ref{sec:teaser}}
\label{app:implementation details}
When the posterior becomes concentrated on a single arm, the resampling step in top-two TS can be time-consuming. To address this, we leverage the posterior matching property of TS. Specifically, we compute the probability that the a given arm is optimal. This is achieved by integrating the product of the cumulative distribution functions of all other arms over the density of the specified arm, using Gauss-Hermite quadrature for numerical integration. To implement top-two TS, we transform the sampling distribution of TS to that of top-two TS according to the formula in \citet[Subsection 3.4]{russo2020simple}.
To further improve the computational efficiency of top-two TS, we implement a slightly modified version of top-two TS where the posterior distribution is updated after every 500 observations, instead of after each individual observation. This batching approach reduces the frequency of posterior updates.

In these numerical experiments, we apply both top-two TS and Epsilon-Greedy with a simple stopping rule based on $Z$-statistics defined in \eqref{eq:z-stats} and a heuristic stopping threshold. Specifically, the stopping criterion is triggered when $\min_{j\neq \hat{I}_t} Z_{t,\hat{I}_t,j}$ is larger than the $1-\frac{1}{n(k-1)}$ quantile of standard normal distribution. We record (length, total regret) of each method for 10,000 trials, where the population size $n=100,000,000$ and the number of arms $k=6$; recall $\hat{I}_t$ is an empirical best arm at time $t$. Figures \ref{fig:teaser_pareto} and \ref{fig:teaser_play_counts} are produced by averaging the results for $10^4$ trials. In Figure \ref{fig:teaser_pareto}, the exploitation rate $\beta$ in top-two TS ranges from $0.1$ to $0.99$, and $\epsilon$ in Epsilon-Greedy ranges from $0.05$ to $1$. The theoretical Pareto dominated points and the theoretical Pareto frontier plot the pairs (normalized length \eqref{eq:NL}$\times\ln(n)$, normalized regret \eqref{eq:NL}$\times\ln(n)$), with $\beta$ ranging from $0.1$ and $0.99$. For the specific instance $\thetabf=(0,0.2,0.4,0.6,0.8,1)$, the pair (normalized length$\times\ln(n)$, normalized regret$\times\ln(n)$), for $\beta = 0.446$, is the leftmost point on the Pareto frontier. This value of $\beta$ corresponds to the optimal exploitation rate for the best arm identification problem.


%% file: app_p_properties.tex
\section{Proof of existence and properties of $\bm{p}^*$}
\label{app:properties of p*}

This appendix formally proves the existence and properties of probability vector $\bm{p}^*$ satisfying information balance condition~\eqref{eq:info-balance-general} and (cost-aware) exploitation rate condition~\eqref{eq:exploitation-rate-general} in Lemma~\ref{lem:properties of p*}, which is restated as follows.

\begin{restatable}[Existence and properties of $\bm{p}^*$]{lemma}{UniquenessPositivity}
\label{lem:properties of p*}
Let $\thetabf\in\Theta$. There exists a unique probability vector $\bm{p}^* = (p^*_1,\ldots,p^*_k)$ satisfying information balance condition~\eqref{eq:info-balance-general} and exploitation rate condition~\eqref{eq:exploitation-rate-general}.
Furthermore, the entries of $\bm{p}^*$ are strictly positive, i.e.,  $p^*_i > 0$ for any $i\in [k]$.
\end{restatable}

Our proof follows the analysis on \citet[page~5]{garivier2016optimal}. While \citet[page~5]{garivier2016optimal} study the properties of the solution to a max-min optimization problem rather than a set of explicit equations, their analysis indicates that the solution must satisfy two equations similar to~\eqref{eq:info-balance-general} and~\eqref{eq:exploitation-rate-general}.

Now we start proving Lemma~\ref{lem:properties of p*}. We first prove $p^*_{I^*} > 0$ by contradiction. Suppose $p^*_{I^*} = 0$. Since $\bm{p}^*$ is a probability vector, there exists some $j'\neq I^*$ such that $p^*_{j'} > 0$, which implies $\bar{\theta}^*_{I^*,j'} = \theta_{j'}$. Hence, $\KL(\theta_{j'},\bar{\theta}^*_{I^*,j'}) = \KL(\theta_{j'},\theta_{j'}) = 0$, and thus the LHS in~\eqref{eq:exploitation-rate-general} becomes infinity. This leads to a contradiction, and therefore $p^*_{I^*} > 0$.

Dividing information balance condition \eqref{eq:info-balance-general} by ${p}^*_{I^*}>0$ gives
\[
{\rm KL}(\theta_{I^*} ,  \bar{\theta}^*_{I^*,i}) + \frac{p^*_i}{p^*_{I^*}} {\rm KL}(\theta_{i} , \bar{\theta}^*_{I^*,i})
={\rm KL}(\theta_{I^*} ,  \bar{\theta}^*_{I^*,j}) + \frac{p^*_j}{p^*_{I^*}} {\rm KL}(\theta_{j} , \bar{\theta}^*_{I^*,j}),
\quad \forall i,j\neq I^*,
\]
where $\bar{\theta}^*_{I^*,j} = \frac{p^*_{I^*}\theta_{I^*} + p^*_{j}\theta_{j}}{p^*_{I^*} + p^*_{j}} = \frac{\theta_{I^*}+\frac{p^*_j}{p^*_{I^*}}\theta_j}{1+\frac{p^*_j}{p^*_{I^*}}}$ for any $j\neq I^*$.
This can be rewritten as
\[
g_i\left(\frac{p^*_i}{p^*_{I^*}}\right) = g_j\left(\frac{p^*_j}{p^*_{I^*}}\right),\quad \forall i,j\neq I^*,
\]
where the function $g_j$ is defined as follows,
\[
g_j(x) \triangleq \KL\left(\theta_{I^*}, \frac{\theta_{I^*} + x\cdot\theta_j}{1+x}\right) + x \cdot \KL\left(\theta_{j}, \frac{\theta_{I^*} + x\cdot\theta_j}{1+x}\right), \quad\forall x\geq 0.
\]
\citet[page~5]{garivier2016optimal} observe that $g_j$ is a strictly increasing mapping from $[0,\infty)$ to $[0, \KL(\theta_{I^*},\theta_j))$, so its inverse function $x_j: [0, \KL(\theta_{I^*},\theta_j))\mapsto [0,\infty)$ such that $x_j(y) \triangleq g_j^{-1}(y)$ is well-defined. Furthermore define the function $x_{I^*}(y)\equiv 1$ for any $y\geq 0$.

The existence and uniqueness of $\bm{p}^*$ follows immediately from the following variant of \citet[Theorem 5]{garivier2016optimal}. Observe that the equation $F_{\thetabf}(y)=1$ in this result corresponds to exploitation rate condition~\eqref{eq:exploitation-rate-general}. 

\begin{lemma*}[A variant of {\citet[Theorem 5]{garivier2016optimal}}]
Fix $\thetabf\in\Theta$. 
For any $i\in[k]$,
\begin{equation}
\label{eq:optimal p exact}
p^*_i = \frac{x_i(y^*)}{\sum_{j\in[k]} x_j(y^*)},
\end{equation}
where $y^*$ is the unique solution of the equation $F_{\thetabf}(y)=1$,
with
\[
F_{\thetabf}:y \to \sum_{j\neq I^*}  \frac{C_j(\thetabf)}{C_{I^*}(\thetabf)}\frac{{\rm KL}\left(\theta_{I^*},\frac{\theta_{I^*}+x_j(y)\cdot\theta_j}{1+x_j(y)}\right)}{{\rm KL}\left(\theta_{j},\frac{\theta_{I^*}+x_j(y)\cdot\theta_j}{1+x_j(y)}\right)}
\]
being a continuous and strictly increasing function on $\left[0, \min_{j\neq I^*}\KL(\theta_{I^*}, \theta_j)\right)$ such that $F_{\thetabf}(0)=0$ and $F_{\thetabf}(y)\to\infty$ when $y\to \min_{j\neq I^*}\KL(\theta_{I^*}, \theta_j)$.
\end{lemma*}
This result also implies the strict positivity of $\bm{p}^*$ as follows. The unique solution $y^*$ is strictly positive since the function $F_{\thetabf}$ is strictly increasing with $F_{\thetabf}(0)=0$.
Hence, for $j\neq I^*$, $x_j(y^*) = g_j^{-1}(y^*) > 0$. Recall that $x_{I^*}(y^*) = 1 > 0$ as well. Therefore, it follows from \eqref{eq:optimal p exact} that for any $i\in[k]$, $p^*_i > 0$.
This completes the proof of Lemma~\ref{lem:properties of p*}.

\begin{remark}
    Since \cite{garivier2016optimal} focus on the best-arm identification problem, in the original definition of $F_{\thetabf}$ in \citet[Theorem~5]{garivier2016optimal}, the ratio of costs $\frac{C_j(\thetabf)}{C_{I^*}(\thetabf)}=1$ for any $j\neq I^*$. It is easy to verify that the proof of \citet[Theorem~5]{garivier2016optimal} readily extends to the above variant with general ratios of costs $\left\{\frac{C_j(\thetabf)}{C_{I^*}(\thetabf)}\right\}_{j\neq I^*}$.
\end{remark}

%% file: app_games_and_equilibrium_structures.tex
\section{Proof of equilibrium and equilibrium value's formula in Theorem \ref{thm:equilibrium}}
\label{app:equilibrium verification}
This appendix proves that $(\bm{p}^*,\bm{q}^*)$ is a pair of equilibrium strategies to the Skeptic's Standoff game
by verifying that the skeptic's mixed strategy $\bm{q}^*$ is a best response to the experimenter's strategy $\bm{p}^*$, and vice versa, i.e.,
\begin{equation}
\label{eq:equilibrium}
  \min_{\bm{q}\in \mathcal{D}\left(\overline{\rm Alt}(\thetabf)\right)}\Gamma_{\thetabf}(\bm{p}^*,\bm{q})
  = \Gamma_{\thetabf}(\bm{p}^*,\bm{q}^*)
  = \max_{\bm{p}\in\Sigma_k}\Gamma_{\thetabf}(\bm{p},\bm{q}^*).
\end{equation}
Also this appendix shows the equilibrium value's formula~\eqref{eq:equilibrium value's formula}.

\subsection{Proof of the first equality in definition of equilibrium}
We prove the first equality in~\eqref{eq:equilibrium} by showing a stronger result: 
\begin{proposition}
\label{prop:best responses to p*}
If the experiment plays the strategy~$\bm{p}^*$, for any $j\neq I^*$, the alternative $\varthetabf^{*j}$ (defined in~\eqref{eq:hard instance}) is a best response, and therefore any skeptic's mixed strategy over the alternatives~$\{\varthetabf^{*j}\}_{j\neq I^*}$ (i.e. $\bm{q}^*$) is a best response. Mathematically,
\[
\forall \widetilde{\bm{q}}\in\mathcal{D}\left(\{\varthetabf^{*j}\}_{j\neq I^*}\right), 
\quad
\Gamma_{\thetabf}(\bm{p}^*,\widetilde{\bm{q}}) 
= \min_{\varthetabf\in \overline{\rm Alt}(\thetabf)}\Gamma_{\thetabf}(\bm{p}^*,\varthetabf) 
=\min_{\bm{q}\in  \mathcal{D}\left(\overline{\rm Alt}(\thetabf)\right)}\Gamma_{\thetabf}(\bm{p}^*,\bm{q}).
\]
\end{proposition}

\begin{proof}
For any $j\neq I^*$, plugging in the definition of $\varthetabf_j$ in~\eqref{eq:hard instance} gives
\[
\Gamma_{\thetabf}(\bm{p}^*, \varthetabf^{*j})= \frac{p^*_{I^*}\KL(\theta_{I^*},\bar\theta^*_{I^*,j})+p^*_{j}\KL(\theta_{j},\bar\theta^*_{I^*,j})}{\sum_{i\in[k]}p^*_iC_i(\thetabf)}
=\frac{D_{\thetabf,I^*,j}(p^*_{I^*},p^*_j)}{\sum_{i\in[k]}p^*_iC_i(\thetabf)}.
\]
Observe that information balance condition \eqref{eq:info-balance-general} implies
\[
D_{\thetabf,I^*,j}(p^*_{I^*},p^*_j) = \min_{i\neq I^*}D_{\thetabf,I^*,i}(p^*_{I^*},p^*_i) = \min_{\varthetabf\in \overline{\rm Alt}(\thetabf)}\sum_{i\in[k]} p_i^* {\rm KL}(\theta_i,\vartheta_i),
\]
where the last equality above applies \citet[Lemma 3]{garivier2016optimal} (which will be stated immediately after this proof),
so
\begin{align*}
\Gamma_{\thetabf}(\bm{p}^*, \varthetabf^{*j}) 
= \frac{ D_{\thetabf,I^*,j}(p^*_{I^*},p^*_j)}{\sum_{i\in[k]}p^*_iC_i(\thetabf)}
= \frac{\min_{\varthetabf\in \overline{\rm Alt}(\thetabf)}\sum_{i\in[k]} p_i^* {\rm KL}(\theta_i,\vartheta_i)}{\sum_{i\in[k]}p_i^*C_i(\thetabf)}  
=
\min_{\varthetabf\in \overline{\rm Alt}(\thetabf)}\Gamma_{\thetabf}(\bm{p}^*,\varthetabf),
\end{align*}
Therefore, for any $j\neq I^*$, the skeptic's pure strategy of picking the alternative~$\varthetabf^{*j}$ is a best response to the experimenter's strategy $\bm{p}^*$, and thus any mixed strategy over the alternatives~$\{\varthetabf^{*j}\}_{j\neq I^*}$ is a best response because of the linearity of the payoff function $\Gamma_{\thetabf}(\bm{p},\bm{q})$ in $\bm{q}$. Also the linearity implies that there is always a minimizer that places all probability on one alternative, so $\min_{\varthetabf\in \overline{\rm Alt}(\thetabf)}\Gamma_{\thetabf}(\bm{p}^*,\varthetabf) =\min_{\bm{q}\in  \mathcal{D}\left(\overline{\rm Alt}(\thetabf)\right)}\Gamma_{\thetabf}(\bm{p}^*,\bm{q})$; as proved above, this game has $(k-1)$ such minimizers. This completes the proof.
\end{proof}

For self-completeness, we state \citet[Lemma 3]{garivier2016optimal}, which was used above in the proof of Proposition~\ref{prop:best responses to p*}.
\begin{lemma}[{\citet[Lemma 3]{garivier2016optimal}}]
\label{lem:Lemma3 in garivier2016optimal}
Fix $\thetabf\in\Theta$. For any $\bm{p}\in\Sigma_k$,
\[
\min_{\varthetabf\in\overline{\rm Alt}(\thetabf)}\sum_{i\in[k]}p_i\KL(\theta_i,\vartheta_i) 
= \min_{j\neq I^*} D_{\thetabf,I^*,j}(p_{I^*},p_j),
\]
where\footnote{Since $\overline{\rm Alt}(\thetabf)$ includes the boundary of ${\rm Alt}(\thetabf)$, for the LHS, the potential minimizers belong to $\overline{\rm Alt}(\thetabf)$.} 
$D_{\thetabf,I^*,j}(p_{I^*},p_j)$ is the weighted version of the Chernoff information defined in~\eqref{eq:chernoff-info}.
\end{lemma}

\subsection{Proof of the second equality in definition of equilibrium}
We prove the second equality in~\eqref{eq:equilibrium} by showing a stronger result:
\begin{proposition}
If the skeptic plays the mixed strategy $\bm{q}^*$, any experimenter's strategy has the same payoff value, and therefore any experimenter's strategy (i.e. $\bm{p}^*$) is a best response. Mathematically,
\[
\forall \bm{p}\in\Sigma_k,\quad\Gamma_{\thetabf}(\bm{p},\bm{q}^*) = \Gamma_{\thetabf}({\bm p}^*, {\bm q}^*) = \max_{\bm{p}'\in\Sigma_k}\Gamma_{\thetabf}(\bm{p}',\bm{q}^*).
\]
\end{proposition}

\begin{proof}
For any $\bm{p}=(p_1,\ldots,p_k)\in\Sigma_k$,
\begin{align*}
\Gamma_{\thetabf}(\bm{p},\bm{q}^*) &= \sum_{j\neq I^*} \bm{q}^*(\varthetabf^{*j})\cdot\Gamma_{\thetabf}(\bm{p},\varthetabf^{*j})\\
&= \Gamma_{\thetabf}(\bm{p}^*,\bm{q}^*)\sum_{j\neq I^*}\frac{C_j(\thetabf)}{{\rm KL}(\theta_{j},\bar{\theta}^*_{I^*,j})}\frac{p_{I^*}\KL(\theta_{I^*},\bar\theta^*_{I^*,j})+p_{j}\KL(\theta_{j},\bar\theta^*_{I^*,j})}{\sum_{i\in[k]}p_iC_i(\thetabf)}\\
&= \Gamma_{\thetabf}({\bm p}^*, {\bm q}^*)\sum_{j\neq I^*}\frac{C_j(\thetabf)}{{\sum_{i\in[k]}p_iC_i(\thetabf)}}
\left[p_{I^*}^*\frac{{\rm KL}(\theta_{I^*},\bar{\theta}^*_{I^*,j})}{{\rm KL}(\theta_{j},\bar{\theta}^*_{I^*,j})} + p_j^*\right]\\
&= \frac{\Gamma_{\thetabf}({\bm p}^*, {\bm q}^*)}{\sum_{i\in[k]}p_iC_i(\thetabf)}\left[p_{I^*}\sum_{j\neq I^*} \frac{C_j(\thetabf)\KL(\theta_{I^*},\bar\theta^*_{I^*,j})}{{\rm KL}(\theta_{j},\bar\theta^{*}_{I^*,j})} + \sum_{j\neq I^*}p_j C_j(\thetabf)\right]\\
&= \frac{\Gamma_{\thetabf}({\bm p}^*, {\bm q}^*)}{\sum_{i\in[k]}p_iC_i(\thetabf)}\left[p_{I^*}C_{I^*}(\thetabf) + \sum_{j\neq I^*} p_jC_j(\thetabf)\right]\\
&= \Gamma_{\thetabf}({\bm p}^*, {\bm q}^*),
\end{align*}
where the second equality plugs in the explicit expression of the probabilities $(\bm{q}^*(\varthetabf^{*j}))_{j\neq I^*}$ (which will be presented shortly in~\eqref{eq:optimal q exact} after this proof) and the definition of the alternatives $\{\varthetabf^{*j}\}_{j\neq I^*}$ in~\eqref{eq:hard instance}.
The penultimate equality follows from exploitation rate condition~\eqref{eq:exploitation-rate-general}. 
Therefore, under the skeptic's mixed strategy $\bm{q}^*$, any experimenter's strategy achieves the same payoff value, and thus any experimenter's strategy is a best response. This completes the proof.
\end{proof}

\subsubsection*{An explicit expression of the probabilities~$(\bm{q}^*(\varthetabf^{*j}))_{j\neq I^*}$.} 
We derive an explicit expression of the probabilities~$(\bm{q}^*(\varthetabf^{*j}))_{j\neq I^*}$ by verifying that
\begin{equation}
\label{eq:normalizing constant}
\sum_{j\neq I^*}q^*_j = 1
\quad\text{where}\quad
q^*_j \triangleq \frac{C_j(\thetabf)}{{\rm KL}(\theta_{j},\bar{\theta}^*_{I^*,j})} \Gamma_{\thetabf}(\bm{p}^*,\bm{q}^*), \quad\forall j\neq I^*.
\end{equation}

If \eqref{eq:normalizing constant} holds, the normalizing constant for the weights in~\eqref{eq:optimal q} is $\frac{1}{\Gamma_{\thetabf}(\bm{p}^*,\bm{q}^*)}$, and thus
\begin{equation}
\label{eq:optimal q exact}
\bm{q}^*(\bm{\vartheta}^{*j}) = \frac{C_j(\thetabf)}{{\rm KL}(\theta_{j},\bar{\theta}^*_{I^*,j})}\Gamma_{\thetabf}(\bm{p}^*,\bm{q}^*), \quad\forall j\neq I^*,
\end{equation}

\begin{proof}[Proof of~\eqref{eq:normalizing constant}]
For $j\neq I^*$, we write
\begin{align*}
q^*_j = \frac{C_j(\thetabf)}{{\rm KL}(\theta_{j},\bar{\theta}^*_{I^*,j})}\Gamma_{\thetabf}(\bm{p}^*,\bm{q}^*)  &= \frac{C_j(\thetabf)}{{\rm KL}(\theta_{j},\bar{\theta}^*_{I^*,j})}\Gamma_{\thetabf}(\bm{p}^*,\varthetabf^{*j}) \\
&= 
\frac{C_j(\thetabf)}{{\rm KL}(\theta_{j},\bar{\theta}^*_{I^*,j})}\frac{p^*_{I^*}\KL(\theta_{I^*},\bar\theta^*_{I^*,j})+p^*_{j}\KL(\theta_{j},\bar\theta^*_{I^*,j})}{\sum_{i\in[k]}p^*_iC_i(\thetabf)} \\
&= \frac{C_j(\thetabf)}{{\sum_{i\in[k]}p^*_iC_i(\thetabf)}}
\left[p_{I^*}^*\frac{{\rm KL}(\theta_{I^*},\bar{\theta}^*_{I^*,j})}{{\rm KL}(\theta_{j},\bar{\theta}^*_{I^*,j})} + p_j^*\right],
\end{align*}
where the second equality follows from Proposition~\ref{prop:best responses to p*}, and the third one plugs in the definition of~$\varthetabf^{*j}$ in~\eqref{eq:hard instance}.
Then taking the summation gives
\begin{align*}
    \sum_{j\neq I^*} q^*_j &= \frac{1}{{\sum_{i\in[k]}p^*_iC_i(\thetabf)}} \left[p_{I^*}^*\sum_{j\neq I^*}\frac{C_j(\thetabf){\rm KL}(\theta_{I^*},\bar{\theta}^*_{I^*,j})}{{\rm KL}(\theta_{j},\bar{\theta}^*_{I^*,j})} + \sum_{j\neq I^*}p_j^*C_j(\thetabf)\right]\\
    &= \frac{1}{{\sum_{i\in[k]}p^*_iC_i(\thetabf)}} \left[p_{I^*}^*C_{I^*}(\thetabf) + \sum_{j\neq I^*}p_j^*C_j(\thetabf)\right]\\
    &=1,
\end{align*}
where the last equality uses exploitation rate condition \eqref{eq:exploitation-rate-general}. This completes the proof of~\eqref{eq:normalizing constant}.
\end{proof}

\subsection{Proof of equilibrium value's formula in~Theorem~\ref{thm:equilibrium}}
Now we have formally proved that $(\bm{p}^*,\bm{q}^*)$ forms an equilibrium. Since $\bm{q}^*$ in~\eqref{eq:optimal q exact} is a probability vector, its explicit expression in~\eqref{eq:optimal q exact} implies the equilibrium value's formula in~Theorem~\ref{thm:equilibrium}.

\section{Proof of unique equilibrium in Theorem~\ref{thm:equilibrium}}
\label{app:unique equilibrium proof}

Appendix~\ref{app:equilibrium verification} has verified that $(\bm{p}^*,\bm{q}^*)$ forms an equilibrium of our game. This appendix completes the proof of Theorem \ref{thm:equilibrium} by showing the uniqueness of $(\bm{p}^*,\bm{q}^*)$.

We first present and prove a classical result applied to our zero-sum game.
\begin{proposition}
\label{prop:equilibrium property}
If $(\widetilde{\bm p}, \widetilde{\bm q})\in\Sigma_k\times \mathcal{D}\left( \overline{\rm Alt}({\bm \theta}) \right)$ is an equilibrium,
then
\[
\widetilde{\bm p}\in \argmax_{\bm{p}\in\Sigma_k}\inf_{\bm{q}\in \mathcal{D}\left(\overline{\rm Alt}(\thetabf)\right)}\Gamma_{\thetabf}(\bm{p},\bm{q})
\quad\text{and}\quad
\widetilde{\bm q}\in \argmin_{\bm{q}\in \mathcal{D}\left(\overline{\rm Alt}(\thetabf)\right)}\sup_{\bm{p}\in\Sigma_k}\Gamma_{\thetabf}(\bm{p},\bm{q}).
\]
\end{proposition}

\begin{proof}
Since $(\widetilde{\bm p}, \widetilde{\bm q})$ forms an equilibrium, 
\begin{equation}
\label{eq:max-min equality}
\sup_{\bm{p}\in\Sigma_k}\inf_{\bm{q}\in \mathcal{D}\left(\overline{\rm Alt}(\thetabf)\right)}\Gamma_{\thetabf}(\bm{p},\bm{q})
 \geq \inf_{\bm{q}\in \mathcal{D}\left(\overline{\rm Alt}(\thetabf)\right)}\Gamma_{\thetabf}(\widetilde{\bm p}, \bm{q})
 =\Gamma_{\thetabf}(\widetilde{\bm p}, \widetilde{\bm q})
 =\sup_{\bm{p}\in\Sigma_k}\Gamma_{\thetabf}(\bm{p},\widetilde{\bm q})
 \geq \inf_{\bm{q}\in \mathcal{D}\left(\overline{\rm Alt}(\thetabf)\right)}\sup_{\bm{p}\in\Sigma_k}\Gamma_{\thetabf}(\bm{p},\bm{q}).
\end{equation}
The reverse inequality (the max-min inequality \citep{boyd2004convex}) always holds,
\[
\sup_{\bm{p}\in\Sigma_k}\inf_{\bm{q}\in \mathcal{D}\left(\overline{\rm Alt}(\thetabf)\right)}\Gamma_{\thetabf}(\bm{p},\bm{q})
\leq
  \inf_{\bm{q}\in \mathcal{D}\left(\overline{\rm Alt}(\thetabf)\right)}\sup_{\bm{p}\in\Sigma_k}\Gamma_{\thetabf}(\bm{p},\bm{q}).
\]
Therefore, the max-min equality holds, 
and the inequalities in~\eqref{eq:max-min equality} become equalities, which imply respectively,
\[
\widetilde{\bm p}\in \argmax_{\bm{p}\in\Sigma_k}\inf_{\bm{q}\in \mathcal{D}\left(\overline{\rm Alt}(\thetabf)\right)}\Gamma_{\thetabf}(\bm{p},\bm{q})
\quad\text{and}\quad
\widetilde{\bm q}\in \argmin_{\bm{q}\in \mathcal{D}\left(\overline{\rm Alt}(\thetabf)\right)}\sup_{\bm{p}\in\Sigma_k}\Gamma_{\thetabf}(\bm{p},\bm{q}).
\]
\end{proof}

\subsection{Proof of uniqueness of experimenter's equilibrium strategy}

The next result shows that the set $\argmax_{\bm{p}\in\Sigma_k}\inf_{\bm{q}\in \mathcal{D}\left(\overline{\rm Alt}(\thetabf)\right)}\Gamma_{\thetabf}(\bm{p},\bm{q})$ in Proposition~\ref{prop:equilibrium property} has only one element, which is $\bm{p}^*$.
\begin{proposition}
\label{prop:uniqueness of the experimenter's equilibrium strategy}
$\bm{p}^*$ is the unique maximizer to the max-min problem
$
\max_{\bm{p}\in\Sigma_k}\inf_{\bm{q}\in \mathcal{D}\left(\overline{\rm Alt}(\thetabf)\right)}\Gamma_{\thetabf}(\bm{p},\bm{q}).
$
\end{proposition}

It suffices to prove that any maximizer to the max-min problem $\max_{\bm{p}\in\Sigma_k}\inf_{\bm{q}\in \mathcal{D}\left(\overline{\rm Alt}(\thetabf)\right)}\Gamma_{\thetabf}(\bm{p},\bm{q})$ must satisfy information balance condition~\eqref{eq:info-balance-general} and exploitation rate condition~\eqref{eq:exploitation-rate-general}; recall that Lemma~\ref{lem:properties of p*} ensures that $\bm{p}^*$ is the only probability vector satisfying~\eqref{eq:info-balance-general} and~\eqref{eq:exploitation-rate-general}.

\paragraph{A rewrite of the inner minimization problem.}
By the linearity of the payoff function $\Gamma_{\thetabf}(\bm{p},\bm{q})$ in~$\bm{q}$, 
we can write the max-min problem as follows, for any $\bm{p}\in\Sigma_k$,
\begin{align}
\inf_{\bm{q}\in \mathcal{D}\left(\overline{\rm Alt}(\thetabf)\right)}\Gamma_{\thetabf}(\bm{p},\bm{q})
=&  \inf_{\varthetabf\in{\overline{\rm Alt}(\thetabf)}}\Gamma_{\thetabf}(\bm{p},\varthetabf)\nonumber\\
    =& \frac{\inf_{\varthetabf\in{\overline{\rm Alt}(\thetabf)}}\sum_{i\in[k]} p_i {\rm KL}(\theta_i,\vartheta_i)}{\sum_{i\in[k]}p_iC_i(\thetabf)}\nonumber\\  
    =&\frac{\min_{j\neq I^*}D_{\thetabf,I^*,j}(p_{I^*},p_j)}{\sum_{i\in[k]}p_iC_i(\thetabf)} \nonumber\\
    =&\frac{\min_{j\neq I^*} \left[p_{I^*}\KL(\theta_{I^*},\bar{\theta}_{I^*,j}) + p_{j}\KL(\theta_{j},\bar{\theta}_{I^*,j})\right]}{\sum_{i\in[k]}p_iC_i(\thetabf)} \nonumber\\ 
    =& \min_{j\neq I^*}\, \left[\frac{p_{I^*}C_{I^*}(\thetabf)}{\sum_{i\in[k]} p_iC_i(\thetabf)}\frac{{\rm KL}(\theta_{I^*},\bar{\theta}_{I^*,j})}{C_{I^*}(\thetabf)} + \frac{p_{j}C_{j}(\thetabf)}{\sum_{i\in[k]} p_iC_i(\thetabf)}\frac{{\rm KL}(\theta_{j},\bar{\theta}_{I^*,j})}{C_j(\thetabf)}\right] \nonumber\\
    =& \min_{j\neq I^*}\, \left[w_{I^*}\frac{{\rm KL}(\theta_{I^*},\bar{\theta}_{I^*,j})}{C_{I^*}(\thetabf)} + w_{j}\frac{{\rm KL}(\theta_{j},\bar{\theta}_{I^*,j})}{C_j(\thetabf)}\right] 
    \triangleq  \min_{j\neq I^*}\, \Phi_{\thetabf,I^*,j}(w_{I^*},w_j), \label{eq:define Phi}
\end{align}
where the third inequality applies Lemma~\ref{lem:Lemma3 in garivier2016optimal}. In~\eqref{eq:define Phi}, we define the probability vector $\bm{w}=(w_1,\ldots,w_k)\in\Sigma_k$ such that
\begin{equation}
\label{eq:tranformation between p and w}
w_j \triangleq \frac{p_{j}C_{j}(\thetabf)}{\sum_{i\in[k]} p_iC_i(\thetabf)}, \quad \forall j\in[k],
\end{equation}
and rewrite the weighted means in terms of $\bm{w}$ as follows,\footnote{When $p_{I^*} = p_j = 0$ or $w_{I^*} = w_j = 0$, we let $\bar{\theta}_{I^*,j}=\theta_j$.}
\begin{equation}
\label{eq:rewrite bar theta}
\bar{\theta}_{I^*,j} = \frac{p_{I^*} \theta_{I^*}+ p_j\theta_j}{ p_{I^*}+p_j} =  \frac{\frac{w_{I^*}}{C_{I^*}(\thetabf)} \theta_{I^*}+ \frac{w_{j}}{C_{j}(\thetabf)}\theta_j}{ \frac{w_{I^*}}{C_{I^*}(\thetabf)} + \frac{w_{j}}{C_{j}(\thetabf)}}, \quad \forall j\neq I^*.
\end{equation}

We observe in~\eqref{eq:tranformation between p and w} that for any $j\in[k]$, $w_j\propto p_{j}C_{j}(\thetabf)$, so we can view the probability vector $\bm{w}$ as a ``cost-relaxed" analog of the ``cost-aware" allocation $\bm{p}$.

\paragraph{A restatement of Proposition~\ref{prop:uniqueness of the experimenter's equilibrium strategy}.}

Given the one-to-one mapping between $\bm{p}$ and $\bm{w}$ in~\eqref{eq:tranformation between p and w}, completing the proof of Proposition~\ref{prop:uniqueness of the experimenter's equilibrium strategy} is equivalent to showing that any maximizer to the max-min problem
$
\max_{\bm{w}\in\Sigma_k}\min_{j\neq I^*}\, \Phi_{\thetabf,I^*,j}(w_{I^*},w_j)
$
must solve
\begin{equation}
\label{eq:info-balance-general-w}
\Phi_{\thetabf,I^*,i}(w_{I^*},w_i) = \Phi_{\thetabf,I^*,j}(w_{I^*},w_j),
\quad \forall i,j\neq I^*
\end{equation}
and
\begin{equation}
\label{eq:exploitation-rate-general-w}
	\sum_{j\neq I^*}\frac{C_j(\thetabf)}{C_{I^*}(\thetabf)}\frac{{\rm KL}(\theta_{I^*},\bar{\theta}_{I^*,j})}{{\rm KL}(\theta_{j},\bar{\theta}_{I^*,j})} = 1
	\quad\text{with}\quad  
	\bar{\theta}_{I^*,j} \text{ written in $\bm{w}$ as in \eqref{eq:rewrite bar theta}},
\end{equation}
where~\eqref{eq:info-balance-general-w} and~\eqref{eq:exploitation-rate-general-w} can be viewed as information balance condition and overall balance condition for $\bm{w}$, respectively.

\begin{proof}[Proof of the restatement of Proposition \ref{prop:uniqueness of the experimenter's equilibrium strategy}]
The max-min problem~
$
\max_{\bm{w}\in\Sigma_k}\min_{j\neq I^*}\, \Phi_{\thetabf,I^*,j}(w_{I^*},w_j)
$ 
can be reformulated as 
\begin{align*}
	\max & \,\,\,\, \phi\\
            \text{s.t} & \,\,\,\, \phi - \Phi_{\thetabf,I^*,j}(w_{I^*},w_j)   \le 0, \quad\forall j \neq I^*\\
		  & \,\,\,\, \sum_{i\in[k]} w_i - 1=0 \\
		& \,\,\,\, w_i \geq 0, \quad \forall i\in[k].
\end{align*}
The Lagrangian is
\[
	\mathcal{L}(\phi, \bm w, \lambda, \bm q) = \phi - \lambda\left( \sum_{i\in[k]}w_i - 1\right) - \sum_{j\neq I^*} q_{j} \left[\phi - \Phi_{\thetabf,I^*,j}(w_{I^*},w_j)\right].
\]
Note that the last set of constraints is not included in the Lagrangian since it is easy to verify that for all entries of an optimal solution are strictly positive, and therefore these constraints are not binding.

KKT conditions give
\begin{align}
1 - \sum_{j\neq I^*} q_j &= 0 \label{eq:KKT1}\\
q_j \frac{\partial \Phi_{\thetabf,I^*,j}(w_{I^*},w_j)}{\partial w_j} = q_j\frac{\KL(\theta_{j}, \bar{\theta}_{I^*,j})}{C_{j}(\thetabf)} &= \lambda, \quad \forall j\neq I^* \label{eq:KKT2}\\
\sum_{j\neq I^*} q_j\frac{\partial \Phi_{\thetabf,I^*,j}(w_{I^*},w_j)}{\partial w_{I^*}} = 
\sum_{j\neq I^*} q_j\frac{\KL(\theta_{I^*}, \bar{\theta}_{I^*,j})}{C_{I^*}(\thetabf)}&= \lambda, \label{eq:KKT3}\\
q_{j} \left[\phi - \Phi_{\thetabf,I^*,j}(w_{I^*},w_j)\right] &= 0,  \quad \forall j\neq I^*, \label{eq:KKT4}
\end{align}
where the partial derivatives of KL divergence in \eqref{eq:KKT2} and \eqref{eq:KKT3} apply Lemma \ref{lem:derivative of Phi} (which will be presently immediately after this proof).

The first observation is that \eqref{eq:KKT1} implies there exists some $j'\neq I^*$ such that $q_{j'}>0$. 
Since all entries of an optimal solution are positive, we have that for any $j\neq I^*$, $\KL(\theta_{j}, \bar{\theta}_{I^*,j}) > 0$, and thus 
$
\frac{\KL\left(\theta_{j}, \bar{\theta}_{I^*,j}\right)}{C_j(\thetabf)} > 0.
$
Then \eqref{eq:KKT2} gives 
$
\lambda = q_{j'} \frac{\partial \Phi_{\thetabf,I^*,j'}\left(w_{I^*},w_{j'}\right)}{\partial w_{j'}} > 0
$, and thus $q_j>0$ for any $j\neq I^*$. Then information balance condition \eqref{eq:info-balance-general-w} are implied by \eqref{eq:KKT4}, and exploitation rate balance \eqref{eq:exploitation-rate-general-w} follows from plugging in the expression of $\{q_j\}_{j\neq I^*}$ derived from \eqref{eq:KKT2} to \eqref{eq:KKT3} and dividing $\lambda$ from both sides. This completes the proof.
\end{proof}

\subsubsection{Supporting lemmas}
We present the technical lemmas on the partial derivative of KL divergence, which were used above in the proof of Proposition \ref{prop:uniqueness of the experimenter's equilibrium strategy}.
\begin{lemma}[Lemma 19 in \cite{qin2023dualdirected}]
\label{lem:derivative of KL}
For two distributions in one-dimensional exponential family parameterized by their mean parameters $\theta$ and $\vartheta$, respectively,
\[
\frac{\partial \KL(\theta,\vartheta)}{\partial \theta} = \eta(\theta) - \eta(\vartheta)
\quad\text{and}\quad
\frac{\partial \KL(\theta,\vartheta)}{\partial \vartheta} = (\theta - \vartheta) \frac{\mathrm{d} \eta(\vartheta)}{\mathrm{d} \vartheta}.
\]
\end{lemma}

With this result, we can calculate the partial derivatives of $\Phi_{\thetabf,I^*,j}(w_{I^*},w_j)$ in \eqref{eq:KKT2} and \eqref{eq:KKT3}.
\begin{lemma}
\label{lem:derivative of Phi}
Fix $j\neq I^*$. 
\[
\forall w_{I^*} > 0,\quad     \frac{\partial \Phi_{\thetabf,I^*,j}(w_{I^*},w_j)}{\partial w_{I^*}} = \frac{\KL(\theta_{I^*}, \bar{\theta}_{I^*,j})}{C_{I^*}(\thetabf)},
\quad\text{and}\quad
\forall w_{j} > 0,\quad \frac{\partial \Phi_{\thetabf,I^*,j}(w_{I^*},w_j)}{\partial w_j} = \frac{\KL(\theta_{j}, \bar{\theta}_{I^*,j})}{C_j(\thetabf)}.
\]

\end{lemma}
\begin{proof}
Fix $j\neq I^*$.
\begin{align*}
\frac{\partial \Phi_{\thetabf,I^*,j}(w_{I^*},w_j)}{\partial w_{I^*}} &= \frac{\KL(\theta_{I^*}, \bar{\theta}_{I^*,j})}{C_{I^*}(\thetabf)}
+ \frac{w_{I^*}}{C_{I^*}(\thetabf)}\frac{\partial \KL(\theta_{I^*}, \bar\theta_{I^*,j})}{\partial \bar\theta_{I^*,j}}\frac{\mathrm{d}\bar\theta_{I^*,j}}{\mathrm{d} w_{I^*}}
+ \frac{w_j}{C_{j}(\thetabf)}\frac{\partial \KL(\theta_{j}, \bar\theta_{I^*,j})}{\partial \bar\theta_{I^*,j}}\frac{\mathrm{d}\bar\theta_{I^*,j}}{\mathrm{d} w_{I^*}}\\
&= \frac{\KL(\theta_{I^*}, \bar\theta_{I^*,j})}{C_{I^*}(\thetabf)} + 
\left[\frac{w_{I^*}}{C_{I^*}(\thetabf)}\frac{\partial \KL(\theta_{I^*}, \bar\theta_{I^*,j})}{\partial \bar\theta_{I^*,j}} + \frac{w_j}{C_{j}(\thetabf)}\frac{\partial \KL(\theta_{j}, \bar\theta_{I^*,j})}{\partial \bar\theta_{I^*,j}}\right]\frac{\mathrm{d}\bar\theta_{I^*,j}}{\mathrm{d} w_{I^*}} \\
&= \frac{\KL(\theta_{I^*}, \bar\theta_{I^*,j})}{C_{I^*}(\thetabf)} +  
\left[\frac{w_{I^*}}{C_{I^*}(\thetabf)}(\theta_{I^*}-\bar\theta_{I^*,j}) + \frac{w_j}{C_{j}(\thetabf)}(\theta_{j}-\bar\theta_{I^*,j}) \right] \frac{\mathrm{d}\eta(\bar\theta_{I^*,j})}{\mathrm{d}\bar\theta_{I^*,j}}\frac{\mathrm{d}\bar\theta_{I^*,j}}{\mathrm{d} w_{I^*}}\\
&= \frac{\KL(\theta_{I^*}, \bar\theta_{I^*,j})}{C_{I^*}(\thetabf)} + \frac{p_{I^*}(\theta_{I^*}-\bar\theta_{I^*,j}) + p_j(\theta_{j}-\bar\theta_{I^*,j})}{\sum_{i\in[k]} p_iC_i(\thetabf)}\frac{\mathrm{d}\eta(\bar\theta_{I^*,j})}{\mathrm{d}\bar\theta_{I^*,j}}\frac{\mathrm{d}\bar\theta_{I^*,j}}{\mathrm{d} w_{I^*}}\\
&= \frac{\KL(\theta_{I^*}, \bar\theta_{I^*,j})}{C_{I^*}(\thetabf)},
\end{align*}
where the third equality applies Lemma \ref{lem:derivative of KL} and the fourth equality follows from \eqref{eq:rewrite bar theta}, which writes $\bar\theta_{I^*,j}$ in terms of $(w_{I^*}, w_j)$. Similarly, 
$
\frac{\partial \Phi_{\thetabf,I^*,j}(w_{I^*},w_j)}{\partial w_j} = \frac{\KL(\theta_{j}, \bar\theta_{I^*,j})}{C_j(\thetabf)}.
$
\end{proof}

\subsection{Proof of uniqueness of skeptic's equilibrium strategy}
We have shown that $\bm{p}^*$ is the unique equilibrium strategy of the experimenter. Now we start to prove that $\bm{q}^*$ is the unique equilibrium strategy of the skeptic. Recall that $\bm{q}^*$ is supported on the alternatives~$\{\varthetabf^{*j}\}_{j\neq I^*}$ (defined in~\eqref{eq:hard instance}).

We first prove that if $\widetilde{\bm q}$ is an equilibrium strategy of the skeptic, it must be supported over the alternatives 
$\{\varthetabf^{*j}\}_{j\neq I^*}$.

\begin{lemma}
\label{lem:support of skeptic's equilibrium strategy}
If $\widetilde{\bm q}\in\mathcal{D}\left( \overline{\rm Alt}({\bm \theta}) \right)$ is an equilibrium strategy of the skeptic, then $\widetilde{\bm q}$ is supported over the alternatives $\{\varthetabf^{*j}\}_{j\neq I^*}$, i.e., $\widetilde{\bm q}\in\mathcal{D}\left( \{\varthetabf^{*j}\}_{j\neq I^*} \right)$.
\end{lemma}

\begin{proof}
Since~$\bm{p}^*$ is the unique equilibrium strategy of the experimenter, the second equality~\eqref{eq:max-min equality} becomes
$\Gamma_{\thetabf}(\bm{p}^*,\widetilde{\bm q}) = \min_{\bm{q}\in \mathcal{D}\left(\overline{\rm Alt}(\thetabf)\right)}\Gamma_{\thetabf}({\bm p}^*, \bm{q})$.\footnote{We can replace $\inf_{\bm{q}\in \mathcal{D}\left(\overline{\rm Alt}(\thetabf)\right)}$ with $\min_{\bm{q}\in \mathcal{D}\left(\overline{\rm Alt}(\thetabf)\right)}$ here since we have proved that $\bm{q}^*$ is a best response to $\bm{p}^*$ (i.e., a minimizer).}
That is, the skeptic's equilibrium strategy~$\widetilde{\bm q}$ is a best response to $\bm{p}^*$, and~$\widetilde{\bm q}$  must be a mixture of the skeptic's pure strategies that best respond to $\bm{p}^*$.

To complete the proof, we are going to show that the alternatives $\{\varthetabf^{*j}\}_{j\neq I^*}$ are the only alternatives that best respond to $\bm{p}^*$.
The definition of the payoff function gives
$
\Gamma_{\thetabf}({\bm p}^*, \varthetabf) = \frac{\sum_{i\in[k]} p_i^* {\rm KL}(\theta_i,\vartheta_i)}{\sum_{i\in[k]}p_i^*C_i(\thetabf)}.
$
Since the denominator is independent of $\varthetabf$, it suffices to find minimizers to the numerator.
We can decompose this minimization problem as follows,
\begin{equation}
\label{eq:double minimization}
\inf_{\varthetabf\in{\overline{\rm Alt}(\thetabf)}}\sum_{i\in[k]} p_i^* {\rm KL}(\theta_i,\vartheta_i) 
= \min_{j\neq I^*}\inf_{\bm{\vartheta}\in \overline{\rm Alt}(\thetabf): \vartheta_j\geq \vartheta_{I^*}}\sum_{i\in[k]} p_i^* {\rm KL}(\theta_i,\vartheta_i).
\end{equation}
For any $j\neq I^*$, we can write the corresponding inner minimization problem on the RHS of~\eqref{eq:double minimization} as follows,
\begin{align*}
\inf_{\bm{\vartheta}\in \overline{\rm Alt}(\thetabf): \vartheta_j\geq \vartheta_{I^*}}\sum_{i\in[k]} p_i^* {\rm KL}(\theta_i,\vartheta_i) 
=& \inf_{\vartheta_j\geq \vartheta_{I^*}}
\left[p_{I^*}^*\KL(\theta_{I^*},\vartheta_{I^*}) + p_{j}^*\KL(\theta_{j},\vartheta_{j})\right] \\
=& p_{I^*}^*\KL(\theta_{I^*},\bar{\theta}^*_{I^*,j}) + p_{j}^*\KL(\theta_{j},\bar{\theta}^*_{I^*,j}),
\end{align*}
where the first equality holds iff $\vartheta_i = \theta_i$ for any $i\notin\{I^*,j\}$, and the second equality is~\eqref{eq:chernoff-info-minimizer}, which holds iff $\bar{\theta}^*_{I^*,j} = \frac{p^*_{I^*} \theta_{I^*}+ p^*_j\theta_j}{ p^*_{I^*}+p^*_j}$ (where $p^*_{I^*},p^*_j > 0$ by Lemma~\ref{lem:properties of p*}). Hence, for any $j\neq I^*$, the corresponding inner minimization problem on the RHS of~\eqref{eq:double minimization} has a unique minimizer, which is $\varthetabf^{*j}$ (defined in~\eqref{eq:hard instance}). By information balance condition~\eqref{eq:info-balance-general}, the $(k-1)$ inner minimization problems on the RHS of~\eqref{eq:double minimization} have the same optimal value, so~$\{\varthetabf^{*j}\}_{j\neq I^*}$ are the minimizers and the only minimizers to the the original minimization problem (the LHS of~\eqref{eq:double minimization}), and thus they are the only pure strategies of the skeptic that best respond to $\bm{p}^*$. This completes the proof.

\end{proof}

Combining Lemma~\ref{lem:support of skeptic's equilibrium strategy} and Proposition~\ref{prop:equilibrium property} yields that if $\widetilde{\bm q}\in\mathcal{D}\left( \overline{\rm Alt}({\bm \theta}) \right)$ is an equilibrium strategy of the skeptic, 
\[
\widetilde{\bm q}\in\mathcal{D}\left( \{\varthetabf^{*j}\}_{j\neq I^*} \right)
\quad\text{and}\quad
\widetilde{\bm q}\in \argmin_{\bm{q}\in \mathcal{D}\left(\overline{\rm Alt}(\thetabf)\right)}\sup_{\bm{p}\in\Sigma_k}\Gamma_{\thetabf}(\bm{p},\bm{q}),
\]
and thus 
\[
\widetilde{\bm q}\in \argmin_{\bm{q}\in\mathcal{D}\left( \{\varthetabf^{*j}\}_{j\neq I^*} \right)}\sup_{\bm{p}\in\Sigma_k}\Gamma_{\thetabf}(\bm{p},\bm{q}),
\]

The next result shows that the set $\argmin_{\bm{q}\in\mathcal{D}\left( \{\varthetabf^{*j}\}_{j\neq I^*} \right)}\sup_{\bm{p}\in\Sigma_k}\Gamma_{\thetabf}(\bm{p},\bm{q})$ has only one element, which is $\bm{q}^*$.

\begin{proposition}
\label{prop:uniqueness of the skeptic's equilibrium strategy}
$\bm{q}^*$ is the only minimizer to the min-max optimization
$
\min_{\bm{q}\in\mathcal{D}\left( \{\varthetabf^{*j}\}_{j\neq I^*} \right)}\sup_{\bm{p}\in\Sigma_k}\Gamma_{\thetabf}(\bm{p},\bm{q})
$.
\end{proposition}

\paragraph{A rewrite of the payoff function $\Gamma_{\thetabf}(\bm{p},\bm{q})$.}

For any $\bm{q}\in \mathcal{D}\left(\{\varthetabf^{*j}\}_{j\neq I^*}\right)$, we overload the notation $\bm{q} = ({q}_j)_{j\neq I^*}\in \Sigma_{k-1}$ with ${q}_j = \bm{q}(\varthetabf^{*j})$ being the probability of picking the alternative $\varthetabf^{*j}$. 

Then for any $\bm{p}\in\Sigma_k$ and $\bm{q}= ({q}_j)_{j\neq I^*}\in \Sigma_{k-1}$, we can rewirte the payoff function as follows,
\begin{align*}
\Gamma_{\thetabf}(\bm{p},\bm{q}) &= \sum_{j\neq I^*} q_j \Gamma_{\thetabf}(\bm{p},\varthetabf^{*j}) \\
&=  \sum_{j\neq I^*} q_j \frac{p_{I^*}\KL(\theta_{I^*},\bar\theta^*_{I^*,j})+p_{j}\KL(\theta_{j},\bar\theta^*_{I^*,j})}{\sum_{i\in[k]}p_iC_i(\thetabf)}\\
&= \sum_{j\neq I^*} q_j \left[ \frac{p_{I^*}C_{I^*}(\thetabf)}{\sum_{i\in[k]} p_iC_i(\thetabf)}\frac{{\rm KL}(\theta_{I^*},\bar{\theta}_{I^*,j}^*)}{C_{I^*}(\thetabf)} + \frac{p_{j}C_{j}(\thetabf)}{\sum_{i\in[k]} p_iC_i(\thetabf)}\frac{{\rm KL}(\theta_{j},\bar{\theta}_{I^*,j}^*)}{C_j(\thetabf)} \right] \\
&= \sum_{j\neq I^*} q_j \left[ w_{I^*}\frac{{\rm KL}(\theta_{I^*},\bar{\theta}_{I^*,j}^*)}{C_{I^*}(\thetabf)} + w_{j}\frac{{\rm KL}(\theta_{j},\bar{\theta}_{I^*,j}^*)}{C_j(\thetabf)} \right] \\
&\triangleq \widetilde{\Gamma}_{\thetabf}({\bm{w}},\bm{q}),
\end{align*}
which is written in terms of the probability vector ${\bm{w}}=({w}_1,\ldots,{w}_k)\in\Sigma_k$ such that for any $j\in[k]$, 
${w}_j = \frac{p_{j}C_{j}(\thetabf)}{\sum_{i\in[k]} p_iC_i(\thetabf)}$.

\paragraph{A restatement of Proposition~\ref{prop:uniqueness of the skeptic's equilibrium strategy}.}

Given the one-to-one mapping between $\bm{p}$ and $\bm{w}$, completing the proof of Proposition~\ref{prop:uniqueness of the skeptic's equilibrium strategy} is equivalent to showing that $\bm{q}^*$ is the only minimizer to the min-max problem $\min_{\bm{q}\in \Sigma_{k-1}} \sup_{{\bm{w}}\in\Sigma_k}\widetilde{\Gamma}_{\thetabf}({\bm{w}},\bm{q})$.

\paragraph{Simplification of the inner maximization problem.}
For any $\bm{q}= ({q}_j)_{j\neq I^*}\in \Sigma_{k-1}$, we write
\begin{align*}
\sup_{\bm{w}\in\Sigma_k}\widetilde{\Gamma}_{\thetabf}(\bm{w},\bm{q})
=&\sup_{\bm{w}\in\Sigma_k}
\left\{w_{I^*}\left[ \sum_{j\neq I^*} q_j \frac{{\rm KL}(\theta_{I^*},\bar{\theta}_{I^*,j}^*)}{C_{I^*}(\thetabf)}\right] + \sum_{j\neq I^*}w_{j}\left[q_j\frac{{\rm KL}(\theta_{j},\bar{\theta}_{I^*,j}^*)}{C_j(\thetabf)} \right]\right\}\\
=&\max\left\{\sum_{j\neq I^*} q_j \frac{{\rm KL}(\theta_{I^*},\bar{\theta}_{I^*,j}^*)}{C_{I^*}(\thetabf)},\, \max_{j\neq I^*}q_j\frac{{\rm KL}(\theta_{j},\bar{\theta}_{I^*,j}^*)}{C_j(\thetabf)}\right\},
\end{align*}
where the last equality holds since at least one of the standard unit vectors in~$\Sigma_k$ is a maximizer.

Now we are ready to complete the proof of the above restatement of Proposition~\ref{prop:uniqueness of the skeptic's equilibrium strategy} by showing that to achieve the minimal value, for any $j\neq I^*$, the probability of playing the alternative $\varthetabf^{*j}$ must be $\bm{q}^*(\varthetabf^{*j})$ (defined in~\eqref{eq:optimal q} with the explicit expression in~\eqref{eq:optimal q exact}).

\begin{proof}[Proof of the restatement of Proposition~\ref{prop:uniqueness of the skeptic's equilibrium strategy}]
For any $\bm{q}= ({q}_j)_{j\neq I^*}\in \Sigma_{k-1}$,
\begin{align*}
\max_{{\bm{w}}\in\Sigma_k}\widetilde{\Gamma}_{\thetabf}({\bm{w}},\bm{q})
&= \max\left\{\sum_{j\neq I^*}q_j\frac{{\rm KL}(\theta_{I^*},\bar{\theta}_{I^*,j}^*)}{C_{I^*}(\thetabf)}, \, \max_{j\neq I^*}q_j\frac{{\rm KL}(\theta_{j},\bar{\theta}_{I^*,j}^*)}{C_j(\thetabf)}\right\}\\
&\geq \frac{p^*_{I^*}C_{I^*}(\thetabf)}{\sum_{i\in[k]} p^*_iC_i(\thetabf)} \sum_{j\neq I^*}q_j\frac{{\rm KL}(\theta_{I^*},\bar{\theta}_{I^*,j}^*)}{C_{I^*}(\thetabf)} + \sum_{j\neq I^*} \frac{p^*_{j}C_{j}(\thetabf)}{\sum_{i\in[k]} p^*_iC_i(\thetabf)} q_j\frac{{\rm KL}(\theta_{j},\bar{\theta}_{I^*,j}^*)}{C_j(\thetabf)} \\
&=\sum_{j\neq I^*} q_j \left[\frac{p^*_{I^*}{\rm KL}(\theta_{I^*},\bar{\theta}_{I^*,j}^*)+p^*_j{\rm KL}(\theta_{j},\bar{\theta}_{I^*,j}^*)}{\sum_{i\in[k]} p^*_iC_i(\thetabf)}\right] \\
&= \sum_{j\neq I^*} q_j \cdot \Gamma_{\thetabf}(\bm{p}^*, \varthetabf^{*j}) \\
& = \min_{\varthetabf\in \overline{\rm Alt}(\thetabf)}\Gamma_{\thetabf}(\bm{p}^*,\varthetabf),
\end{align*}
where the inequality holds since the sum of $\left(\frac{p^*_{j}C_{j}(\thetabf)}{\sum_{i\in[k]} p^*_iC_i(\thetabf)}\right)_{j\in[K]}$ is 1; the penultimate equality follows from the definition of the alternative $\varthetabf^{*j}$ in~\eqref{eq:hard instance}; the last equality applies Proposition~\ref{prop:best responses to p*}.

The inequality above becomes equality if and only if,
\[
\sum_{j'\neq I^*}q_{j'}\frac{{\rm KL}\left(\theta_{I^*},\bar{\theta}_{I^*,j'}^*\right)}{C_{I^*}(\thetabf)} = q_j\frac{{\rm KL}(\theta_{j},\bar{\theta}_{I^*,j}^*)}{C_j(\thetabf)},\quad\forall j\neq I^*
\quad\iff\quad
q_j = \frac{C_j(\thetabf)}{{\rm KL}(\theta_{j},\bar{\theta}_{I^*,j}^*)}\Gamma_{\thetabf}(\bm{p}^*,\bm{q}^*),\quad\forall j\neq I^*.
\]
That is, to achieve the minimal value, for any $j\neq I^*$, the probability of playing the alternative $\varthetabf^{*j}$ must be $\bm{q}^*(\varthetabf^{*j})$ (defined in~\eqref{eq:optimal q} with the explicit expression in~\eqref{eq:optimal q exact}), which completes the proof.
\end{proof}

%% file: app_lower_bound.tex
\section{Proof of lower bound on cost} 
\label{app:lower bound proof}

In this appendix, we complete the proof of lower bound  on the growth rate of total cost.
We use the asymptotic notion $o_{\varsigma}(\cdot)$ and $O_{\varsigma}(\cdot)$ and , where the inclusion of a subscript $\varsigma$ in $o_{\varsigma}(\cdot)$ is meant to highlight that $\varsigma$ remains fixed as the population size $n$ varies, and that in $O_{\varsigma}(\cdot)$ highlights that the hidden constant could depend on the subscript $\varsigma$.

\begin{restatable}[Lower bound]{proposition}{LowerBound}
	\label{prop:lower bound}
	For any (consistent) policy $\pi \in \Pi$, 
	\begin{equation}
		\label{eq:lower bound}
		\forall\thetabf\in\Theta:\quad\liminf_{n\to \infty} \frac{\mathrm{Cost}_{\bm{\theta}}(n, \pi) }{\ln(n)} \geq \frac{1}{\Gamma_{\thetabf}(\bm{p}^*,\bm{q}^*)}.
	\end{equation}
\end{restatable}

We first show the following property satisfied by a consistent policy.
\begin{lemma}
\label{lem:consistency leads to pcs}
A consistent policy satisfies  
\begin{equation}
\label{eq:pcs_constraint}
\forall\thetabf\in\Theta: 
\quad \E_{\bm{\theta}}\left[\Delta_{\hat{I}_\tau}(\thetabf)\right] = \frac{e^{o_{\thetabf}(\ln(n))}}{n}.
\end{equation}
\end{lemma}
\begin{proof}
Consider a consistent policy $\pi\in\Pi$, which by definition, enjoys the property~\eqref{eq:uniformly good rule}: 
\[
\forall \thetabf\in\Theta:\quad
\ln\left(\mathrm{Cost}_{\bm{\theta}}(n, \pi)\right) = o_{\thetabf}(\ln(n)).
\]
This property implies that
\begin{align*}
\forall \thetabf\in\Theta:\quad
\ln\left(\E_{\bm{\theta}}\left[\tau\cdot\Delta_{\hat{I}_\tau}(\thetabf)\right]\right) 
\leq & \ln\left(\E_{\bm{\theta}}\left[\tau\right]\cdot \Delta_{\max}(\thetabf)\right) \\
\leq & \ln\left(\frac{\mathrm{Cost}_{\bm{\theta}}(n, \pi)}{C_{\min}(\thetabf)}\cdot\Delta_{\max}(\thetabf)\right) = o_{\thetabf}(\ln(n)),
\end{align*}
where $\Delta_{\max}(\thetabf)=\max_{i\in[k]}\Delta_i$ and $C_{\min}(\thetabf) = \min_{i\in[k]}C_{i}(\thetabf)>0$ by Assumption \ref{asm:sampling cost} (the within-experiment costs are positive),
and thus
\begin{align*}
\forall \thetabf\in\Theta:\quad
\E_{\bm{\theta}}\left[\Delta_{\hat{I}_\tau}(\thetabf)\right] 
\leq& 
\frac{\E_{\bm{\theta}}\left[ n\cdot\Delta_{\hat{I}_\tau}(\thetabf) + \sum_{t=0}^{\tau-1}C_{I_t}(\thetabf) \right]}{n} \\
=& \frac{\E_{\bm{\theta}}\left[\tau\cdot\Delta_{\hat{I}_\tau}(\thetabf)\right]+\mathrm{Cost}_{\bm{\theta}}(n,\pi)}{n}=\frac{e^{o_{\thetabf}(\ln(n))}}{n}.
\end{align*}
\end{proof}

Much of the argument is similar to those that have appeared in the literature on adaptive hypothesis testing \citep{chernoff1959sequential} or best-arm identification \citep{garivier2016optimal}. Here we sketch the main challenge which requires novelty.

\begin{remark}[Technical challenges overcome in the proof]\label{rem:lower_bound_proof}
	For a best-arm identification problem, \citep{garivier2016optimal} impose a uniform constraint on the probability of incorrect selection of the form
	\begin{equation}\label{eq:pcs-constraint-uniform}
		\sup_{\thetabf\in\Theta} \Prob_{\bm{\theta}}\left(\hat{I}_\tau \neq I^*(\thetabf)\right) \leq \delta.
	\end{equation}
	They show \eqref{eq:pcs-constraint-uniform} implies 
	\begin{equation}\label{eq:uniform-discrimination}
		\inf_{\varthetabf \in \mathrm{Alt}(\thetabf)}   \sum_{i\in[k]}\E_{\thetabf}\left[N_{\tau,i}\right] {\rm KL}(\theta_i,\vartheta_i) \geq d(\delta, 1-\delta) 
	\end{equation}
	holds for every $\thetabf \in \Theta$, where $d(p,q)=p\ln(p/q) + (1-p)\ln((1-p)/(1-q))$ is the KL-divergence between Bernoulli distributions. This can be used to lower bounds how many samples are collected from each arm, and through that, the total cost incurred.

	In our problem, it is not hard to deduce that under any consistent algorithm ${\rm Cost}_{\thetabf}(n,\pi) = o_{\thetabf}( n^{\epsilon} )$ for any $\epsilon>0$ and therefore, the probability of incorrect selection decays as $\Prob_{\bm{\theta}}\left(\hat{I}_\tau \neq I^*(\thetabf)\right) = o_{\thetabf}\left( \frac{1}{n^{1-\epsilon}} \right)$ for each $\thetabf$, where the inclusion of a subscript $\thetabf$ in $o_{\thetabf}(\cdot)$ is meant to highlight that $\thetabf$ remains fixed as the population size $n$ varies.
	This property can be rewritten as: for each $\thetabf>0$, there exists $c_{\thetabf}$ such that 
	\begin{equation}\label{eq:PCS-constraint-ours}
		\Prob_{\bm{\theta}}\left(\hat{I}_\tau \neq I^*(\thetabf)\right) \leq  \frac{c_{\thetabf}}{n^{1-\epsilon}}.
	\end{equation}
	Since nothing rules out that $\sup_{\thetabf} c_{\thetabf}=\infty$, one cannot easily derive a uniform lower bound on the probability of incorrect selection like \eqref{eq:pcs-constraint-uniform} and therefore the argument leading to \eqref{eq:uniform-discrimination} breaks down.  
	
	The main innovation in our lower bound argument is that we solve for the skeptic's worst-case distribution~${\bm q}^*$, as shown already above. Rather than take the infimum over $\varthetabf$ in the infinite set ${\rm Alt}(\thetabf)$ as in \eqref{eq:uniform-discrimination}, we derive similar inequality in terms of the expectation over the finitely supported worst-case distribution $\bm{q}^*$. 
\end{remark}

\paragraph{A restatement Proposition \ref{prop:lower bound}.}
We argue that the lower bound~\eqref{eq:lower bound} for any instance $\thetabf\in\Theta$ in Proposition~\ref{prop:lower bound} follows from the next proposition, which lower bounds the cumulative within-experiment cost for a policy ensuring that for any problem instance, the post-experiment per-person cost vanishes at a desired rate.

\begin{proposition}
\label{prop:garivier_lower}
If a policy $\pi$ satisfies the condition \eqref{eq:pcs_constraint},
then\footnote{It is easy to verify $O_{\thetabf}\left(\frac{e^{o_{\thetabf}(\ln(n))}}{n}\right) = \frac{e^{o_{\thetabf}(\ln(n))}}{n}$, so for simplicity, we write the condition \eqref{eq:pcs_constraint} without the redundant $O_{\thetabf}$. }
\[
\forall\thetabf\in\Theta: 
\quad \E_{\bm{\theta}} \left[ \sum_{t=0}^{\tau-1}C_{I_t}(\thetabf)\right] \geq \frac{1+o_{\thetabf}(1)}{\Gamma_{\thetabf}(\bm{p}^*,\bm{q}^*)}\ln(n).
\]
\end{proposition}

\begin{proof}[Proof of Proposition~\ref{prop:lower bound} (based on Proposition~\ref{prop:garivier_lower})]

By Lemma~\ref{lem:consistency leads to pcs}, a consistent policy $\pi\in\Pi$ satisfies the condition~\eqref{eq:pcs_constraint}. Then applying Proposition~\ref{prop:garivier_lower} immediately gives
\[
\forall\thetabf\in\Theta:\quad
\mathrm{Cost}_{\bm{\theta}}(n,\pi)\geq \E_{\bm{\theta}} \left[ \sum_{t=0}^{\tau-1}C_{I_t}(\thetabf)\right] \geq
\frac{1+o_{\thetabf}(1)}{\Gamma_{\thetabf}(\bm{p}^*,\bm{q}^*)}\ln(n),
\]
which completes the proof of Proposition~\ref{prop:lower bound}.

\end{proof}

The remaining of this appendix establishes a sequence of results that lead to the lower bound in Proposition~\ref{prop:garivier_lower}.

\subsection{A general change of measure argument}
Consider two states of nature $\thetabf, \varthetabf \in \Theta$ under which the best arms differ  (i.e. $I^*(\varthetabf) \neq I^*(\thetabf))$. The experimenter would like to reach the correct decision in each case. That is, if the experimenter stops at time $\tau$ and picks arm $\hat{I}_{\tau}$, they would like both $\Prob_{\thetabf}\left(\hat{I}_{\tau}  = I^*(\thetabf)\right)$ and $\Prob_{\varthetabf}\left(\hat{I}_{\tau}  = I^*(\varthetabf)\right)$ to be close to one. Stated differently, they would like there to be a large divergence between the distributions of $\hat{I}_{\tau}$ under $\thetabf$ and $\varthetabf$. The next result makes this formal.

For the following results, we extend the definition of the variables $\hat{I}_{\tau}$ and $\mathcal{H}_{\tau}$ on sample paths where $\tau=\infty$.\footnote{Under our problem formulation, the stopping time is always bounded by the population size, while For other problems, it is possible that the stopping time $\tau=\infty$, so we introduce $\hat{I}_{\infty}$ and $\mathcal{H}_{\infty}$.}
When $\tau=\infty$, define $\hat{I}_{\tau}=\emptyset \notin \{1,\ldots, k\}$, indicating that $\hat{I}_\tau$ cannot correctly identify the best arm, and take $\mathcal{H}_{\infty} = \{(I_t, Y_{t,I_t})\}_{t \in \mathbb{N}_0}$ to be the infinite history generated by the allocation rule. 

\begin{lemma}[Bretagnolle–Huber inequality {\citep[Theorem 14.2]{lattimore2020bandit}}]
\label{lem:Bretagnolle–Huber inequality}
	For any $\thetabf,\varthetabf \in \Theta$,
	\[
	\Prob_{\thetabf}\left( \hat{I}_\tau \neq I^*(\thetabf) \right) + \Prob_{\varthetabf}\left( \hat{I}_\tau = I^*(\thetabf) \right) \geq \frac{1}{2} \exp\left\{ - \KL\left(\Prob_{\thetabf}(\mathcal{H}_{\tau}  \in \cdot) , \Prob_{\varthetabf}(\mathcal{H}_{\tau}  \in \cdot )\right)   \right\}.
	\]
\end{lemma}

\begin{corollary}\label{cor:kl-growth-under-pcs}
	Suppose the condition \eqref{eq:pcs_constraint} holds. For any $\thetabf,\varthetabf\in \Theta$ such that $I^*(\varthetabf) \neq I^*(\thetabf)$,  
	\[
	 \frac{\KL\left(\Prob_{\thetabf}(\mathcal{H}_{\tau}  \in \cdot) ,  \Prob_{\varthetabf}(\mathcal{H}_{\tau}  \in \cdot)\right)}{\ln(n)} \geq  1+o_{\thetabf,\varthetabf}(1).
	\]
\end{corollary}
\begin{proof}[Proof of Corollary~\ref{cor:kl-growth-under-pcs}]

The LHS in Lemma~\ref{lem:Bretagnolle–Huber inequality} can be upper bounded as follows,
\begin{align*}
\Prob_{\thetabf}\left( \hat{I}_\tau \neq I^*(\thetabf) \right) + \Prob_{\varthetabf}\left( \hat{I}_\tau = I^*(\thetabf) \right) 
\leq& \frac{\E_{\bm{\theta}}\left[\Delta_{\hat{I}_\tau}(\thetabf)\right]}{\min_{j\neq I^*(\thetabf)}\Delta_j(\thetabf)} + \frac{\E_{\bm{\vartheta}}\left[\Delta_{\hat{I}_\tau}(\thetabf)\right]}{\min_{j\neq I^*(\varthetabf)}\Delta_j(\varthetabf)}\\
=& O_{\thetabf}\left(\frac{e^{o_{\thetabf}(\ln(n))}}{n}\right) + O_{\varthetabf}\left(\frac{e^{o_{\varthetabf}(\ln(n))}}{n}\right) = O_{\thetabf,\varthetabf}\left(\frac{e^{o_{\thetabf,\varthetabf}(\ln(n))}}{n}\right),
\end{align*}
where the denominators in the first inequality are positive because of $\thetabf,\varthetabf\in\Theta$ and Assumption \ref{asm:post-experimentation cost}, and the second equality follows from the condition~\eqref{eq:pcs_constraint}.
Then taking the logarithm of both sides in Lemma \ref{lem:Bretagnolle–Huber inequality} and dividing each side by $\ln(n)$ completes the proof.
\end{proof}

If the experimenter wants to reach a correct decision with high probability, they need to gather a collection of observations (or history) that discriminates between $\thetabf$ and $\varthetabf$ in the sense of making the divergence above large. Thankfully, this divergence has an explicit expression, which is given in the lemma below.

\begin{lemma}
\label{lem:kl_formula}
	For any $\thetabf, \varthetabf \in \Theta$, if $\Prob_{\thetabf}(\tau < \infty)=1$,
	then
	\begin{align*}
		\KL( \Prob_{\thetabf}\left( \mathcal{H}_\tau \in \cdot \right) ,  \Prob_{\varthetabf}\left( \mathcal{H}_\tau\in \cdot \right)   )  
  = \sum_{i\in[k]}\E_{\thetabf}\left[N_{\tau,i}\right] {\rm KL}(\theta_i,\vartheta_i).
	\end{align*}
\end{lemma}

\begin{proof}[Proof of Lemma~\ref{lem:kl_formula}]
	Recall that $\mathcal{H}_\tau = \left\{(I_{\ell}, Y_{\ell,I_{\ell}})\right\}_{\ell\in\{0,\ldots,\tau-1\}}$. 
 The chain rule for stopping times in~Lemma~\ref{lem:chain_rule} implies
	\begin{align*}
		&\KL\left( \Prob_{\thetabf}( \mathcal{H}_\tau \in \cdot) , \Prob_{\varthetabf}( \mathcal{H}_\tau \in \cdot) \right)\\
  =& \E_{\thetabf} \left[ \sum_{\ell=0}^{\tau-1} \KL\left( \Prob_{\thetabf}\left( (I_{\ell}, Y_{\ell,I_{\ell}}) \in \cdot \mid \mathcal{H}_{\ell}\right) , \Prob_{\varthetabf}\left( (I_{\ell}, Y_{\ell,I_{\ell}}) \in \cdot \mid \mathcal{H}_{\ell}\right) \right)   \right] \\ 
		=& \E_{\thetabf} \left[ \sum_{\ell=0}^{\tau-1} \KL\left( \Prob_{\thetabf}\left( I_{\ell} \in \cdot \mid \mathcal{H}_{\ell}\right), \Prob_{\varthetabf}\left( I_{\ell} \in \cdot \mid \mathcal{H}_{\ell}\right) \right)   \right]\\
		&+ \E_{\thetabf} \left[ \sum_{\ell=0}^{\tau-1} \KL\left( \Prob_{\thetabf}\left( Y_{\ell,I_{\ell}} \in \cdot \mid \mathcal{H}_{\ell}, I_{\ell}\right) , \Prob_{\varthetabf}\left( Y_{\ell,I_{\ell}} \in \cdot \mid \mathcal{H}_{\ell}, I_{\ell}\right) \right)   \right]\\
		=& \E_{\thetabf} \left[ \sum_{\ell=0}^{\tau-1} \KL\left(\theta_{I_{\ell-1}},\vartheta_{I_{\ell-1}}\right)  \right]\\
  =& \sum_{i\in[k]}\E_{\thetabf}\left[N_{\tau,i}\right] \KL(\theta_i,\vartheta_i),
	\end{align*}
	where the second equality above applies the chain rule, and the penultimate equality uses that conditioned on $\mathcal{H}_{\ell}$,  the  distribution of $I_{\ell}$ under $\mathbb{P}_{\thetabf}$ is the same as that under $\Prob_{\varthetabf}$.  
\end{proof}

Now we state and prove the chain rule with stopping times, which was used above in the proof of Lemma~\ref{lem:kl_formula}.
\begin{lemma}[Chain rule with stopping times {\citep[Problem 14.13]{lattimore2020bandit}}]
\label{lem:chain_rule}
	Consider two probability spaces $(\Omega, \Fc, \mathbb{P})$ and $(\Omega, \Fc, \mathbb{Q})$ where $\mathbb{Q}$ is absolutely continuous with respect to $\mathbb{P}$. Take $\tau$ to be a $\mathbb{P}$-almost-surely finite stopping time adapted to $(Z_\ell)_{\ell\in \mathbb{N}_1}$ (i.e., $\Prob(\tau<\infty) = 1$). Then, 
	\[
	\KL\left( \Prob(Z_{1:\tau} \in \cdot ) ,  \mathbb{Q}(Z_{1:\tau} \in \cdot )  \right)  = \E_{\mathbb{P}}\left[\sum_{\ell=1}^{\tau} 
	\KL\left( \Prob\left(Z_{\ell} \in \cdot \mid Z_{1:(\ell-1)} \right) ,  \mathbb{Q}\left(Z_{\ell} \in \cdot  \mid Z_{1:(\ell-1)}\right)  \right)  \right],
	\]
 where $Z_{1:(\ell-1)} \triangleq (Z_1, \ldots, Z_{\ell-1})$ with $Z_{1:0}\triangleq \emptyset$.
\end{lemma}
\begin{proof}[Proof of Lemma \ref{lem:chain_rule}]
	The result follows by applying the usual chain rule to the censored random variables: $\tilde{Z}_{\ell} \triangleq Z_\ell \ind(\ell \leq \tau)$ for $\ell\in\mathbb{N}_1$. It is clear that the sequence $Z_{1:\tau} = (Z_1, \ldots, Z_{\tau})$ contains the same information as $\tilde{Z}_{1:\infty} = (\tilde{Z}_1, \tilde{Z}_2, \ldots)$. Formally, there exists a function $f$ with $\tilde{Z}_{1:\infty} = f(Z_{1:\tau})$ and $Z_{1:\tau}= f^{-1}(\tilde{Z}_{1:\infty})$. It follows by the data-processing inequality that 
	\[ 
	\KL\left( \Prob(Z_{1:\tau} \in \cdot ) , \mathbb{Q}(Z_{1:\tau} \in \cdot )  \right)  = \KL\left( \Prob(\tilde{Z}_{1:\infty} \in \cdot ) ,  \mathbb{Q}(\tilde{Z}_{1:\infty} \in \cdot )  \right).
	\]
	Now, 
	\begin{align*}
		\KL\left( \Prob(\tilde{Z}_{1:\infty} \in \cdot ) ,  \mathbb{Q}(\tilde{Z}_{1:\infty} \in \cdot )  \right) &= \lim_{n\to \infty}  \KL\left( \Prob(\tilde{Z}_{1:n} \in \cdot ) ,  \mathbb{Q}(\tilde{Z}_{1:n} \in \cdot )  \right)\\
		&=\lim_{n\to \infty}  \E_{\Prob}\left[ \sum_{\ell=1}^{n} \KL\left( \Prob\left(\tilde{Z}_\ell \in \cdot \mid \tilde{Z}_{1:(\ell-1)}  \right) ,  \mathbb{Q}\left(\tilde{Z}_\ell \in \cdot \mid \tilde{Z}_{1:(\ell-1)} \right)  \right) \right] \\
		&=\lim_{n\to \infty}  \E_{\Prob}\left[ \sum_{\ell=1}^{\tau \wedge n} \KL\left( \Prob\left(\tilde{Z}_\ell \in \cdot \mid \tilde{Z}_{1:(\ell-1)}  \right) ,  \mathbb{Q}\left(\tilde{Z}_\ell \in \cdot \mid \tilde{Z}_{1:(\ell-1)} \right)  \right) \right]\\
		&=\E_{\Prob}\left[ \sum_{\ell=1}^{\tau} \KL\left( \Prob\left(\tilde{Z}_\ell \in \cdot \mid \tilde{Z}_{1:(\ell-1)}  \right) ,  \mathbb{Q}\left(\tilde{Z}_\ell \in \cdot \mid \tilde{Z}_{1:(\ell-1)} \right)  \right) \right]\\
		&=\E_{\Prob}\left[ \sum_{\ell=1}^{\tau} \KL\left( \Prob\left(Z_\ell \in \cdot \mid Z_{1:(\ell-1)} \right) ,  \mathbb{Q}\left(Z_\ell \in \cdot \mid Z_{1:(\ell-1)} \right)  \right) \right].
	\end{align*}
 The first equality is \citet[Corollary 5.2.5]{gray2011entropy}. The second equality is the usual chain rule for KL divergence. The third equality recognizes that conditioned on $\ell > \tau$, $\tilde{Z}_\ell =0$ almost surely under both $\Prob(\cdot)$ and $\mathbb{Q}(\cdot)$ and so 
 \[
 \KL\left( \Prob(\tilde{Z}_\ell \in \cdot \mid \tilde{Z}_{1:(\ell-1)}  ) ,  \mathbb{Q}(\tilde{Z}_\ell \in \cdot \mid \tilde{Z}_{1:(\ell-1)} )  \right) =0.
 \]
 The fourth equality uses that KL divergence is non-negative and applies the monotone convergence theorem to interchange the limit and expectation. 
\end{proof}

Combining Corollary \ref{cor:kl-growth-under-pcs} and Lemma \ref{lem:kl_formula}, we have the following result:
\begin{corollary}
\label{cor:kl-growth-under-pcs and kl_formula}
Suppose the condition \eqref{eq:pcs_constraint} holds. For any $\thetabf, \varthetabf\in\Theta$ such that $I^*(\varthetabf) \neq I^*(\thetabf)$, if $\Prob_{\thetabf}(\tau < \infty)=1$, then
\[
\frac{\sum_{i\in[k]}\E_{\thetabf}\left[N_{\tau,i}\right] {\rm KL}(\theta_i,\vartheta_i)}{\ln(n)} \geq 1+o_{\thetabf,\varthetabf}(1).
\]
\end{corollary}

\subsection{Completing the proof of Proposition \ref{prop:garivier_lower}} 

Fix $\thetabf\in\Theta$. A slight subtlety arises in completing the proof:
to apply Corollary~\ref{cor:kl-growth-under-pcs and kl_formula}, we have to consider $\varthetabf\in\Theta$, while the alternatives $\{\varthetabf^{*j}\}_{j\neq I^*}$ (defined in~\eqref{eq:hard instance}) do not belong to $\Theta$ (since each of them does not have a unique best arm).
To complete the proof, we first consider arbitrary instances~$\{\varthetabf^j\}_{j\neq I^*}$ with each $\varthetabf^j = (\vartheta^j_1,\ldots,\vartheta^j_k)\in{\rm Alt}(\thetabf)=\{\varthetabf\in \Theta \,:\, I^*(\varthetabf) \neq I^*(\thetabf)\}$, and will take a limit as $\varthetabf^j\to\varthetabf^{*j}$ for any $j\neq I^*$ in the very end.

Recall that the skeptic's unique equilibrium strategy $\bm{q}^*$ is supported on the alternatives~$\{\varthetabf^{*j}\}_{j\neq I^*}$ with probabilities $(\bm{q}^*(\varthetabf^{*j}))_{j\neq I^*}$.
Then applying Corollary \ref{cor:kl-growth-under-pcs and kl_formula} yields\footnote{Notice that it is critical that the skeptic's equilibrium strategy is supported over \emph{a finite} set of instances, as this allows us to pass a limit through the sum and inf (or sup). Consider a family of functions $\{f_j\}_{j\in \mathcal{J}}$ such that for any $j\in \mathcal{J}$, $\lim_{n\to \infty} f_{j}(n)=0$.
When $\mathcal{J}$ is finite, it follows that $\lim_{n\to \infty}\sum_{j \in \mathcal{J}} f_{j}(n) = 0$ and $\lim_{n\to\infty}\inf_{j\in\mathcal{J}}f_{j}(n) = 0$, but these may not hold if $\mathcal{J}$ is infinite.}
\begin{align*}
\left[1+o_{\thetabf,\{\varthetabf^{j}\}_{j\neq I^*}}(1)\right]\ln(n) &\leq \sum_{j\neq I^*} \bm{q}^*(\varthetabf^{*j}) \sum_{i\in[k]}\E_{\thetabf}[N_{\tau,i}] {\rm KL}(\theta_i,\vartheta_i^{j}) \\
&= \E_{\thetabf} \left[ \sum_{t=0}^{\tau-1}C_{I_t}(\thetabf)\right] 
\sum_{j\neq I^*(\thetabf)} \bm{q}^*(\varthetabf^{*j}) \frac{\sum_{i\in[k]}\E_{\thetabf}[N_{\tau,i}]{\rm KL}(\theta_i,\vartheta_i^{j})}{\sum_{i\in[k]}\E_{\thetabf}[N_{\tau,i}]C_i(\thetabf)} \\
&= \E_{\thetabf} \left[ \sum_{t=0}^{\tau-1}C_{I_t}(\thetabf)\right] 
\sum_{j\neq I^*(\thetabf)} \bm{q}^*(\varthetabf^{*j}) \frac{\sum_{i\in[k]}\frac{\E_{\thetabf}[N_{\tau,i}]}{\E_{\thetabf}[\tau]}{\rm KL}(\theta_i,\vartheta_i^{j})}{\sum_{i\in[k]}\frac{\E_{\thetabf}[N_{\tau,i}]}{\E_{\thetabf}[\tau]}C_i(\thetabf)} \\
&\leq  \E_{\thetabf} \left[ \sum_{t=0}^{\tau-1}C_{I_t}(\thetabf)\right] 
\sup_{\bm{p}\in\Sigma_k}\sum_{j\neq I^*(\thetabf)} \bm{q}^*(\varthetabf^{*j}) \frac{\sum_{i\in[k]}p_i{\rm KL}(\theta_i,\vartheta_i^{j})}{\sum_{i\in[k]}p_iC_i(\thetabf)} \\
&= \E_{\thetabf} \left[ \sum_{t=0}^{\tau-1}C_{I_t}(\thetabf)\right] \sup_{\bm{p}\in\Sigma_k}
\sum_{j\neq I^*(\thetabf)} \bm{q}^*(\varthetabf^{*j})  \Gamma_{\thetabf}(\bm{p},\varthetabf^{j}),
\end{align*}
where the first equality holds since $\E_{\thetabf} \left[ \sum_{t=0}^{\tau-1}C_{I_t}(\thetabf)\right] = \sum_{i\in[k]}\E_{\thetabf}[N_{\tau,i}]C_i(\thetabf)$, and the last inequality follows from $\sum_{i\in[k]} \frac{\E_{\thetabf}[N_{\tau,i}]}{\E_{\thetabf}[\tau]} = 1$.
Dividing each side by $\ln(n)$ and taking $n\to\infty$ gives the following inequality:
\begin{align*}
\liminf_{n\to\infty} \frac{\E_{\thetabf} \left[ \sum_{t=0}^{\tau-1}C_{I_t}(\thetabf)\right]}{\ln(n)} \geq \frac{1}{\sup_{\bm{p}\in\Sigma_k}
\sum_{j\neq I^*(\thetabf)} \bm{q}^*(\varthetabf^{*j})  \Gamma_{\thetabf}(\bm{p},\varthetabf^{j})}.
\end{align*}
As $\varthetabf^j\to\varthetabf^{*j}$ for any $j\neq I^*$, the RHS above is approaching to
\[
\frac{1}{\sup_{\bm{p}\in\Sigma_k}
\sum_{j\neq I^*(\thetabf)} \bm{q}^*(\varthetabf^{*j})  \Gamma_{\thetabf}(\bm{p},\varthetabf^{*j})}
= \frac{1}{\sup_{\bm{p}\in\Sigma_k}
\Gamma_{\thetabf}(\bm{p},\bm{q}^*)} = \frac{1}{\Gamma_{\thetabf}(\bm{p}^*,\bm{q}^*)},
\]
where the last equality applies \eqref{eq:equilibrium}. 
This completes the proof of Proposition \ref{prop:garivier_lower}.

%% file: app_if_allocation.tex
\section{Proof of the first statement in Theorem~\ref{thm:efficient-p-general}}
\label{app:proof of sufficient condition}

We restate Theorem~\ref{thm:efficient-p-general} as follows, and prove in this appendix the first statement in Theorem~\ref{thm:efficient-p-general}. The proof of the second statement is deferred to Appendix \ref{app:proof of sufficient condition being almost necessary}.

\SufficientCondition*

Recall the lower bound~\eqref{eq:lower bound} in Proposition~\ref{prop:lower bound}: for any consistent policy $\varpi\in\Pi$ (Definition \ref{def:uniformly good rule}),
\[
		\forall\thetabf\in\Theta:\quad\liminf_{n\to \infty} \frac{\mathrm{Cost}_{\bm{\theta}}(n, \varpi) }{\ln(n)} \geq \frac{1}{\Gamma_{\thetabf}(\bm{p}^*,\bm{q}^*)}.
\]
By Lemma \ref{lem:cost-scaling}, this immediately leads to the following sufficient conditions for a policy being universally efficient:

\begin{restatable}[A sufficient condition of universal efficiency]{proposition}{AlternativeSufficientCondition}
\label{prop:sufficient conditions of universal efficiency}
A policy ${\pi}$ is universally efficient (Definition \ref{def:universal-efficiency}) if
\begin{equation}
\label{eq:alternative sufficient (and necessary) condition of universal efficiency}
\forall\thetabf\in\Theta:\quad\limsup_{n\to \infty} \frac{\mathrm{Cost}_{\bm{\theta}}(n, {\pi}) }{\ln(n)} \leq \frac{1}{\Gamma_{\thetabf}(\bm{p}^*,\bm{q}^*)}.
\end{equation}
\end{restatable}

With Proposition~\ref{prop:sufficient conditions of universal efficiency}, we establish the first statement in Theorem~\ref{thm:efficient-p-general} by proving~\eqref{eq:alternative sufficient (and necessary) condition of universal efficiency}.
Next we present a more general result as follows, and the rest of this appendix proves this result.

\begin{Theorem}[A more general result]
\label{thm:efficient-p-general restatement}
Fix $\thetabf\in\Theta$. If the policy $\pi$ implements Algorithm \ref{alg:general-template-exp-family} that takes an (anytime) allocation rule (applied without a stopping rule) satisfying 
$
\bm{p}_t\Sto \widetilde{\bm{p}} > \bm{0} 
$,
then
\[
\limsup_{n\to \infty} \frac{\mathrm{Cost}_{\bm{\theta}}(n, \pi) }{\ln(n)} \leq \frac{1}{\min_{\varthetabf\in \overline{\rm Alt}(\thetabf)}\Gamma_{\thetabf}(\widetilde{\bm{p}},\varthetabf)},
\]
which becomes~\eqref{eq:alternative sufficient (and necessary) condition of universal efficiency} for $\widetilde{\bm{p}} = \bm{p}^*$.
\end{Theorem}

Recall that the total cost under $\pi$ can be decomposed into the cumulative within-experiment and post-experiment costs,
\[
\mathrm{Cost}_{\bm{\theta}}(n, {\pi}) \triangleq \E_{\bm{\theta}}\left[ \sum_{t=0}^{\tau-1}C_{I_t}(\thetabf)\right] + \E_{\bm{\theta}}\left[(n-\tau)\Delta_{\hat{I}_\tau}(\thetabf) \right].
\]
For the rest of this appendix, we prove Theorem~\ref{thm:efficient-p-general restatement} by upper bounding the within-experiment and post-experiment costs, respectively.
For notational convenience, besides $C_{\min}(\thetabf) = \min_{i\in[k]}C_i(\thetabf)$ introduced in Assumption~\ref{asm:sampling cost}, we write
\[
C_{\max}(\thetabf) = \max_{i\in[k]}C_i(\thetabf)
\quad\text{and}\quad
\Delta_{\max}(\thetabf) = \max_{i\in[k]}\Delta_i(\thetabf).
\]

\subsection{Bounding cumulative post-experiment cost}\label{app:post-experiment-cost}

This subsection proves the following result that bounds the cumulative post-experiment cost.
\begin{proposition}
\label{prop:bound on post-experimentation cost}
Fix $\thetabf\in\Theta$ and $n\geq \frac{3}{k-1}$. For any allocation rule, Algorithm \ref{alg:general-template-exp-family} ensures that 
\[
\Prob_{\thetabf}(\tau < n,\, \hat{I}_{\tau}\neq I^*)\leq \frac{1}{n}, 
\quad\text{and thus}\quad
\E_{\thetabf}\left[(n-\tau)\Delta_{\hat{I}_{\tau}}\right] \leq \Delta_{\max}(\thetabf).
\]
\end{proposition}

Observe that the post-experiment cost $(n-\tau)\Delta_{\hat{I}_{\tau}}$ can be positive only if (1) the experiment stops before observing all the population (i.e. $\tau < n$) and (2) a suboptimal arm $\hat{I}_{\tau}\neq I^*$ is recommended. Hence,
\[
\E_{\thetabf}\left[(n-\tau)\Delta_{\hat{I}_{\tau}}\right] \leq \Prob_{\thetabf}(\tau < n,\, \hat{I}_{\tau}\neq I^*)\cdot (n\cdot \Delta_{\max}(\thetabf)).
\]

The remainder of this subsection proves
$
\Prob_{\thetabf}(\tau < n,\, \hat{I}_{\tau}\neq I^*) \leq \frac{1}{n}.
$
The proof relies on the following result in \citet{Kaufmann2021martingale}, which controls the self-normalized sum (for the best arm $I^*$ and a suboptimal arm $j\neq I^*$),
$
N_{t,j}\KL(m_{t,j},\theta_j) + N_{t,I^*}\KL(m_{t,I^*},\theta_{I^*}),
$
under an adaptive allocation rule.

\begin{lemma}[{\citet[Subsection 5.1]{Kaufmann2021martingale}}]
\label{lem:corollary of Theorem 7 in Kaufmann2021martingale}
Fix $\thetabf\in\Theta$ and  $\delta > 0$. For any allocation rule,
\[
\Prob\left(\exists t\in\mathbb{N}_1,\,\exists j\neq I^* \,:\, \min_{i\in[k]} N_{t,i}\geq 1, \, N_{t,j}\KL(m_{t,j},\theta_j) + N_{t,I^*}\KL(m_{t,I^*},\theta_{I^*}) \geq \hat{c}_t(\delta) \right) \leq \delta,
\]
where for $t\in\mathbb{N}_1$,
\begin{equation}
\label{eq:non-explicit threshold}
\hat{c}_t(\delta) = 2\mathcal{C}_{\mathrm{exp}}\left(\frac{\ln\left(\frac{k-1}{\delta}\right)}{2}\right) + 6\ln\left(\ln\left(\frac{t}{2}\right)+1\right).
\end{equation}
\end{lemma}

For completeness, we define the calibration function $\mathcal{C}_{\mathrm{exp}}$ in \eqref{eq:non-explicit threshold} as follows by 
first introducing the following functions:
\begin{enumerate}
    \item For $u\geq 1$, let $h(u) = u - \ln(u)$ and its inverse function $h^{-1}$. As stated in \citet[Proposition 8]{Kaufmann2021martingale}, $h$ is an increasing function with both domain and range being $[1,\infty)$. So is its inverse function $h^{-1}$.
    \item Further define
\begin{equation}
\label{eq:tilde h}
\tilde{h}(y) = \begin{cases}
    \exp\left(\frac{1}{h^{-1}(y)}\right)\cdot h^{-1}(y) & \hbox{if } y\ge h\left(\frac{1}{\ln(3/2)}\right)\approx 1.564, \\
    (3/2)[y - \ln(\ln(3/2))]        & \hbox{if } 0< y< h\left(\frac{1}{\ln(3/2)}\right). 
\end{cases}
\end{equation}
\end{enumerate}
The calibration function $\mathcal{C}_{\mathrm{exp}}$ is defined as, for $x\geq 0$,
\begin{equation}
\label{eq:C_exp}
\mathcal{C}_{\mathrm{exp}}(x) = 2 \tilde{h}\left(y_x\right) \quad\text{where}\footnote{The definition of $\mathcal{C}_{\mathrm{exp}}$ in \eqref{eq:C_exp} includes the constant $\frac{\pi^2}{3}$, which equals $2\zeta(2)$ in \citet[Equation~(10)]{Kaufmann2021martingale}, where $\zeta(2)=\sum_{n=1}^\infty n^{-2} = \frac{\pi^2}{6}\approx 1.645$ is the value of Riemann zeta function when applied to 2.}\quad
y_x = \frac{h^{-1}(x+1) + \ln\left(\frac{\pi^2}{3}\right)}{2}. 
\end{equation}

To simply the less explicit threshold \eqref{eq:non-explicit threshold}, 
we provide a simpler, though looser, upper bound on it which we use in the rest of the analysis.
\begin{proposition}
\label{prop:C_exp upper bound}
When $x\geq 0.52$,
\[
\mathcal{C}_{\mathrm{exp}}(x) \leq x + 3\ln(x+2) + 7.
\]
\end{proposition}
The proof of Proposition~\ref{prop:C_exp upper bound} is deferred to Appendix \ref{subsec:C_exp upper bound proof}. The explicit upper bound on $\mathcal{C}_{\mathrm{exp}}(x)$ in Proposition~\ref{prop:C_exp upper bound} yields that
when $\frac{\ln\left(n(k-1)\right)}{2}\geq 0.52$, for any $t\in\mathbb{N}_1$, the threshold~\eqref{eq:non-explicit threshold} with $\delta = \frac{1}{n}$ can be upper bounded as follows,
\begin{equation}
\label{eq:threshold comparison}
\hat{c}_t\left(\frac{1}{n}\right) = 2\mathcal{C}_{\mathrm{exp}}\left(\frac{\ln(n(k-1))}{2}\right) + 6\ln\left(\ln\left(\frac{t}{2}\right)+1\right) \leq \gamma_t(n),
\end{equation}
where $\gamma_t(n)$ is the threshold~\eqref{eq:tight threshold T} with the expression,
\[
\gamma_t(n) =  \ln(n) + \ln(k-1) + 6\ln\left(\frac{\ln(n)+\ln(k-1)}{2}+2\right) +  6\ln\left(\ln\left(\frac{t}{2}\right)+1\right) + 14.
\]

With~\eqref{eq:threshold comparison} and Lemma~\ref{lem:corollary of Theorem 7 in Kaufmann2021martingale}, we are ready to complete the proof of Proposition~\ref{prop:bound on post-experimentation cost}.

\begin{proof}[Proof of Proposition \ref{prop:bound on post-experimentation cost}]
The stopping rule in Algorithm \ref{alg:general-template-exp-family} ensures that the experiment has not stopped until all arms are sampled at least once. Hence,
\begin{align*}
     &\Prob_{\thetabf}(\tau < n,\, \hat{I}_{\tau}\neq I^*) \\
    \leq& \Prob\left(\exists t\in \mathbb{N}_1,\,\exists j\neq I^* \,:\, \min_{i\in[k]} N_{t,i}\geq 1,\, m_{t,j} > m_{t,I^*}, \, t\cdot D_{{\bm m}_t , j, I^*}({p}_{t,j}, {p}_{t,I^*})\geq \gamma_t(n)\right) \\
    \leq& \Prob\left(\exists t\in \mathbb{N}_1,\,\exists j\neq I^* \,:\, \min_{i\in[k]} N_{t,i}\geq 1,\, N_{t,j}\KL(m_{t,j},\theta_j) + N_{t,I^*}\KL(m_{t,I^*},\theta_{I^*}) \geq \gamma_t(n) \right)\\
    \leq& \Prob\left(\exists t\in \mathbb{N}_1,\,\exists j\neq I^* \,:\, \min_{i\in[k]} N_{t,i}\geq 1,\,N_{t,j}\KL(m_{t,j},\theta_j) + N_{t,I^*}\KL(m_{t,I^*},\theta_{I^*}) \geq \hat{c}_t\left(\frac{1}{n}\right) \right)\\
    \leq& \frac{1}{n},
\end{align*}
where the second inequality follows from that when $m_{t,j} > m_{t,I^*}$ and $\theta_j < \theta_{I^*}$ together implies, 
\[
D_{{\bm m}_t , j, I^*}({p}_{t,j},p_{t,I^*}) = \inf_{\vartheta_j < \vartheta_{I^*}}  \, p_{t,j} \cdot {\rm KL}\left(m_{t,j} , \vartheta_j\right) + p_{t,I^*} \cdot {\rm KL}\left(m_{t,I^*} , \vartheta_{I^*}\right)\leq p_{t,j}\KL(m_{t,j},\theta_j) + p_{t,I^*}\KL(m_{t,I^*},\theta_{I^*});
\]
the third inequality holds since $\hat{c}_t\left(\frac{1}{n}\right) \leq \gamma_t(n)$ in \eqref{eq:threshold comparison} when $n\geq \frac{3}{k-1}$;
the last inequality applies Lemma~\ref{lem:corollary of Theorem 7 in Kaufmann2021martingale}.

Recall that the cumulative post-experiment cost $(n-\tau)\Delta_{\hat{I}_{\tau}}$ can be positive only if both $\tau < n$ and $\hat{I}_{\tau}\neq I^*$ happen. Hence,
\[
\E_{\thetabf}\left[(n-\tau)\Delta_{\hat{I}_{\tau}}\right] \leq \Prob_{\thetabf}(\tau < n,\, \hat{I}_{\tau}\neq I^*)\cdot n\cdot \Delta_{\max}(\thetabf) =  \Delta_{\max}(\thetabf).
\]
This completes the proof of Proposition \ref{prop:bound on post-experimentation cost}.

\end{proof}

\subsection{Bounding cumulative within-experiment cost}

Fix $\thetabf\in\Theta$.
This subsection proves a general result that bounds the cumulative within-experiment cost for Algorithm \ref{alg:general-template-exp-family} if the allocation rule satisfies that the empirical allocation $\bm{p}_t$, under any produced indefinite-allocation sample paths~\eqref{eq:allocaiton-only-sample}, converges strongly to some deterministic probability vector $\widetilde{\bm{p}} = (\widetilde{{p}}_1,\ldots,\widetilde{{p}}_k) >~\bm{0}$. 

\begin{proposition}
\label{prop:implication of strong convergence of allocation and our stopping rule}
Consider Algorithm \ref{alg:general-template-exp-family} with the (anytime) allocation rule satisfying that 
the empirical allocation ${\bm p}_t$, under any produced indefinite-allocation sample paths~\eqref{eq:allocaiton-only-sample}, converges strongly (Definition \ref{def: strong convergence}) to some deterministic probability vector $\widetilde{\bm{p}} > \bm{0}$.
Then, for any $\epsilon\in \left(0, \frac{\min_{\varthetabf\in \overline{\rm Alt}(\thetabf)}\Gamma_{\thetabf}(\widetilde{\bm{p}},\varthetabf)}{2}\right)$,
if the population size $n$ is large enough such that
\begin{equation}
\label{eq:large enough population size}
\left[\min_{\varthetabf\in \overline{\rm Alt}(\thetabf)}\Gamma_{\thetabf}(\widetilde{\bm{p}},\varthetabf) - 2\epsilon\right]\cdot C_{\min}(\thetabf)\cdot (n-1)  \geq \ln(n) + 6\ln(\ln(n)),
\end{equation}
then the cumulative within-experiment cost is bounded as follows,
\[
\sum_{\ell=0}^{\tau-1}C_{I_\ell}(\thetabf) \leq \frac{\ln(n) + 6\ln(\ln(n))}{\Gamma_{\thetabf}(\bm{p}^*,\bm{q}^*) - 2\epsilon} + C_{\max}(\thetabf)\cdot \left({T}_\epsilon + 1\right),
\]
where ${T}_\epsilon \in \mathcal{L}^1$ (i.e. $\E[{T}_\epsilon] < \infty$) is a random time  that is independent of the population size $n$.
\end{proposition}

The proof of Proposition \ref{prop:implication of strong convergence of allocation and our stopping rule} is established in a sequence of results.

\subsubsection{Sufficient exploration implies convergence of mean estimations}
Since $\bm{p}_t\Sto \widetilde{\bm{p}} > \bm{0}$, the play count for each arm grows linearly with $t$  as it grows large.
The next proposition shows that such sufficient exploration implies convergence of mean estimations.
\begin{proposition}
\label{prop:sufficient exploration implies convergence of mean estimations}
Fix $\thetabf\in\Theta$. If an allocation rule satisfies that
there exist $\rho > 0$, $\alpha > 0$ and
${T}\in\mathcal{L}^1$ such that for any $t\geq {T}$, $\min_{i\in[k]} N_{t,i}\geq  \rho \cdot t^{\alpha}$, then $\bm{m}_t \Sto \thetabf$.
\end{proposition}

To simplify the proof of Proposition~\ref{prop:sufficient exploration implies convergence of mean estimations}, we adopt the following standard tool of analyzing allocation rules in bandit literature. 
\paragraph{Latent reward table.} Imagine writing code to simulate an allocation rule. One way is to generate the random reward $Y_{t,I_t}$ after observing the selected arm $I_t$. An equivalent way is to generate all the randomness upfront, using what \citet[Section 4.6]{lattimore2020bandit} call a \emph{reward table}. Precisely, we generate a collection of latent independent random variables $(R_{\ell,i})_{\ell\in\mathbb{N}_1,i\in[k]}$ where each $R_{\ell,i}\sim P(\cdot \mid \theta_i)$. For any $(\ell,i)\in\mathbb{N}_1\times [k]$, we define the estimated mean reward 
$
\vartheta_{\ell,i} = \frac{\sum_{s = 1}^{\ell}R_{s,i}}{\ell}.
$

\begin{proof}[Proof of Proposition~\ref{prop:sufficient exploration implies convergence of mean estimations}]

Fix $\epsilon > 0$. We want to find ${T}_\epsilon\in\mathcal{L}^1$ such that for any $t\geq {T}_\epsilon$, $\max_{i\in[k]}|m_{t,i} - \theta_i|\leq\epsilon$.
By considering the latent reward table introduced above,
it suffices to show that there exists $L_{\epsilon}$ such that 
\begin{enumerate}
    \item for any $\ell \geq L_\epsilon$, $\max_{i\in[k]}|\vartheta_{\ell,i} - \theta_i|\leq\epsilon$, and
    \item $\E\left[{L}_{\epsilon}^{\frac{1}{\alpha}}\right] < \infty$,
\end{enumerate}
since taking ${T}_\epsilon = \max\left\{{T},\left\lceil\left(\frac{{L}_{\epsilon}}{\rho}\right)^{\frac{1}{\alpha}}\right\rceil\right\}\in\mathcal{L}^1$ completes the proof.

Define
\[
{L}_{\epsilon} \triangleq \inf\left\{s\in\mathbb{N}_1\,\,:\,\, \forall \ell\geq s,\,
\max_{i\in[k]}|\vartheta_{\ell,i} - \theta_i|\leq \epsilon\right\}.
\]
Writing the expected value as an integral gives
\[
\E\left[{L}_{\epsilon}^{\frac{1}{\alpha}}\right] = \int_0^\infty\Prob\left({L}_{\epsilon}^{\frac{1}{\alpha}}>x\right)\mathrm{d}x \leq \sum_{x\in\mathbb{N}_0} \Prob\left({L}_{\epsilon}^{\frac{1}{\alpha}}>x\right).
\]
For any $x\in \mathbb{N}_1$, we have
\begin{align*}
\Prob\left({L}_{\epsilon}^{\frac{1}{\alpha}} > x\right) 
\leq  \Prob\left({L}_{\epsilon} > \lfloor x^\alpha\rfloor\right)
&\leq \Prob\left(\exists \ell \geq \lfloor x^\alpha\rfloor,\, \max_{i\in[k]}|\vartheta_{\ell,i} - \theta_i| > \epsilon\right) \\
&\leq \sum_{\ell\geq \lfloor x^\alpha\rfloor} \sum_{i\in [k]} \Prob\left(|\vartheta_{\ell,i} - \theta_i| > \epsilon\right).
\end{align*}
By Chernoff inequality, for any $(\ell,i)\in\mathbb{N}_1\times[k]$,
\[
\Prob\left(|\vartheta_{\ell,i} - \theta_i|>\epsilon\right) 
= \Prob\left(\vartheta_{\ell,i} > \theta_i + \epsilon\right) + \Prob\left(\vartheta_{\ell,i} < \theta_i - \epsilon\right) \leq e^{-\ell\cdot\KL(\theta_i + \epsilon, \theta_i)} + e^{-\ell\cdot\KL(\theta_i - \epsilon, \theta_i)} \leq 2e^{-c_\epsilon\cdot \ell},
\]
where $c_\epsilon \triangleq \min_{i\in[k]} \min\left\{ \KL(\theta_i + \epsilon, \theta_i), \KL(\theta_i - \epsilon, \theta_i)\right\} > 0$.
Hence, for any $x\in \mathbb{N}_1$,
\begin{align*}
\Prob\left({L}_{\epsilon}^{\frac{1}{\alpha}} > x\right) 
\leq \sum_{\ell\geq \lfloor x^\alpha\rfloor} \sum_{i\in [k]} \Prob\left(|\vartheta_{\ell,i} - \theta_i| > \epsilon\right)
\leq \sum_{\ell\geq \lfloor x^\alpha\rfloor} \sum_{i\in [k]} 2e^{-c_\epsilon \cdot \ell} 
\leq \frac{2k}{c_\epsilon} e^{-c_\epsilon\cdot\left(\lfloor x^\alpha\rfloor-1\right)}.
\end{align*}
There exists large enough $\underline{x}\in\mathbb{N}_1$ such that for any $x\geq \underline{x}$, the RHS above is bounded by $\frac{1}{x^2}$. Therefore,
\[
\E\left[{L}_\epsilon^{\frac{1}{\alpha}}\right] \leq  \sum_{x\in\mathbb{N}_0} \Prob\left({L}_{\epsilon}^{\frac{1}{\alpha}}>x\right)
\leq \underline{x} +\sum_{x = \underline{x}}^\infty \frac{1}{x^2}
<\infty,
\]
which completes the proof.\footnote{Our proof readily extends to show $\E\left[{L}_{\epsilon}^{\frac{\beta}{\alpha}}\right] < \infty$ for any $\beta\geq 0$. Hence, if ${T}$ in the statement of Proposition~\ref{prop:sufficient exploration implies convergence of mean estimations} belongs to $\mathcal{L}^\beta$, then for any $t\geq {T}_\epsilon = \max\left\{{T},\left\lceil\left(\frac{{L}_{\epsilon}}{\rho}\right)^{\frac{1}{\alpha}}\right\rceil\right\}\in~\mathcal{L}^{\beta}$, $\max_{i\in[k]}|m_{t,i}-\theta_i|\leq \epsilon$.
}
\end{proof}

For completeness, we state Chernoff inequality below, which was used above in the proof of Proposition~\ref{prop:sufficient exploration implies convergence of mean estimations}.

\begin{lemma*}[Chernoff inequality {\citep[Theorem 15.9]{PolyanskiyWu}}]
\label{lem:chernoff}
Let $\ell\in\mathbb{N}_1$ and $X_1,X_2,\ldots, X_\ell$ be i.i.d. random variables drawn from one-dimensional exponential family with $\E[X_1]=\theta\in\mathbb{R}$. Then for any $\epsilon > 0$,
\[
\Prob\left(\frac{1}{\ell}\sum_{s=1}^\ell X_s \geq \theta + \epsilon \right)\leq \exp(-\ell\cdot \KL(\theta + \epsilon,\theta))
\]
and 
\[
\Prob\left(\frac{1}{\ell}\sum_{s=1}^\ell X_s \leq \theta-\epsilon \right)\leq \exp(-\ell\cdot \KL(\theta-\epsilon,\theta)).
\]
\end{lemma*}

\subsubsection{Completing the proof of Proposition~\ref{prop:implication of strong convergence of allocation and our stopping rule}}

The next result shows that when applied without a stopping rule, both convergence of estimated means and empirical proportions implies that
the cumulative within-experiment cost scales as
\[
\sum_{\ell=0}^{t-1} C_{I_\ell}(\thetabf) \sim \frac{t\cdot\min_{j\neq \hat{I}_t}D_{\bm{m}_t,\hat{I}_t, j}(p_{t,\hat{I}_t},p_{t,j})}{\min_{\varthetabf\in \overline{\rm Alt}(\thetabf)}\Gamma_{\thetabf}(\widetilde{\bm{p}},\varthetabf)}
\quad\text{as}\quad 
t\to\infty,
\]
where the numerator on the RHS is the stopping statistic in Algorithm~\ref{alg:general-template-exp-family}.
This result is formally stated and proved below.

\begin{lemma}
\label{lem:implication of strong convergence of allocation}
If an allocation rule satisfying $\bm{p}_t\Sto \widetilde{\bm{p}} > 0$,  
\[
\frac{t\cdot\min_{j\neq \hat{I}_t}D_{\bm{m}_t,\hat{I}_t, j}(p_{t,\hat{I}_t},p_{t,j})}{\sum_{\ell=0}^{t-1} C_{I_\ell}(\thetabf)} \Sto 
\min_{\varthetabf\in \overline{\rm Alt}(\thetabf)}\Gamma_{\thetabf}(\widetilde{\bm{p}},\varthetabf) > 0.
\]
\end{lemma}

\begin{proof}[Proof of Lemma~\ref{lem:implication of strong convergence of allocation}]

We can write
\[
\frac{t\cdot\min_{j\neq \hat{I}_t}D_{\bm{m}_t,\hat{I}_t, j}(p_{t,\hat{I}_t},p_{t,j})}{\sum_{\ell=0}^{t-1} C_{I_\ell}(\thetabf)} 
= \frac{\min_{j\neq \hat{I}_t}D_{\bm{m}_t,\hat{I}_t, j}(p_{t,\hat{I}_t},p_{t,j})}{\frac{1}{t}\sum_{i\in[k]} N_{t,i}C_i(\thetabf)} 
= \frac{\min_{j\neq \hat{I}_t}D_{\bm{m}_t,\hat{I}_t, j}(p_{t,\hat{I}_t},p_{t,j})}{\sum_{i\in[k]} p_{t,i}C_i(\thetabf)}.
\]
The condition $\bm{p}_t\Sto \widetilde{\bm{p}} > \bm{0}$ in Proposition~\ref{prop:implication of strong convergence of allocation and our stopping rule} implies the sufficient exploration condition in Proposition~\ref{prop:sufficient exploration implies convergence of mean estimations}, and thus we have $\bm{m}_t \Sto \thetabf$. 
Given this,
there exists a random time  ${T} \in \mathcal{L}^1$ such that for $t\geq {T}$, the empirical best arm $\hat{I}_t$ is the unique best arm  $I^*$, and thus $\hat{I}_t$ in the RHS can be replaced by $I^*$. 
Then by the continuity of $D_{\thetabf,I^*,j}$ and the continuous mapping theorem for strong convergence $\Sto $  (Definition \ref{def: strong convergence}), we have
\[
\frac{\min_{j\neq I^*}D_{\bm{m}_t,I^*, j}(p_{t,I^*},p_{t,j})}{\sum_{i\in[k]} p_{t,i}C_i(\thetabf)}
\Sto 
\frac{\min_{j\neq I^*}D_{\thetabf,I^*, j}(\widetilde{{p}}_{I^*},\widetilde{{p}}_{j})}{\sum_{i\in[k]} \widetilde{p}_{i}C_i(\thetabf)}
= \min_{\varthetabf\in \overline{\rm Alt}(\thetabf)}\Gamma_{\thetabf}(\widetilde{\bm{p}},\varthetabf),
\]
where the equality uses Lemma \ref{lem:Lemma3 in garivier2016optimal}. 
Since $\widetilde{\bm{p}} > \bm{0}$, $\min_{j\neq I^*}D_{\thetabf,I^*, j}(\widetilde{{p}}_{I^*},\widetilde{{p}}_{j}) > 0$, and thus $\min_{\varthetabf\in \overline{\rm Alt}(\thetabf)}\Gamma_{\thetabf}(\widetilde{\bm{p}},\varthetabf) >~0$.
This completes the proof.
\end{proof}

Now we are ready to complete the proof of Proposition \ref{prop:implication of strong convergence of allocation and our stopping rule}.

\begin{proof}[Proof of Proposition \ref{prop:implication of strong convergence of allocation and our stopping rule}]
Fix $\epsilon\in \left(0, \frac{\min_{\varthetabf\in \overline{\rm Alt}(\thetabf)}\Gamma_{\thetabf}(\widetilde{\bm{p}},\varthetabf)}{2}\right)$. We first upper bound the threshold $\gamma_t(n)$ in \eqref{eq:tight threshold T} as follows,
\[
\gamma_t(n) \leq \left[\ln(n) + 6\ln(\ln(n))\right] + \left[6\ln(\ln(t)) + W_k\right],
\] 
where $W_k$ is some term only dependent of $k$ (and independent of the population size $n$ and time $t$). 
Since the minimal per-period within-experiment cost $\min_{i\in [k]}C_i(\thetabf) > 0$ (Assumption \ref{asm:post-experimentation cost}), there exists a deterministic time $\ut_\epsilon$ (independent of the population size $n$) such that for any $t\geq \ut_\epsilon$,  the term on the RHS above,
$
6\ln(\ln(t)) + W_k \leq \epsilon \cdot C_{\min}(\thetabf)\cdot t,
$ 
and thus
\[
\gamma_t(n) \leq \left[\ln(n) + 6\ln(\ln(n))\right] + \epsilon \cdot C_{\min}(\thetabf)\cdot t
\leq \left[\ln(n) + 6\ln(\ln(n))\right] + \epsilon \cdot\sum_{\ell = 0}^{t-1} C_{I_{\ell}}(\thetabf).
\]

On the other hand, 
by Lemma \ref{lem:implication of strong convergence of allocation}, there exists a random time ${T}_{\epsilon,0} \in \mathcal{L}^1$ (independent of the population size $n$) such that for any $t\geq {T}_{\epsilon,0}$, the stopping statistic in Algorithm \ref{alg:general-template-exp-family} can be lower bounded as follows,
\[
t\cdot\min_{j\neq \hat{I}_t}D_{\bm{m}_t,\hat{I}_t, j}(p_{t,\hat{I}_t},p_{t,j}) \geq  \left[\min_{\varthetabf\in \overline{\rm Alt}(\thetabf)}\Gamma_{\thetabf}(\widetilde{\bm{p}},\varthetabf) - \epsilon\right] \cdot \sum_{\ell=0}^{t-1}C_{I_\ell}(\thetabf).
\]

For notational convenience, we write the (path-dependent) cumulative within-experiment cost as $F_t \triangleq \sum_{\ell=0}^{t-1}C_{I_\ell}(\thetabf)$.
Therefore, for any $t\geq {T}_\epsilon \triangleq \max\left\{\ut_{\epsilon}, {T}_{\epsilon,0}\right\} \in \mathcal{L}^1$ (independent of the population size $n$),
\begin{align*}
t\cdot\min_{j\neq \hat{I}_t}D_{\bm{m}_t,\hat{I}_t, j}(p_{t,\hat{I}_t},p_{t,j}) - \gamma_t(n) 
\geq& \left[\min_{\varthetabf\in \overline{\rm Alt}(\thetabf)}\Gamma_{\thetabf}(\widetilde{\bm{p}},\varthetabf) - 2\epsilon\right] F_t - \left[\ln(n) + 6\ln(\ln(n))\right] \triangleq G_t.
\end{align*}
By the definition of the stopping rule, we know the experiment has stopped by the first time $G_t$ is non-negative, i.e. 
\begin{equation}
\label{eq:stopping time bound}
\tau \leq \min\{t \geq T_{\epsilon} \,:\, G_t \geq 0 \} 
= \min\left\{t \geq T_{\epsilon} \,:\,  F_t \geq \frac{\ln(n) + 6\ln(\ln(n)) }{\min_{\varthetabf\in \overline{\rm Alt}(\thetabf)}\Gamma_{\thetabf}(\widetilde{\bm{p}},\varthetabf) - 2\epsilon}\right\}.
\end{equation}

Now we fix the population size $n$ satisfying the condition \eqref{eq:large enough population size}, and upper bound the cumulative within-experiment cost upon the stopping time $\tau$, denoted by $F_\tau$, 
under the following cases, respectively.
\begin{enumerate}
    \item $G_{{T}_{\epsilon}} \geq 0$
    
    For this case, the experiment must have been stopped upon period ${T}_\epsilon$ (i.e., $\tau \leq {T}_\epsilon$), and the cumulative within-experiment cost is bounded as 
    \[
    F_\tau \leq C_{\max}(\thetabf) \cdot {T}_\epsilon.
    \]

    \item $G_{{T}_{\epsilon}} < 0$.

    For this case, we do not know whether the experiment has stopped by period ${T}_\epsilon$, but the condition~\eqref{eq:large enough population size} implies $G_{n-1} \geq 0$. By the definition of the stopping rule, the experiment must have been stopped by period $n-1$ (i.e., $\tau \leq n-1$). 
    Observe that $G_t$ is strictly increasing in $t$ (since $F_t$ is strictly increasing in $t$), 
    so the cumulative within-experiment cost 
    \[
    F_\tau \leq F_{\tau-1} + C_{\max}(\thetabf) 
    \leq \frac{\ln(n) + 6\ln(\ln(n)) }{\min_{\varthetabf\in \overline{\rm Alt}(\thetabf)}\Gamma_{\thetabf}(\widetilde{\bm{p}},\varthetabf) - 2\epsilon} + C_{\max}(\thetabf),
    \]
    where the first inequality holds since $C_{\max}(\thetabf)$ is the maximal per-period within-experiment cost, and the second inequality follows from~\eqref{eq:stopping time bound}.
\end{enumerate}

Combining the bounds on $F_\tau$ in the above cases gives
\begin{align*}
F_\tau
&\leq \max\left\{ C_{\max}(\thetabf)\cdot {T}_\epsilon, \, \frac{\ln(n) + 6\ln(\ln(n))}{\min_{\varthetabf\in \overline{\rm Alt}(\thetabf)}\Gamma_{\thetabf}(\widetilde{\bm{p}},\varthetabf)-2\epsilon} + C_{\max}(\thetabf)\right\} \\
&\leq \frac{\ln(n) + 6\ln(\ln(n))}{\min_{\varthetabf\in \overline{\rm Alt}(\thetabf)}\Gamma_{\thetabf}(\widetilde{\bm{p}},\varthetabf) - 2\epsilon} +   C_{\max}(\thetabf) \cdot \left({T}_\epsilon + 1\right).
\end{align*}
This completes the proof.
\end{proof}

\subsection{Completing the proof of Theorem~\ref{thm:efficient-p-general restatement}}

With the respective upper bounds on cumulative within-experiment and post-experiment costs in Propositions~\ref{prop:implication of strong convergence of allocation and our stopping rule} and \ref{prop:bound on post-experimentation cost}, we are ready to complete the proof of Theorem~\ref{thm:efficient-p-general restatement}.

\begin{proof}[Proof of Theorem~\ref{thm:efficient-p-general restatement}]
Denote Algorithm~\ref{alg:general-template-exp-family} with an allocation rule satisfying $\bm{p}_t\Sto \widetilde{\bm{p}} > \bm{0}$ by $\pi$.

Fix $\epsilon\in \left(0, \frac{\min_{\varthetabf\in \overline{\rm Alt}(\thetabf)}\Gamma_{\thetabf}(\widetilde{\bm{p}},\varthetabf)}{2}\right)$. Pick ${T}_\epsilon\in \mathcal{L}^1$ in Proposition~\ref{prop:implication of strong convergence of allocation and our stopping rule}, which is independent of the population size $n$.
Also observe that the condition for the population size $n$ in~\eqref{eq:large enough population size} is path-independent. Then for any large enough $n$ such that the condition~\eqref{eq:large enough population size} holds, applying Propositions~\ref{prop:implication of strong convergence of allocation and our stopping rule} and \ref{prop:bound on post-experimentation cost} yields
\begin{align*}
\mathrm{Cost}_{\bm{\theta}}(n, \pi) 
&= \E_{\bm{\theta}}\left[ \sum_{t=0}^{\tau-1}C_{I_t}(\thetabf) + (n-\tau)\Delta_{\hat{I}_\tau}(\thetabf) \right]\\
&\leq 
\frac{\ln(n) + 6\ln(\ln(n))}{\min_{\varthetabf\in \overline{\rm Alt}(\thetabf)}\Gamma_{\thetabf}(\widetilde{\bm{p}},\varthetabf) - 2\epsilon} + C_{\max}(\thetabf)\cdot \left(\E_{\thetabf}\left[{T}_\epsilon\right] + 1\right) + \Delta_{\max}(\thetabf).
\end{align*}

Since ${T}_\epsilon\in \mathcal{L}^1$, it has finite expectation $\E_{\thetabf}\left[{T}_\epsilon\right]<\infty$. Moreover, $\E_{\thetabf}\left[{T}_\epsilon\right]$ is independent of the population size $n$ since $T_{\epsilon}$ does not depend on $n$. Therefore, we have
\[
\limsup_{n\to\infty}\frac{\mathrm{Cost}_{\bm{\theta}}(n, \pi)}{\ln(n)} \leq \frac{1}{\min_{\varthetabf\in \overline{\rm Alt}(\thetabf)}\Gamma_{\thetabf}(\widetilde{\bm{p}},\varthetabf) - 2\epsilon}.
\]
The above inequality holds for arbitrary small $\epsilon > 0$, so 
\[
\limsup_{n\to\infty}\frac{\mathrm{Cost}_{\bm{\theta}}(n, \pi)}{\ln(n)} \leq \frac{1}{\min_{\varthetabf\in \overline{\rm Alt}(\thetabf)}\Gamma_{\thetabf}(\widetilde{\bm{p}},\varthetabf)},
\]
which completes the proof.
\end{proof}

\subsection{Proof of upper bound on calibration function $\mathcal{C}_{\mathrm{exp}}$ in Proposition \ref{prop:C_exp upper bound}}
\label{subsec:C_exp upper bound proof}

Here we complete the proof of upper bound on calibration function $\mathcal{C}_{\mathrm{exp}}$ in Proposition \ref{prop:C_exp upper bound}, which is established in a sequence of results that bound the supporting functions $\tilde{h}(y)$ and $h^{-1}(y)$.
We first show the following lower and upper bounds of $h^{-1}(y)$, where the upper bound is derived based on \citet[Proposition 8]{Kaufmann2021martingale}.
\begin{lemma}
\label{lem:bounds on inverse h}
For $y\geq 1$,
\[
y + \ln(y)\leq h^{-1}(y) \leq y + \ln(2y).
\]
\end{lemma}
\begin{proof}
We have
\[
h(y + \ln(y)) =  (y +\ln(y)) - \ln(y+\ln(y))\leq y.
\] 
Since $h$ is an increasing function, this implies
\[
y + \ln(y) \leq h^{-1}(y).
\]

The upper bound is derived based on \citet[Proposition 8]{Kaufmann2021martingale}, which gives
\[
h^{-1}(y) \leq y + \ln\left(y + \sqrt{2(y-1)}\right) \leq y+\ln(2y),
\]
where the last inequality follows from $2(y-1)\leq y^2$.
\end{proof}

Next we derive an upper bound on $\tilde{h}(y)$ if the argument $y\ge h\left(\frac{1}{\ln(3/2)}\right)$. Recall the definition~\eqref{eq:tilde h}:
\[
\tilde{h}(y) = \exp\left(\frac{1}{h^{-1}(y)}\right)\cdot h^{-1}(y), \quad \hbox{if } y\ge h\left(\frac{1}{\ln(3/2)}\right).
\]

\begin{lemma}
\label{lem:bounds on tilde h}
For $y\ge h\left(\frac{1}{\ln(3/2)}\right)$,
\[
\tilde{h}(y) \leq h^{-1}(y) + 2(e^{0.5} - 1).
\]
\end{lemma}
\begin{proof}
Applying the lower bound of $h^{-1}$ in Lemma \ref{lem:bounds on inverse h} gives
\[
\frac{1}{h^{-1}(y)}\leq \frac{1}{y+ \ln(y)} \leq \frac{1}{2},
\]
where the last inequality holds when $y\ge h\left(\frac{1}{\ln(3/2)}\right)$,
and then
\[
\tilde{h}(y) = \exp\left(\frac{1}{h^{-1}(y)}\right)\cdot h^{-1}(y) \leq \left[1+ \frac{2(e^{0.5} - 1)}{h^{-1}(y)}\right]\cdot h^{-1}(y) = h^{-1}(y) + 2(e^{0.5} - 1),
\]
where the inequality applies $e^x\leq 1+ 2(e^{0.5} - 1)x$ for $x\in \left[0,\frac{1}{2}\right]$, which holds since $e^x$ is convex.
\end{proof}

Now we are ready to complete the proof of Proposition \ref{prop:C_exp upper bound}.
\begin{proof}[Proof of Proposition \ref{prop:C_exp upper bound}]
Recall the definition \eqref{eq:C_exp}:
\[
\mathcal{C}_{\mathrm{exp}}(x) = 2 \tilde{h}\left(y_x\right) 
\quad\text{where}\quad
y_x = \frac{h^{-1}(x+1) + \ln\left(\frac{\pi^2}{3}\right)}{2}.
\]
To apply the upper bound of $\tilde{h}$ in Lemma \ref{lem:bounds on tilde h}, we lower bound the argument $y_x$ as follows,
\[
y_x = \frac{h^{-1}(x+1) + \ln\left(\frac{\pi^2}{3}\right)}{2}\geq \frac{(x + 1) + \ln(x+1) + \ln\left(\frac{\pi^2}{3}\right)}{2} \geq h\left(\frac{1}{\ln(3/2)}\right),
\]
where the first inequality uses the lower bound of $h^{-1}$ in Lemma \ref{lem:bounds on inverse h}, and the last inequality holds for $x\geq 0.52$.

Then by Lemma \ref{lem:bounds on tilde h},
\begin{align*}
\mathcal{C}_{\mathrm{exp}}(x) = 2 \tilde{h}\left(y_x\right) 
&\leq 2[h^{-1}\left(y_x\right) + 2(e^{0.5}-1)] \\
&\leq 2y_x + 2\ln(2y_x) + 4(e^{0.5}-1),
\end{align*}
where the last inequality follows from the upper bound of $h^{-1}$ in Lemma \ref{lem:bounds on inverse h}.
To further upper bound $\mathcal{C}_{\mathrm{exp}}(x)$, we upper bound $y_x$ as follows,
\[
y_x \leq \frac{(x + 1) + \ln(2(x+1)) + \ln\left(\frac{\pi^2}{3}\right)}{2} = \frac{(x + 1) + \ln(x+1) + \ln\left(\frac{2\pi^2}{3}\right)}{2},
\]
where the inequality applies the upper bound of $h^{-1}$ in Lemma \ref{lem:bounds on inverse h},
and thus
\begin{align*}
\mathcal{C}_{\mathrm{exp}}(x) 
&\leq 2y_x + 2\ln(2y_x) + 4(e^{0.5}-1) \\
&\leq (x + 1) + \ln(x+1) + \ln\left(\frac{2\pi^2}{3}\right) + 2\ln\left((x + 1)+\ln(x+1) + \ln\left(\frac{2\pi^2}{3}\right)\right) + 4(e^{0.5}-1)\\
&\leq (x + 1) + \ln(x+1) + \ln\left(\frac{2\pi^2}{3}\right) + 2\ln(2(x+2)) + 4(e^{0.5}-1) \\
&\leq x + 3\ln(x+2) + 7,
\end{align*}
where the second-to-last inequality follows from
\[
\ln\left((x + 1)+\ln(x+1) + \ln\left(\frac{2\pi^2}{3}\right)\right) \leq \ln((x + 1)+ (x+1) + 2) = \ln(2(x+2)).
\]
This completes the proof.
\end{proof}

%% file: app_iff_LaiRobbins.tex
\section{Construction of a universally efficient policy}
\label{app:tracking}

In this appendix, we construct a universally efficient policy. Recall that Proposition~\ref{prop:sufficient conditions of universal efficiency} shows that a policy that satisfies~\eqref{eq:alternative sufficient (and necessary) condition of universal efficiency} is universally efficient. 
Additionally, Theorem~\ref{thm:efficient-p-general restatement} proves that \eqref{eq:alternative sufficient (and necessary) condition of universal efficiency} is guaranteed under Algorithm~\ref{alg:general-template-exp-family} that takes an allocation rule such that $\forall \thetabf\in\Theta,\bm{p}_t\Sto \bm{p}^*$. We are going to construct allocation rules with this property.

\paragraph{Tracking equilibrium strategy.} 
It is possible to construct such allocation rules by directly mimicking the equilibrium strategy of the experimenter $\bm{p}^*$,
for instance, variations of the tracking rule \citep{garivier2016optimal}. We study one instantiation named \emph{Direct-tracking} (Algorithm~\ref{alg:D-Tracking}). At each time $t$, it has a forced exploration step to ensure that MLE estimate $\bm{m}_t$ converges to $\thetabf$. If no arm is under-explored, then it executes another step of tracking the "target proportion vector" by playing the arm whose empirical proportion is most far away from its "target proportion". 
If $\bm{m}_t$ has a unique largest entry, the "target proportion vector" $\hat{\bm{p}}_t$ is the unique equilibrium strategy of the experimenter under $\bm{m}_t$:\footnote{Proposition~\ref{prop:uniqueness of the experimenter's equilibrium strategy} guarantees the uniqueness of such probability vector.}
\begin{equation}
\label{eq:calculate target proportion vector}
\hat{\bm{p}}_t = \left(\hat{p}_{t,1},\ldots,\hat{p}_{t,k}\right) \triangleq 
\bm{p}^*(\bm{m}_t) = 
\argmax_{\bm{p}\in\Sigma_k}\min_{\varthetabf\in \overline{\rm Alt}(\bm{m}_t)}\Gamma_{\bm{m}_t}(\bm{p},\bm{\vartheta}),
\end{equation}
Otherwise, set $\hat{\bm{p}}_t = \hat{\bm{p}}_{t-1}$. Note that the initial "target proportion vector" $\bm{p}_0 = \left(\frac{1}{k},\ldots,\frac{1}{k}\right)$.

\begin{algorithm}[H]
	\centering
	\caption{Direct-tracking allocation rule}\label{alg:D-Tracking}
	\begin{algorithmic}[1]
		\State {\bf Input:} $\mathcal{H}_0 \gets \{ \}, \hat{\bm{p}}_0 = \left(\frac{1}{k},\ldots,\frac{1}{k}\right), \bm{m}_0 = \bm{0}$ 
		\For{$t=0,1,\ldots$}
		      \State{Obtain the set of under-explored arms: $U_t \triangleq \left\{i\in[k]\,:\, N_{t,i} \leq \max\left\{\sqrt{t} - \frac{k}{2},0\right\}\right\}$}
                \If{$U_t$ is non-empty}
                    \State Play arm $I_t \in \argmin_{i\in U_t} N_{t,i}$
                \Else
                    \If{$\bm{m}_t\in\Theta$}
                        \State Calculate $\hat{\bm{p}}_t$ in \eqref{eq:calculate target proportion vector} 
                    \Else
                        \State Set $\hat{\bm{p}}_t = \hat{\bm{p}}_{t-1}$
                    \EndIf
                        \State Play arm $I_t\in \argmax_{i\in[k]} \left( t\cdot\hat{p}_{t,i} - N_{t,i}\right)$
                \EndIf
		\State Observe $Y_{t,I_t}$ and update history $\mathcal{H}_{t+1} \gets \mathcal{H}_t\cup \{(I_t, Y_{t,I_t}) \}$.
		\EndFor	
	\end{algorithmic}
\end{algorithm}

In a nutshell, direct-tracking allocation rule directly mimics the experimenter's unique equilibrium strategy $\bm{p}^*$ using plug-in mean estimators $\bm{m}_t$. The next result shows that it indeed has limiting allocation~$\bm{p}^*$.

\begin{proposition}[Optimality of direct-tracking]
\label{prop:D-tracking}
Direct-tracking allocation rule (Algorithm~\ref{alg:D-Tracking}) satisfies that 
$
\forall\thetabf\in\Theta,\bm{p}_t\Sto \bm{p}^*.
$
\end{proposition}

\begin{proof}[Proof of Proposition~\ref{prop:D-tracking}]
Fix $\thetabf\in\Theta$. Thanks to the forced exploration step, direct-tracking allocation rule ensures that for any $(t,i)\in\mathbb{N}_0\times [k]$, $N_{t,i}\geq \max\left\{\sqrt{t} - \frac{k}{2},0\right\}-1$, as proved in \citet[Lemma~8 part~1]{garivier2016optimal}. This result will be formally restated after this proof.
Then we have $\bm{m}_t\Sto \thetabf \in \Theta$ by Proposition~\ref{prop:sufficient exploration implies convergence of mean estimations}, and therefore $\bm{p}^*(\bm{m}_t)\Sto \bm{p}^*(\thetabf)$ because of the continuity of within-experiment costs (Assumption~\ref{asm:sampling cost}). This implies that for any $\epsilon > 0$, there exists a (potentially random) time $T_{\epsilon,0}\in\mathcal{L}^1$ such that for any $t\geq T_{\epsilon,0}$, $\max_{i\in[k]} |\hat{p}_{t,i} - p^*_i| \leq \epsilon$.
Then by \citet[Lemma~8 part~2]{garivier2016optimal} (which will be presented after this proof), there exists $T_{\epsilon,1}\in\mathcal{L}^1$ such that for any $t\geq T_{\epsilon,1}$, $\max_{i\in[k]} |p_{t,i} - p^*_i| \leq 3(k-1)\epsilon$. This completes the proof of Proposition~\ref{prop:D-tracking}.
\end{proof}

For completeness, we formally present \citet[Lemma 8]{garivier2016optimal}, which was used above in the proof of Proposition~\ref{prop:D-tracking}.

\begin{lemma*}[{\citet[Lemma 8]{garivier2016optimal}}, implied by Lemma 17 therein]
\label{lem:tracking}
For any $\thetabf\in\Theta$, direct-tracking allocation rule ensures that 
\begin{enumerate}
    \item for any $(t,i)\in\mathbb{N}_0\times [k]$,  $N_{t,i}\geq \max\left\{\sqrt{t} - \frac{k}{2},0\right\}-1$, and
    \item for any $\epsilon > 0$, for any $t_{\epsilon,0}$, there exists $t_{\epsilon,1} \geq t_{\epsilon,0}$ such that
    \[
    \sup_{t\geq t_{\epsilon,0}} \max_{i\in[k]} |\hat{p}_{t,i} - p^*_i| \leq \epsilon
    \quad\implies\quad 
    \sup_{t\geq t_{\epsilon,1}} \max_{i\in[k]} |p_{t,i} - p^*_i| \leq 3(k-1)\epsilon,
    \]
    where $t_{\epsilon,1}$ takes the form $a_\epsilon t_{\epsilon,0} + b_{\epsilon}$ for some constants $a_\epsilon$ and $b_{\epsilon}$ that are independent of both the value of $t_{\epsilon,0}$ and the sample path.\footnote{The form of $t_{\epsilon,1}$ is implied by the proof of \citet[Lemma 17]{garivier2016optimal}.}
\end{enumerate}
\end{lemma*}

\section{Proof of Theorem~\ref{thm:Lai-Robbins-type formula_general}}
\label{app:tracking satisfies sufficient condition}

This appendix completes the proof of Theorem~\ref{thm:Lai-Robbins-type formula_general}, restated as follows.
\LaiRobbins*

The ``if" statement in Theorem~\ref{thm:Lai-Robbins-type formula_general} immediately follows from Proposition~\ref{prop:sufficient conditions of universal efficiency}, 
since the conditions~\eqref{eq:sufficient and necessary condition of universal efficiency} and~\eqref{eq:alternative sufficient (and necessary) condition of universal efficiency} are equivalent given the equilibrium value's formula in Theorem~\ref{thm:equilibrium} and the lower bound in Proposition~\ref{prop:lower bound}.

The rest of this appendix is going to prove that any universally efficient policy must satisfy~\eqref{eq:alternative sufficient (and necessary) condition of universal efficiency}. 
The following result helps simplify the proof to focus on constructing of one policy that satisfies~\eqref{eq:alternative sufficient (and necessary) condition of universal efficiency}. 

\begin{proposition}
\label{prop:reduction to constructing an allocation rule}
If there exists a policy satisfying~\eqref{eq:alternative sufficient (and necessary) condition of universal efficiency},
then any universally efficient policy must satisfy~\eqref{eq:alternative sufficient (and necessary) condition of universal efficiency}.
\end{proposition}

\begin{proof} 
Let $\varpi^*$ be a policy that satisfies~\eqref{eq:alternative sufficient (and necessary) condition of universal efficiency}, which implies that it is consistent.
For any universally efficient policy $\pi^*$, the condition~\eqref{eq:universally efficient rule} in Definition~\ref{def:universal-efficiency} gives
\begin{align*}
\forall\thetabf\in\Theta:\quad
\limsup_{n\to \infty}  \frac{\mathrm{Cost}_{\bm{\theta}}(n, \pi^*)}{\mathrm{Cost}_{\bm{\theta}}(n, \varpi^*)} \leq 1,
\end{align*}
and therefore
\begin{align*}
\limsup_{n\to \infty} \frac{\mathrm{Cost}_{\bm{\theta}}(n, \pi^*)}{\ln(n)}
=& \limsup_{n\to \infty} \frac{\mathrm{Cost}_{\bm{\theta}}(n, \pi^*)}{\mathrm{Cost}_{\bm{\theta}}(n, \varpi^*)}\frac{\mathrm{Cost}_{\bm{\theta}}(n, \varpi^*)}{\ln(n)} \\
\leq& \limsup_{n\to \infty}\frac{\mathrm{Cost}_{\bm{\theta}}(n, \pi^*)}{\mathrm{Cost}_{\bm{\theta}}(n, \varpi^*)}\limsup_{n\to\infty}\frac{\mathrm{Cost}_{\bm{\theta}}(n, \varpi^*)}{\ln(n)} \leq \frac{1}{\Gamma_{\thetabf}(\bm{p}^*,\bm{q}^*)},
\end{align*}
which shows that $\pi^*$ satisfies~\eqref{eq:alternative sufficient (and necessary) condition of universal efficiency}.
The first inequality follows from the fact that for two nonnegative sequences 
$\{a_\ell\}_{\ell\in\mathbb{N}_1}$ and $\{b_\ell\}_{\ell\in\mathbb{N}_1}$, $\limsup_{\ell\to\infty}a_\ell b_\ell\leq \limsup_{\ell\to\infty}a_\ell\limsup_{\ell\to\infty}b_\ell$, and the last inequality holds since $\varpi^*$ satisfies~\eqref{eq:alternative sufficient (and necessary) condition of universal efficiency}. 
\end{proof}

%% file: app_onlyif_allocation.tex
\section{Proof of the second statement in Theorem~\ref{thm:efficient-p-general}}
\label{app:proof of sufficient condition being almost necessary} 
Appendix~\ref{app:proof of sufficient condition} has proved the first statement in Theorem~\ref{thm:efficient-p-general}. This appendix completes the proof of the second statement, which is restated as follows. 

\begin{proposition}[The second statement in Theorem~\ref{thm:efficient-p-general}]
\label{prop:efficient-p-general almost necessary}
Algorithm~\ref{alg:general-template-exp-family} is not universally efficient if there exists $\thetabf\in\Theta$ such that the empirical allocation ${\bm p}_t$ under the input allocation rule converges strongly to a probability vector other than ${\bm p}^*$.
\end{proposition}

We show that Proposition~\ref{prop:efficient-p-general almost necessary} immediately follows from the next result:

\begin{proposition}
\label{prop:lower_bound_for_consistent_policy}
Fix $\thetabf\in\Theta$ and let $\widetilde{\pi}$ implement Algorithm~\ref{alg:general-template-exp-family} with an (anytime) allocation rule under which the empirical allocation $\bm{p}_t$ converges strongly to a probability vector $\widetilde{\bm{p}}\in\Sigma_k$. 
If $\widetilde{\pi}$ is consistent, then
\begin{align*}
\liminf_{n\to\infty} \frac{\E_{\thetabf} \left[ \sum_{t=0}^{\tau-1}C_{I_t}(\thetabf)\right]}{\ln(n)} \geq \frac{1}{\min_{\varthetabf\in \overline{\rm Alt}(\thetabf)}\Gamma_{\thetabf}(\widetilde{\bm{p}},\varthetabf)}.
\end{align*}
\end{proposition}

\begin{proof}[Proof of Proposition~\ref{prop:efficient-p-general almost necessary} (based on Proposition~\ref{prop:lower_bound_for_consistent_policy})]

Denote this policy by $\widetilde{\pi}$. 
If $\widetilde{\pi}$ is not consistent, then it is not universally efficient, and we are done.

Now we consider $\widetilde{\pi}$ to be consistent.
Denote the probability vector other than $\bm{p}^*$ by $\widetilde{\bm{p}}\neq \bm{p}^*$. By Proposition~\ref{prop:lower_bound_for_consistent_policy},
\begin{align*}
\liminf_{n\to\infty} \frac{\E_{\thetabf} \left[ \sum_{t=0}^{\tau-1}C_{I_t}(\thetabf)\right]}{\ln(n)} \geq \frac{1}{\min_{\varthetabf\in \overline{\rm Alt}(\thetabf)}\Gamma_{\thetabf}(\widetilde{\bm{p}},\varthetabf)}
>\frac{1}{\min_{\varthetabf\in \overline{\rm Alt}(\thetabf)}\Gamma_{\thetabf}({\bm{p}}^*,\varthetabf)} = \frac{1}{\Gamma_{\thetabf}(\bm{p}^*,\bm{q}^*)},
\end{align*}
where the strict inequality follows from $\bm{p}^*$ being the unique maximizer shown in Proposition~\ref{prop:uniqueness of the experimenter's equilibrium strategy} and $\widetilde{\bm{p}}\neq \bm{p}^*$. Hence,
\[
\liminf_{n\to \infty} \frac{\mathrm{Cost}_{\bm{\theta}}(n, \widetilde{\pi}) }{\ln(n)} \geq \liminf_{n\to\infty} \frac{\E_{\thetabf} \left[ \sum_{t=0}^{\tau-1}C_{I_t}(\thetabf)\right]}{\ln(n)} > \frac{1}{\Gamma_{\thetabf}(\bm{p}^*,\bm{q}^*)},
\]
which implies that $\widetilde{\pi}$ is not universally efficient by Proposition \ref{prop:reduction to constructing an allocation rule}.
\end{proof}

The remaining of this appendix formally proves Proposition~\ref{prop:lower_bound_for_consistent_policy}.

\subsection{Proof of Proposition~\ref{prop:lower_bound_for_consistent_policy}}
We apply the same proof strategy used for showing the lower bound in Proposition~\ref{prop:garivier_lower}.

\begin{proof}[Proof of Proposition~\ref{prop:lower_bound_for_consistent_policy}]
Lemma~\ref{lem:consistency leads to pcs} indicates that a consistent policy satisfies~\eqref{eq:pcs_constraint}, and we are going to apply Corollary~\ref{cor:kl-growth-under-pcs and kl_formula}.
Let
$
\widetilde{\varthetabf} \in \argmin_{\varthetabf\in \overline{\rm Alt}(\thetabf)}\Gamma_{\thetabf}(\widetilde{\bm{p}},\varthetabf)
$
be a best response to the allocation $\widetilde{\bm{p}}$.
Note that $\widetilde{\varthetabf}$ may not belong to $\Theta$ (the set of instances with a unique best arm), so we first consider arbitrary instance $\varthetabf = (\vartheta_1,\ldots,\vartheta_k)\in{\rm Alt}(\thetabf)$ in order to use Corollary~\ref{cor:kl-growth-under-pcs and kl_formula}, and will take a limit as $\varthetabf\to\widetilde{\varthetabf}$ in the very end.

Applying Corollary \ref{cor:kl-growth-under-pcs and kl_formula} yields
\begin{align*}
\left[1+o_{\thetabf,\varthetabf}(1)\right]\ln(n) &\leq \sum_{i\in[k]}\E_{\thetabf}[N_{\tau,i}] {\rm KL}(\theta_i,\vartheta_i) \\
&= \E_{\thetabf} \left[ \sum_{t=0}^{\tau-1}C_{I_t}(\thetabf)\right] 
\frac{\sum_{i\in[k]}\E_{\thetabf}[N_{\tau,i}]{\rm KL}(\theta_i,\vartheta_i)}{\sum_{i\in[k]}\E_{\thetabf}[N_{\tau,i}]C_i(\thetabf)} \\
&= \E_{\thetabf} \left[ \sum_{t=0}^{\tau-1}C_{I_t}(\thetabf)\right] 
\frac{\sum_{i\in[k]}\frac{\E_{\thetabf}[N_{\tau,i}]}{\E_{\thetabf}[\tau]}{\rm KL}(\theta_i,\vartheta_i)}{\sum_{i\in[k]}\frac{\E_{\thetabf}[N_{\tau,i}]}{\E_{\thetabf}[\tau]}C_i(\thetabf)} 
\end{align*}
where the first equality holds since $\E_{\thetabf} \left[ \sum_{t=0}^{\tau-1}C_{I_t}(\thetabf)\right] = \sum_{i\in[k]}\E_{\thetabf}[N_{\tau,i}]C_i(\thetabf)$.
Dividing each side by $\ln(n)$ and taking $n\to\infty$ gives
\begin{align*}
\liminf_{n\to\infty} \frac{\E_{\thetabf} \left[ \sum_{t=0}^{\tau-1}C_{I_t}(\thetabf)\right]}{\ln(n)} \geq \frac{1}{\Gamma_{\thetabf}(\widetilde{\bm{p}},\varthetabf)},
\end{align*}
where the inequality applies Lemma~\ref{lem:expected-by-expected} and the definition of $\Gamma_{\thetabf}(\widetilde{\bm{p}},\varthetabf)$.
As $\varthetabf\to \widetilde\varthetabf$, the RHS above is approaching to
\begin{align*}
\frac{1}{\Gamma_{\thetabf}\left(\widetilde{\bm{p}},\widetilde\varthetabf\right)} =& \frac{1}{\min_{\varthetabf\in \overline{\rm Alt}(\thetabf)}\Gamma_{\thetabf}(\widetilde{\bm{p}},\varthetabf)},
\end{align*}
which completes the proof.
\end{proof}

The following intuitive result was used above in the proof of Proposition~\ref{prop:lower_bound_for_consistent_policy}.
\begin{lemma}
\label{lem:expected-by-expected}
Fix $\thetabf\in\Theta$. Suppose Algorithm~\ref{alg:general-template-exp-family} uses an (anytime) allocation rule under which satisfies the following property: along any indefinite-allocation sample path, the empirical allocation $\bm{p}_t$ converges strongly to a probability vector $\widetilde{\bm{p}}\in\Sigma_k$. Then,
\[
\lim_{n\to\infty} \frac{\E[N_{\tau,i}]}{\E[\tau]} = \widetilde{p}_i,  \quad \forall i\in[k].
\]
\end{lemma}

\begin{proof}[Proof of Lemma~\ref{lem:expected-by-expected}]
In this proof, we use $\tau_n$ to denote the stopping time to make its dependence on the population size $n\in\mathbb{N}_1$ explicit. We fix an arbitrary $i\in [k]$ throughout. 

The property $\bm{p}_t \Sto\widetilde{\bm{p}}$ satisfied by the allocation rule yields that for any $\epsilon>0$, there exists a random time $T_\epsilon \in \mathcal{L}^1$ such that 
\begin{equation}
	\label{eq:property of ut}
	t\geq T_\epsilon 
	\quad\implies\quad
    \left|\frac{N_{t,i}}{t} - \widetilde{p}_i\right| < \epsilon.
\end{equation}

Under Algorithm~\ref{alg:general-template-exp-family} with this allocation rule, as $n$ grows, $\tau_n$ is non-decreasing and $\lim_{n\to\infty}\tau_n = \infty$  almost surely. 
Hence, there is an almost surely finite random variable 
\[ 
M_\epsilon \triangleq \inf\left\{n \in \mathbb{N}_1 :  \tau_n \geq T_{\epsilon} \right\}. 
\]
By this definition, $n< M_{\epsilon}$ implies $\tau_{n} < T_\epsilon$, which leads to an essential fact in our proof:
\begin{equation}\label{eq:integrability_of_tau_M}
 \E\left[   \tau_{n} \ind\{ n< M_{\epsilon} \} \right] \leq \E[T_\epsilon] < \infty.
\end{equation}
Write
\begin{align*} 
\frac{\E[N_{\tau_n,i}]}{\E[\tau_n]} &= \frac{\E[N_{\tau_n,i}\ind\{n \geq M_\epsilon\}]}{\E[\tau_n]}+ \frac{\E[N_{\tau_n,i}\ind\{n < M_\epsilon\}]}{\E[\tau_n]}\\
&= \underbrace{\frac{\E[N_{\tau_n,i}\ind\{n \geq M_\epsilon\}]}{\E[\tau_n \ind\{n \geq M_\epsilon\}]}}_{(*)} \times \left(1-\underbrace{\frac{\E[\tau_n \ind\{n < M_\epsilon\}]}{\E[\tau_n]}}_{(**)}\right) + \underbrace{\frac{\E[N_{\tau_n,i}\ind\{n < M_\epsilon\}]}{\E[\tau_n]}}_{(***)}
\end{align*}
The term $(*)$ turns out to be the dominant one as $n\to\infty$. By the definition of $M_\epsilon$, we know that $n\geq M_\epsilon$ implies $\tau_n \geq T_\epsilon$, which implies $\tau_n (\widetilde{p}_i - \epsilon) \leq N_{\tau_n,i}  \leq \tau_n (\widetilde{p}_i + \epsilon)$. Using this,
\[
\underbrace{\frac{\E[N_{\tau_n,i}\ind\{n \geq M_\epsilon\}]}{\E[\tau_n \ind\{n \geq M_\epsilon\}]}}_{(*)} \leq \frac{\E[ \tau_n (\widetilde{p}_i + \epsilon)\ind\{n \geq M_\epsilon\}]}{\E[\tau_n \ind\{n \geq M_\epsilon\}]} = \widetilde{p}_i +\epsilon.
\]
A symmetric bound shows $(*)$ is no less than $\widetilde{p}_{i} - \epsilon$. 

Now consider $(**)$. We have that $\tau_n \leq n$ by definition. Therefore, 
\[
\left| \underbrace{\frac{\E[N_{\tau_n,i}\ind\{n < M_\epsilon\}]}{\E[\tau_n]}}_{(***)} \right|  \leq \left|\underbrace{\frac{\E[\tau_n \ind\{n < M_\epsilon\}]}{\E[\tau_n]}}_{(**)}\right| \leq  \left|\frac{\E[\tau_{M_\epsilon}\ind\{n < M_\epsilon\}]}{\E[\tau_n]}\right| \leq \left|\frac{\E[T_{\epsilon}]}{\E[\tau_n]}\right| \to 0.  
\] 
The first and second inequalities above use the basic facts that $N_{t,i} \leq t$ and that $\tau_n$ is increasing in $n$. That the final term tends to zero uses two properties. First, since $\lim_{n\to\infty}\tau_n = \infty$, the monotone convergence theorem implies $\lim_{n\to \infty} \E[\tau_n] =\infty$. Second, \eqref{eq:integrability_of_tau_M} implies $\E[\tau_{M_\epsilon}] < \infty$.

Together, these arguments show
\[
\widetilde{p}_i - \epsilon \leq \liminf_{n\to \infty} \frac{\E[N_{\tau_n,i}]}{\E[\tau_n]} \leq \limsup_{n\to \infty} \frac{\E[N_{\tau_n,i}]}{\E[\tau_n]} \leq \widetilde{p}_{i} + \epsilon.
\]
Since $\epsilon$ is arbitrary, this concludes the proof.
\end{proof}

%% file: app_TTTS.tex
\section{Proof of Proposition~\ref{prop:TTTS} (Optimality of Algorithm \ref{alg:ttts})}
\label{app:TTTS}

This appendix formally prove Proposition~\ref{prop:TTTS}, which is restated as follows,
\TTTS*

From now on, we fix a problem instance $\thetabf\in\widetilde{\Theta}$. The lower bound in Proposition~\ref{prop:lower bound} gives that for any consistent policy $\pi\in\Pi$,
\[
\liminf_{n\to\infty}\frac{\mathrm{Cost}_{\bm{\theta}}(n, \pi) }{\ln(n)} \geq \frac{1}{\Gamma_{\thetabf}(\bm{p}^*,\bm{q}^*)}.
\]
On the other hand, Theorem~\ref{thm:efficient-p-general restatement} implies that if Algorithm~\ref{alg:ttts} satisfies $\bm{p}_t\Sto \bm{p}^*$, then
\[
\limsup_{n\to \infty} \frac{\mathrm{Cost}_{\bm{\theta}}(n, \widetilde\pi) }{\ln(n)} \leq \frac{1}{\Gamma_{\thetabf}(\bm{p}^*,\bm{q}^*)}.
\]
Hence, to complete the proof of Proposition~\ref{prop:TTTS}, it suffices to prove that under Assumption~\ref{asm:sampling cost stronger}, Algorithm~\ref{alg:ttts} satisfies $\bm{p}_t\Sto \bm{p}^*$.
Our proof is closely related to the analyses in \citep{russo2020simple,qin2017improving,shang2020fixed,qin2023dualdirected}.

\subsection{Preliminaries}

\paragraph{An alternative way of measuring allocation of effort.}
For analyzing randomized algorithms such as Algorithm~\eqref{alg:ttts}, we denote the probability of measuring arm $i\in[k]$ at time $t\in\mathbb{N}_0$ by $\psi_{t,i} \triangleq \Prob(I_t=i\mid \mathcal{H}_t)$. Then we introduce an alternative way of measuring cumulative effort and effort as follows, 
\[
\Psi_{t,i}\triangleq \sum_{\ell = 0}^{t-1}\psi_{\ell,i}
\quad\text{and}\quad
w_{t,i} \triangleq \frac{\Psi_{t,i}}{t}.
\]

\paragraph{Maximal inequalities.}
Following \citet{qin2017improving} and \citet{shang2020fixed}, we introduce the following path-dependent random variable to control the impact of observation noises:
	\begin{equation}
		\label{eq:W1}
		W_1 \triangleq \sup_{(t,i)\in\mathbb{N}_0\times [k]} \sqrt{\frac{N_{t,i}+1}{\ln(N_{t,i}+e)}}\frac{|m_{t,i}-\theta_i|}{\sigma}.
	\end{equation}
	and the other path-dependent random variable to control the impact of algorithmic randomness:
	\begin{equation}
		\label{eq:W2}
		W_2 \triangleq \sup_{(t,i)\in\mathbb{N}_0\times [k]} \frac{|N_{t,i}-\Psi_{t,i}|}{\sqrt{(t+1)\ln(t+e^2)}} = \sup_{(t,i)\in\mathbb{N}_0\times [k]} \frac{|p_{t,i}-w_{t,i}|t}{\sqrt{(t+1)\ln(t+e^2)}}.
	\end{equation}

 As presented in the next result, these maximal deviations have light tails.
	\begin{lemma}[{\citet[Lemma~6]{qin2017improving}} and {\citet[Lemma~4]{shang2020fixed}}]
		\label{lem:W1 and W2}
		For any $\lambda > 0$, 
		\[
		\E[e^{\lambda W_1}] < \infty 
		\quad\text{and}\quad 
		\E[e^{\lambda W_2}] < \infty.
		\]
	\end{lemma}
Lemma~\ref{lem:W1 and W2} ensures that any $\poly(W_1,W_2)$ has finite expectation, i.e., $\poly(W_1,W_2)\in \mathcal{L}^1$. Now we provide an example of applying this result.

	\begin{corollary}
		\label{cor:W2}
		For any arm $i\in[k]$,
		\[
		p_{t,i}-w_{t,i}\Sto 0.
		\]
\end{corollary}
\begin{proof}
By the definition of $W_2$ in~\eqref{eq:W2}, there exists a deterministic time $t_0$ such that\footnote{The exponent $-0.4$ is not essential in the sense that it can be any constant in $(-0.5,0)$.} for any $i\in[k]$ and any $t\geq t_0$,
		$
		|p_{t,i} - w_{t,i}| \leq W_2t^{-0.4}.
		$ 
Now consider any fixed $\epsilon > 0$. If we further have $t\geq \left(\frac{W_2}{\epsilon}\right)^{2.5}$, then
		$|p_{t,i} - w_{t,i}| \leq \epsilon$.
  Observe that $t_0 + \left(\frac{W_2}{\epsilon}\right)^{2.5} = \poly(W_2)\in \mathcal{L}^1$, which completes the proof.
\end{proof}

\subsection{Sufficient exploration}\label{subsec:app_sufficient_exploration}

The next proposition shows that Algorithm~\ref{alg:ttts} sufficiently explores all arms:
	\begin{proposition}
 \label{prop:sufficient exploration}
  Algorithm~\ref{alg:ttts} satisfies that there exists a random time $T =\poly(W_1,W_2)$ such that for any $t\geq T$,
		\[
		\min_{i\in[k]}N_{t,i}\geq \sqrt{\frac{t}{k}}.
		\]
	\end{proposition}

To prove Proposition~\ref{prop:sufficient exploration}, 
we define 
the following auxiliary arms, representing the two ``most promising arms'':
\begin{equation}
	\label{eq:promising arms chosen by TTPS}
	J_t^{(1)} \in \argmax_{i\in[k]} \alpha_{t,i} 
	\quad\text{and}\quad
	J_t^{(2)} \in \argmax_{i\neq J_t^{(1)}} \alpha_{t,i}.
\end{equation}
where $\alpha_{t,i}$ is the posterior probability that arm $i\in [k]$ is the best at time $t\in\mathbb{N}_0$.
We further define the arm (among these two) that is less sampled as:
\begin{equation}
\label{eq:less sampled promising arm}
J_t \in \argmin_{i\in\left\{J_t^{(1)},J_t^{(2)}\right\}} N_{t,i}. 
\end{equation}
While the identity of arm $J_t$ can change over time,
the next result shows that Algorithm~\ref{alg:ttts} at any time $t$ allocates a decent amount of effort to arm $J_t$.

\begin{lemma}
	\label{lem:less sampled promising arm receives decent amount}
 Algorithm~\ref{alg:ttts} satisfies that 
	for any $t\in \mathbb{N}_0$, 
	\begin{equation*}
		\psi_{t,J_t} \geq \frac{1}{k(k-1)} \frac{c_{\min}}{c_{\min} + c_{\max}}.
	\end{equation*}
\end{lemma}

Lemma~\ref{lem:less sampled promising arm receives decent amount} generalizes \citet[Lemma~8] {qin2023dualdirected} with generalized cost functions.

\begin{proof}[Proof of Lemma~\ref{lem:less sampled promising arm receives decent amount}]
	Fix $t\in\mathbb{N}_0$. 
We lower bound the probability of measuring arm $J_t\in\left\{J_t^{(1)},J_t^{(2)}\right\}$ (defined in~\eqref{eq:less sampled promising arm}) in the following two cases, respectively.
\begin{enumerate}
\item $J_t = J_t^{(1)}$

By the definition of $\left(J_t^{(1)}, J_t^{(2)}\right)$ in~\eqref{eq:promising arms chosen by TTPS}, Algorithm~\ref{alg:ttts} chooses $\left(I_t^{(1)}, I_t^{(2)}\right) = \left(J_t^{(1)}, J_t^{(2)}\right)$ with probability $\alpha_{t,J_t^{(1)}} \frac{\alpha_{t,J_t^{(2)}}}{1-\alpha_{t,J_t^{(1)}}} \geq \frac{1}{k(k-1)}$.
Then we have
	\begin{align*}
	\psi_{t,J_t} \geq \frac{1}{k(k-1)}h_{t, J_t^{(1)},J_t^{(2)}}
 &=\frac{1}{k(k-1)}\frac{p_{t,J_t^{(2)}}C_{J_t^{(2)}}(\bm{m}_t)}{p_{t,J_t^{(1)}}C_{J_t^{(1)}}(\bm{m}_t) + p_{t,J_t^{(2)}}C_{J_t^{(2)}}(\bm{m}_t)}\\
 &\geq \frac{1}{k(k-1)}\frac{C_{J_t^{(2)}}(\bm{m}_t)}{C_{J_t^{(1)}}(\bm{m}_t) + C_{J_t^{(2)}}(\bm{m}_t)} \geq \frac{1}{k(k-1)} \frac{c_{\min}}{c_{\min} + c_{\max}},
	\end{align*}
where the equality follows from the selection rule~\eqref{eq:cost-aware-IDS}; the second inequality holds since $J_t = J_t^{(1)}$ implies $p_{t,J_t^{(1)}}\leq p_{t,J_t^{(2)}}$; the last inequality holds since Assumption~\ref{asm:sampling cost stronger} gives $C_{J_t^{(1)}}(\bm{m}_t)\leq \frac{c_{\max}}{c_{\min}}C_{J_t^{(2)}}(\bm{m}_t)$

\item $J_t = J_t^{(2)}$ 

Similarly, we have
	\begin{align*}
	\psi_{t,J_t} \geq \frac{1}{k(k-1)}\left(1 - h_{t,J_t^{(1)},J_t^{(2)}}\right) &=  \frac{1}{k(k-1)}\frac{p_{t,J_t^{(1)}}C_{J_t^{(1)}}(\bm{m}_t)}{p_{t,J_t^{(1)}}C_{J_t^{(1)}}(\bm{m}_t) + p_{t,J_t^{(2)}}C_{J_t^{(2)}}(\bm{m}_t)}\\
 &\geq \frac{1}{k(k-1)}\frac{C_{J_t^{(1)}}(\bm{m}_t)}{C_{J_t^{(1)}}(\bm{m}_t) + C_{J_t^{(2)}}(\bm{m}_t)} \geq \frac{1}{k(k-1)} \frac{c_{\min}}{c_{\min} + c_{\max}},
	\end{align*}
where where the equality follows from the selection rule~\eqref{eq:cost-aware-IDS}; the second inequality holds since $J_t = J_t^{(2)}$ implies $p_{t,J_t^{(2)}}\leq p_{t,J_t^{(1)}}$; and the last inequality holds since Assumption~\ref{asm:sampling cost stronger} implies $C_{J_t^{(2)}}(\bm{m}_t)\leq \frac{c_{\max}}{c_{\min}}C_{J_t^{(1)}}(\bm{m}_t)$.
 \end{enumerate}
This completes the proof.
\end{proof}

Lemma~\ref{lem:less sampled promising arm receives decent amount} serves as a counterpart of \citet[Lemma~10]{shang2020fixed} for our cost-aware selection rule adopted in Algorithm~\ref{alg:ttts}, while \citet[Lemma~10]{shang2020fixed} is stated for top-two TS with a fixed coin bias.
With Lemma~\ref{lem:less sampled promising arm receives decent amount}, we are ready to complete the proof of Proposition~\ref{prop:sufficient exploration} based on \citet[Lemmas~9 and 11]{shang2020fixed}.
\begin{proof}[Proof of Proposition~\ref{prop:sufficient exploration}] 
For any $t\in\mathbb{N}_0$ and $s\geq 0$, define the insufficiently sampled set:
\[
U_t^s \triangleq \{i\in [k] \,:\, N_{t,i} < \sqrt{s}\}.
\]
Combining \citet[Lemma~9]{shang2020fixed} and the proof of \citet[Lemma~11]{shang2020fixed}  implies that if top-two TS with a selection rule satisfies that
		\begin{equation}
  \label{eq:if less sampled promising arm receives decent amount}
		\exists \psi_{\min} > 0 \,\,:\,\, \psi_{t,J_t} \geq \psi_{\min}, \quad \forall t\in\mathbb{N}_0,
		\end{equation}
then there exists $S=\poly(W_1,W_2)$ such that for any $s\geq S$, $U^s_{\lfloor k\cdot s\rfloor} = \emptyset$. 
By Lemma~\ref{lem:less sampled promising arm receives decent amount}, Algorithm~\ref{alg:ttts} satisfies the condition~\eqref{eq:if less sampled promising arm receives decent amount}. 
Taking $T\triangleq  k\cdot S =\poly(W_1,W_2)$ completes the proof of Proposition~\ref{prop:sufficient exploration}. 
\end{proof}

\subsection{Empirical version of exploitation rate condition}

We show that under Algorithm~\ref{alg:ttts}, the empirical version of exploitation rate condition~\eqref{eq:exploitation-rate-gaussian} holds ``asymptotically".

 \begin{proposition}
        \label{prop:empirical overall balance}
		Algorithm~\ref{alg:ttts} ensures that
		\label{prop:overall_balance_Psi}
		\[
		C_{I^*}(\thetabf)w_{t,I^*}^2 - \sum_{j\neq I^*}C_j(\thetabf)w_{t,j}^2\Sto 0
		\quad\text{and}\quad
		C_{I^*}(\thetabf)p_{t,I^*}^2 - \sum_{j\neq I^*}C_j(\thetabf)p_{t,j}^2\Sto 0.
		\]
\end{proposition}

Proposition~\ref{prop:empirical overall balance} generalizes \citet[Proposition~12] {qin2023dualdirected} with generalized cost functions.

  \begin{proof}[Proof of Proposition~\ref{prop:empirical overall balance}] 

By Corollary~\ref{cor:W2}, we only need to show
		\[
		C_{I^*}(\thetabf)w_{t,I^*}^2 - \sum_{j\neq I^*}C_j(\thetabf)w_{t,j}^2\Sto 0,
		\]
  which is equivalent to 
  \[
  \frac{G_t}{t^2}\Sto 0
  \quad\text{where}\quad 
  G_t \triangleq {C_{I^*}(\thetabf)\Psi_{t,I^*}^2} - \sum_{j\neq I^*}{C_j(\thetabf)\Psi_{t,j}^2},\quad\forall t\in\mathbb{N}_0.
  \]
		We further argue that it suffices to show 
		\begin{equation}
	\label{eq:overall_balance_Psi_sufficient_condition}
			\exists \delta>0,\, \exists M = \poly(W_1,W_2),\, \exists T = \poly(W_1,W_2) \,\,:\,\, t\geq T\implies  |G_{t+1}-G_t|\leq M \cdot t^{1-\delta}.
		\end{equation}
		Suppose \eqref{eq:overall_balance_Psi_sufficient_condition} holds. For any $t\geq T$,
		\begin{align}
			|G_{t}| 
			= \left|G_T +\sum_{\ell=T}^{t-1} (G_{\ell+1} -G_\ell)\right| 
			&\leq  M\left(\left|G_T\right| + \sum_{\ell = T}^{t-1}|G_{\ell+1} -G_\ell|\right) \nonumber\\
			&\leq M\left(\left|G_T\right| + \sum_{\ell = T}^{t-1}\ell^{1-\delta}\right) \nonumber\\
			&\leq M\left(\left|G_T\right| + \int_{T+1}^{t+1}x^{1-\delta}\mathrm{d}x\right) \nonumber\\
			&= M\left(\left|G_T\right| + \frac{1}{2-\delta}\left[(t+1)^{2-\delta}-(T+1)^{2-\delta}\right]\right) \nonumber\\
			&\leq M\cdot C_{\max}(\thetabf)  T^2 + \frac{M}{2-\delta}(t+1)^{2-\delta}, \label{eq:upper bound on G_t}
		\end{align}
		where the last inequality holds  since
		\[
  G_T\leq C_{I^*}(\thetabf)\Psi_{T,I^*}^2\leq C_{\max}(\thetabf)T^2
  \quad\text{and}\quad
		G_T\geq -\sum_{j\neq I^*}C_j(\thetabf)\Psi_{T,j}^2 \geq -C_{\max}(\thetabf)\left(\sum_{j\neq I^*}\Psi_{T,j}\right)^2\geq -C_{\max}(\thetabf) T^2.
     \]
		Since the RHS of~\eqref{eq:upper bound on G_t} is $O(t^{2-\delta})$, for any $\epsilon > 0$, there exists a random time $T_\epsilon = \poly(W_1,W_2)$ such that for any $t\geq T_\epsilon$, the RHS of~\eqref{eq:upper bound on G_t} is bounded by $\epsilon$. This gives
		$
		\frac{|G_t|}{t^2}\Sto 0.
		$

Now we are ready to complete the proof of Proposition~\ref{prop:empirical overall balance} by showing the sufficient condition in~\eqref{eq:overall_balance_Psi_sufficient_condition}.
Since $\thetabf$ is fixed, for notational convenience, for any $(t,i) \in \mathbb{N}_0\times [k]$, we write
\[
C_{t,i}  = C_i(\bm{m}_t) \quad\text{and}\quad C_{i}  = C_i(\thetabf).
\]
  
Fix $t\in\mathbb{N}_0$.  We calculate
		\begin{align*}
			G_{t+1} - G_t &= C_{I^*}\left[\left(\Psi_{t,I^*}+\psi_{t,I^*}\right)^2 - \Psi_{t,I^*}^2\right]
			- \sum_{j\neq I^*}C_j\left[\left(\Psi_{t,j}+\psi_{t,j}\right)^2 - \Psi_{t,j}^2\right] \\
			& =  \underbrace{\left(C_{I^*}\psi_{t,I^*}^2 -\sum_{j\neq I^*}C_j\psi_{t,j}^2\right)}_{(a)}
   + 2 \underbrace{\left(C_{I^*}\Psi_{t,I^*}\psi_{t,I^*} - \sum_{j\neq I^*}C_j\Psi_{t,j}\psi_{t,j}\right)}_{(b)}.
		\end{align*}
		We first upper and lower bound term $(a)$ can be bounded as follows,
            \[
            (a) \leq C_{I^*}\psi_{t,I^*}^2 \leq C_{\max}
            \quad\text{and}\quad
            (a) \geq -\sum_{j\neq I^*}C_j\psi_{t,j}^2\geq -C_{\max}  \left(\sum_{j\neq I^*}\psi_{t,j}\right)^2\geq -C_{\max},
            \]
            where $C_{\max} = C_{\max}(\thetabf)$.
  
  Applying Lemma~\ref{lem:psi_bounds}, we can upper and lower bound term $(b)$ as follows,
  		\begin{align*}
			&(b)\\
			\leq&  C_{I^*}\Psi_{t,I^*}\left[\alpha_{t,I^*}\sum_{j\neq I^*} \frac{\alpha_{t,j}}{1-\alpha_{t,I^*}}\frac{C_{t,j}p_{t,j}} {C_{t,I^*}p_{t,I^*}+C_{t,j}p_{t,j}} +  (1-\alpha_{t,I^*})\right]  -  \sum_{j\neq I^*}C_j\Psi_{t,j}\alpha_{t,I^*}\frac{\alpha_{t,j}}{1-\alpha_{t,I^*}} \frac{C_{t,I^*}p_{t,I^*}}{C_{t,I^*}p_{t,I^*}+C_{t,j}p_{t,j}} \\
   =& (1-\alpha_{t,I^*})C_{I^*}\Psi_{t,I^*} + \alpha_{t,I^*}\sum_{j\neq I^*} \frac{\alpha_{t,j}}{1-\alpha_{t,I^*}}\frac{(C_{I^*}\Psi_{t,I^*})(C_{t,j}p_{t,j}) - (C_j\Psi_{t,j})(C_{t,I^*}p_{t,I^*})} {C_{t,I^*}p_{t,I^*}+C_{t,j}p_{t,j}}  \\
			\leq& (1-\alpha_{t,I^*})C_{\max}t + \alpha_{t,I^*}\sum_{j\neq I^*} \frac{\alpha_{t,j}}{1-\alpha_{t,I^*}}\frac{(C_{I^*}\Psi_{t,I^*})(C_{t,j}p_{t,j}) - (C_j\Psi_{t,j})(C_{t,I^*}p_{t,I^*})} {C_{t,I^*}p_{t,I^*}+C_{t,j}p_{t,j}},
		\end{align*}
		where the last inequality follows from $C_{I^*}\leq C_{\max}$ and $\Psi_{t,I^*}\leq t$, and
		\begin{align*}
			&(b)\\
			\geq&  C_{I^*}\Psi_{t,I^*}\alpha_{t,I^*}\sum_{j\neq I^*} \frac{\alpha_{t,j}}{1-\alpha_{t,I^*}}\frac{C_{t,j}p_{t,j}} {C_{t,I^*}p_{t,I^*}+C_{t,j}p_{t,j}}  -  \sum_{j\neq I^*}C_j\Psi_{t,j}\left[\alpha_{t,I^*}\frac{\alpha_{t,j}}{1-\alpha_{t,I^*}} \frac{C_{t,I^*}p_{t,I^*}}{C_{t,I^*}p_{t,I^*}+C_{t,j}p_{t,j}} + (1-\alpha_{t,I^*})\right]  \\
			=& -(1-\alpha_{t,I^*})\sum_{j\neq I^*}C_j \Psi_{t,j} + \alpha_{t,I^*}\sum_{j\neq I^*} \frac{\alpha_{t,j}}{1-\alpha_{t,I^*}}\frac{(C_{I^*}\Psi_{t,I^*})(C_{t,j}p_{t,j}) - (C_{j}\Psi_{t,j})(C_{t,I^*}p_{t,I^*})} {C_{t,I^*}p_{t,I^*}+C_{t,j}p_{t,j}}  \\
            \geq&  -(1-\alpha_{t,I^*})C_{\max} t + \alpha_{t,I^*}\sum_{j\neq I^*} \frac{\alpha_{t,j}}{1-\alpha_{t,I^*}}\frac{(C_{I^*}\Psi_{t,I^*})(C_{t,j}p_{t,j}) - (C_{j}\Psi_{t,j})(C_{t,I^*}p_{t,I^*})} {C_{t,I^*}p_{t,I^*}+C_{t,j}p_{t,j}} ,
		\end{align*}
		where the last inequality follows from $C_j\leq C_{\max}$ for any $j\neq I^*$ and $\sum_{j\neq I^*}\Psi_{t,j}\leq t$.

Combining the bounds on terms $(a)$ and $(b)$ yields
		\begin{align*}
			|G_{t+1} - G_t| &\leq  [1 + 2(1-\alpha_{t,I^*})t]C_{\max} + 2 \alpha_{t,I^*}\sum_{j\neq I^*} \frac{\alpha_{t,j}}{1-\alpha_{t,I^*}}\left|\frac{(C_{I^*}\Psi_{t,I^*})(C_{t,j}p_{t,j}) - (C_j\Psi_{t,j})(C_{t,I^*}p_{t,I^*})} {C_{t,I^*}p_{t,I^*}+C_{t,j}p_{t,j}}\right|\\
   &\leq [1 + 2(1-\alpha_{t,I^*})t]C_{\max}  + 2\max_{j\neq I^*} \left|\frac{(C_{I^*}\Psi_{t,I^*})(C_{t,j}p_{t,j}) - (C_j\Psi_{t,j})(C_{t,I^*}p_{t,I^*})} {C_{t,I^*}p_{t,I^*}+C_{t,j}p_{t,j}}\right|\\
   &=  \underbrace{[1 + 2(1-\alpha_{t,I^*})t]}_{(c)}C_{\max} + 2\max_{j\neq I^*} \underbrace{\left|\frac{(C_{I^*}\Psi_{t,I^*})(C_{t,j}N_{t,j}) - (C_j\Psi_{t,j})(C_{t,I^*}N_{t,I^*})} {C_{t,I^*}N_{t,I^*}+C_{t,j}N_{t,j}}\right|}_{(d)},
		\end{align*}
where the second inequality follows from $\alpha_{t,I^*}\leq 1$ and $\sum_{j\neq I^*}\frac{\alpha_{t,j}}{1-\alpha_{t,I^*}}=1$, and the equality replaces $(p_{t,j},p_{t,I^*})$ with $(N_{t,j},N_{t,I^*})$.

We first bound term $(c)$. As shown in the proof of~\citet[Lemma~12]{shang2020fixed}, the sufficient exploration property in Proposition~\ref{prop:sufficient exploration} implies that there exists a random time $T_0=\poly(W_1,W_2)$ such that for any $t\geq T_0$, 
\begin{equation}
\label{eq:exponent c}
\alpha_{t,j}\leq  \exp\left(-c\cdot t^{1/2}\right), \quad\forall j\neq I^*,
\quad\text{where}\quad
c \triangleq \frac{(\theta_{I^*} - \max_{j\neq I^*}\theta_j)^2}{16\sigma^2 \sqrt{k}} > 0.
\end{equation}
Also there exists a deterministic value $x_0 > 0$ such that for any $x \geq x_0$, $\exp\left(-c\cdot x^{1/2}\right) \leq \frac{1}{2(k-1)x}$.
Therefore, for $t\geq T_0 + x_0$, 
\[
(c) = 1 + 2t\sum_{j\neq I^*}\alpha_{t,j}t \leq 1 +  2t(k-1)\exp\left(-c\cdot t^{1/2}\right) \leq 2t(k-1)\frac{1}{2(k-1)x} = 2.
\]

Next we bound term $(d)$. By Lemma \ref{lem:convergence of sampling costs}, there exists a random value $M_1=\poly(W_1,W_2)$ and a random time $T_{1}=\poly(W_1,W_2)$ such that for any $t\geq T_{1}$,
		\begin{align*}
			&\left|\frac{(C_{I^*}\Psi_{t,I^*})(C_{t,j}N_{t,j}) - (C_j\Psi_{t,j})(C_{t,I^*}N_{t,I^*})} {C_{t,I^*}N_{t,I^*}+C_{t,j}N_{t,j}}\right|\\
   \leq& \left|\frac{\left(C_{t,I^*}N_{t,I^*}+ M_1\cdot t^{0.8}\right)C_{t,j}N_{t,j} - \left(C_{t,j}N_{t,j}- M_1\cdot t^{0.8} \right)C_{t,I^*}N_{t,I^*}}{C_{t,I^*}N_{t,I^*}+C_{t,j}N_{t,j}}\right|\\		
   \leq& M_1\cdot t^{0.8}.
		\end{align*}
		
  Therefore, for $t\geq T_0+x_0+T_1 = \poly(W_1,W_2)$,
		\[
  |G_{t+1} - G_t| \leq 2C_{\max} + 2M_1 \cdot t^{0.8} \leq 2\left(C_{\max} + M_1\right)t^{0.8},
  \]
  where $2\left(C_{\max} + M_1\right) = \poly(W_1,W_2)$. This completes the proof of the sufficient condition~\eqref{eq:overall_balance_Psi_sufficient_condition}.
	\end{proof}

\subsubsection{Supporting lemmas for the proof of Proposition~\ref{prop:empirical overall balance}}

We present and prove the following lemmas, which were used above in the proof of Proposition~\ref{prop:empirical overall balance}.

\begin{lemma}
		\label{lem:psi_bounds}
  Algorithm~\ref{alg:ttts} satisfies that for any $t\in\mathbb{N}_0$,
		\[
		0
		\leq \psi_{t,I^*} -
		\alpha_{t,I^*}\sum_{i\neq I^*} \frac{\alpha_{t,i}}{1-\alpha_{t,I^*}}\frac{p_{t,i}C_{t,i}(\bm{m}_t)} {p_{t,I^*}C_{I^*}(\bm{m}_t)+p_{t,i}C_{i}(\bm{m}_t)} 
        \leq 1-\alpha_{t,I^*},
		\]
		and 
		\[
  \forall j\neq I^*, \quad
		0
		\leq \psi_{t,j} -
		\alpha_{t,I^*}\frac{\alpha_{t,j}}{1-\alpha_{t,I^*}} \frac{p_{t,I^*}C_{t,I^*}(\bm{m}_t)}{p_{t,I^*}C_{t,I^*}(\bm{m}_t)+p_{t,j}C_{t,j}(\bm{m}_t)} 
  \leq 1-\alpha_{t,I^*}.
		\]
\end{lemma}

Lemma~\ref{lem:psi_bounds} generalizes \citet[Lemma~6] {qin2023dualdirected} with generalized cost functions.
Lemma~\ref{lem:convergence of sampling costs} is a technical result, which generalizes the argument in the proof of Corollary \ref{cor:W2}.
 \begin{lemma}
 \label{lem:convergence of sampling costs}
    Algorithm~\ref{alg:ttts} satisfies that there exist a random value $M=\poly(W_1,W_2)$ and a random time $T=\poly(W_1,W_2)$ such that for any $t\geq T$,
    \[
    |C_{i}(\bm{m}_t)N_{t,i} - C_i(\thetabf)\Psi_{t,i}| \leq M \cdot t^{0.8}, \quad \forall i\in [k].
    \]
 \end{lemma}

We first complete the proof of Lemma~\ref{lem:convergence of sampling costs}, and then prove Lemma~\ref{lem:psi_bounds}.

\begin{proof}[Proof of Lemma~\ref{lem:convergence of sampling costs}]

We write
\begin{align*}
|C_{i}(\bm{m}_t)N_{t,i} - C_i(\thetabf)\Psi_{t,i}| &= \left|\left(C_{i}(\bm{m}_t)-C_i(\thetabf)\right)N_{t,i} + C_i(\thetabf)\left(N_{t,i}-\Psi_{t,i}\right)\right|\\
&\leq |C_{i}(\bm{m}_t)-C_i(\thetabf)|t + |N_{t,i}-\Psi_{t,i}|C_i(\thetabf)\\
&\leq |C_{i}(\bm{m}_t)-C_i(\thetabf)|t + |N_{t,i}-\Psi_{t,i}|C_{\max}(\thetabf).
\end{align*}

We first observe that by the definition of $W_2$ in~\eqref{eq:W2}, there exists a deterministic time $t_0$ such that for any $t\geq t_0$ and any $i\in[k]$, $|N_{t,i}-\Psi_{t,i}|\leq W_2t^{0.6}$.

Next we are going to bound $|C_{i}(\bm{m}_t)-C_i(\thetabf)|$. By Proposition \ref{prop:sufficient exploration}, there exists a random time $T_0=\poly(W_1,W_2)$ such that for any $t\geq T_0$, $\min_{i\in[k]}N_{t,i}\geq \sqrt{\frac{t}{k}}$. Also there exists a deterministic value $x_0 > 0$ such that $x \geq  x_0$ implies $\sqrt{\frac{\ln(x+e)}{x+1}} \leq x^{-0.4}$.
Hence, for any $t\geq T_0 + k x_0^2$, $\min_{i\in[k]}N_{t,i}\geq \sqrt{\frac{t}{k}}\geq x_0$,
and then
\[
|m_{t,i} -\theta_i|\leq \sigma W_1\sqrt{\frac{\ln(N_{t,i}+e)}{N_{t,i}+1}} \leq \sigma W_1 N_{t,i}^{-0.4} \leq \sigma W_1\left(\frac{k}{t}\right)^{0.2}, \quad\forall i\in[k],
\]
where the first inequality applies the definition of $W_1$ in~\eqref{eq:W1}. 
Furthermore, for any $t\geq T_0 + k x_0^2$, we have
\[
\|\bm{m}_t - \thetabf\|_2 = \sqrt{\sum_{i\in[k]}(m_{t,i} -\theta_i)^2} \leq \sqrt{k}\cdot \sigma W_1 \left(\frac{k}{t}\right)^{0.2} = 
\sigma k^{0.7} W_1 t^{-0.2},
\]
and thus by Lipschitz continuity (Assumption~\ref{asm:sampling cost stronger}),
\[
|C_{i}(\bm{m}_t) - C_i(\thetabf)| \leq \rho(\thetabf)\|\bm{m}_t - \thetabf\|_2\leq \rho(\thetabf)\sigma k^{0.7} W_1 t^{-0.2},
\]
where $\rho(\thetabf)$ is the Lipschitz constant.

Therefore, for any $t\geq t_0 + T_0 + k x_0^2 = \poly(W_1,W_2)$ and any $i\in[k]$,
\begin{align*}
|C_{i}(\bm{m}_t)N_{t,i} - C_i(\thetabf)\Psi_{t,i}| 
&\leq |C_{i}(\bm{m}_t)-C_i(\thetabf)|t + |N_{t,i}-\Psi_{t,i}|C_{\max}(\thetabf) \\
&\leq \rho(\thetabf)\sigma k^{0.7}W_1t^{0.8} + C_{\max}(\thetabf)W_2t^{0.6}\\
&\leq \left[\rho(\thetabf)\sigma k^{0.7}W_1 + C_{\max}(\thetabf)W_2\right] t^{0.8},
\end{align*}
where $\rho(\thetabf)\sigma k^{0.7}W_1 + C_{\max}(\thetabf)W_2 = \poly(W_1,W_2)$.
This completes the proof of Lemma~\ref{lem:convergence of sampling costs}.
\end{proof}

Now we complete the proof of Lemma~\ref{lem:psi_bounds}.

\begin{proof}[Proof of Lemma~\ref{lem:psi_bounds}]
Since $\thetabf$ is fixed, For notational convenience, for any $(t,j)\in\mathbb{N}_0\times [k]$, we write
$C_{t,j}  = C_j(\bm{m}_t)$. 

Now fix $t\in\mathbb{N}_0$ and consider an arm $j\in [k]$. For any alternative arm $i \neq j$, Algorithm~\ref{alg:ttts} chooses
\begin{equation*}
			\left(I_t^{(1)}, I_t^{(2)}\right) = 
			\begin{cases}
				(j,i), & \text{with probability } \alpha_{t,j}\frac{\alpha_{t,i}}{1-\alpha_{t,j}},\\
				(i,j), & \text{with probability } \alpha_{t,i}\frac{\alpha_{t,j}}{1-\alpha_{t,i}},
			\end{cases} 
\end{equation*}
		and then the selection rule~\eqref{eq:cost-aware-IDS} used in Algorithm~\ref{alg:ttts} implies
\begin{equation}
\label{eq:probability of measuring an arm under TTTS with IDS}
			\psi_{t,j} = \alpha_{t,j}\sum_{i\neq j} \frac{\alpha_{t,i}}{1-\alpha_{t,j}}\frac{p_{t,i}C_{t,i}} {p_{t,j}C_{t,j}+p_{t,i}C_{t,i}} + \sum_{i\neq j}\alpha_{t,i}\frac{\alpha_{t,j}}{1-\alpha_{t,i}}\frac{p_{t,i}C_{t,i}} {p_{t,j}C_{t,j}+p_{t,i}C_{t,i}}.
\end{equation}

We first lower and upper bound $\psi_{t,I^*}$. Replacing $j$ with $I^*$ in~\eqref{eq:probability of measuring an arm under TTTS with IDS} gives
		\begin{align*}
			0\leq \psi_{t,I^*} - \alpha_{t,I^*}\sum_{i \neq I^*} \frac{\alpha_{t,i}}{1-\alpha_{t,I^*}}\frac{p_{t,i}C_{t,i}} {p_{t,I^*}C_{t,I^*}+p_{t,i}C_{t,i}}
   \leq 1-\alpha_{t,I^*}
		\end{align*}
		where the upper bound follows from 
  \begin{equation}
  \label{eq:simple analysis}
  \sum_{i\neq I^*}\alpha_{t,i}\frac{\alpha_{t,I^*}}{1-\alpha_{t,i}}\frac{p_{t,i}C_{t,i}} {p_{t,I^*}C_{t,I^*}+p_{t,i}C_{t,i}} 
  \leq  \sum_{i\neq I^*} \alpha_{t,i}\frac{\alpha_{t,I^*}}{1-\alpha_{t,i}} \leq \sum_{i\neq I^*} \alpha_{t,i} = 1-\alpha_{t,I^*}.
  \end{equation}
  The second inequality above holds since $\frac{\alpha_{t,I^*}}{1-\alpha_{t,i}} = \frac{\alpha_{t,I^*}}{\sum_{i'\neq i}\alpha_{t,i'}}\leq 1$. This completes the proof of the lower and upper bounds on $\psi_{t,I^*}$ in Lemma~\ref{lem:psi_bounds}.

Now we lower and upper bound $\psi_{t,j}$ with $j\neq I^*$.
We rewrite $\psi_{t,j}$ in~\eqref{eq:probability of measuring an arm under TTTS with IDS} as follows,
		\begin{align}
			&\psi_{t,j}  -  \alpha_{t,I^*}\frac{\alpha_{t,j}}{1-\alpha_{t,I^*}} \frac{p_{t,I^*}C_{t,I^*}}{p_{t,j}C_{t,j}+p_{t,I^*}C_{t,I^*}}\nonumber\\
   =&\, \alpha_{t,j}\sum_{i\neq j} \frac{\alpha_{t,i}}{1-\alpha_{t,j}}\frac{p_{t,i}C_{t,i}}{p_{t,j}C_{t,j}+p_{t,i}C_{t,i}}
			+ \sum_{i\neq j,I^*}\alpha_{t,i}\frac{\alpha_{t,j}}{1-\alpha_{t,i}}\frac{p_{t,i}C_{t,i}}{p_{t,j}C_{t,j}+p_{t,i}C_{t,i}}.
   \label{eq:gap}
		\end{align}
		Applying the same analysis in~\eqref{eq:simple analysis} yields that the RHS of~\eqref{eq:gap} is upper bounded by
            \[
            \alpha_{t,j} + \sum_{i\neq j, I^*} \alpha_{t,i} = \sum_{i\neq I^*}\alpha_{t,i} = 1 - \alpha_{t,I^*}.
            \]
            Observe that the RHS of~\eqref{eq:gap} is nonnegative. This completes the proof of the lower and upper bounds on $\psi_{t,j}$ with $j\neq I^*$ in Lemma~\ref{lem:psi_bounds}.
\end{proof}

\subsection{Completing the proof of strong convergence to optimal proportions}

We first observe the equivalence between the strong convergence to optimal proportions ($\bm{p}_t\Sto \bm{p}^*$) and the strong convergence in the following result, with the proof being exactly the same as that of \citet[Proposition~10]{qin2023dualdirected}.

\begin{proposition}
\label{prop:TTTS restatement}
Algorithm~\ref{alg:ttts} satisfies that 
\[
\frac{w_{t,j}}{w_{t,I^*}} \Sto \frac{p_{j}^*}{p_{I^*}^*}, \quad\forall j\neq I^*.
\]
\end{proposition}

The rest of this subsection is to prove the strong convergence in Proposition~\ref{prop:TTTS restatement}. We first present several results implied by the empirical version of exploitation rate condition in Proposition~\ref{prop:empirical overall balance}.

\subsubsection{Implication of empirical version of exploitation rate condition in Proposition~\ref{prop:empirical overall balance}}
The empirical version of exploitation rate condition in Proposition~\ref{prop:empirical overall balance} implies that the fraction of measurements allocated to the best arm $I^*$ is uniformly lower bounded from zero for any large enough $t$.

\begin{lemma}
\label{lem:proportion uniform lower bound}
Algorithm~\ref{alg:ttts} satisfies that there exist deterministic values $b_1,b_2\in (0,\infty)$ and a random time $T \in \mathcal{L}^1$ such that for any $t\geq T$,
		\[
		b_1 \leq w_{t,I^*} \leq b_2
		\quad\text{and}\quad
		b_1 \leq p_{t,I^*} \leq b_2.
		\]
\end{lemma}
This result extends \citet[Lemma~7]{qin2023dualdirected} for generalized cost functions, using the exact same proof.
Since the uniform lower bound $b_1 > 0$, we can rewrite the empirical version of exploitation rate condition in Proposition~\ref{prop:empirical overall balance} as follows.
\begin{corollary}
		\label{cor:overall_balance_psi}
		Algorithm~\ref{alg:ttts} satisfies that
		\[
		\sum_{j\neq I^*}\frac{C_j(\thetabf) w_{t,j}^2}{C_{I^*}(\thetabf)w_{t,I^*}^2}\Sto 1
		\quad\text{and}\quad
		\sum_{j\neq I^*}\frac{C_j(\thetabf) p_{t,j}^2}{C_{I^*}(\thetabf)p_{t,I^*}^2}\Sto 1.
		\]
\end{corollary}

\subsubsection{Completing the proof of Proposition~\ref{prop:TTTS restatement}}

The following result demonstrates that when time $t$ is large, if the ratio between the empirical proportions allocated to a suboptimal arm and the best arm exceeds the ratio of the optimal proportions, then the probability of sampling this suboptimal arm is exponentially small in $t$.

\begin{lemma}
Algorithm~\ref{alg:ttts} satisfies that for any $\epsilon > 0$, there exist a deterministic constant $c_\epsilon>0$ and a random time $T_\epsilon\in\mathcal{L}^1$ such that for any $t\geq T_\epsilon$, 
\begin{equation}
\label{eq:over sample}
\forall j\neq I^*: \quad \frac{w_{t,j}}{w_{t,I^*}} > \frac{p_{j}^* + \epsilon}{p_{I^*}^*} \,\,\implies\,\, \psi_{t,j} \leq \exp(c_\epsilon t) + (k-1)\exp\left(-c\cdot t^{1/2}\right),
\end{equation}
where $c = \frac{(\theta_{I^*} - \max_{j\neq I^*}\theta_j)^2}{16\sigma^2 \sqrt{k}} > 0$ is defined in~\eqref{eq:exponent c}.
\end{lemma}

The above result extends \citet[Lemma~12]{qin2023dualdirected} for generalized cost functions, using the exact same proof based on Corollary~\ref{cor:overall_balance_psi} and the sufficient exploration property in Proposition~\ref{prop:sufficient exploration}.

Given the above result suggests, one may expect that when time $t$ is large, the proportion ratios are likely to self-correct. The next result formalizes this observation.

\begin{lemma}
\label{lem:not over sample}
Algorithm~\ref{alg:ttts} satisfies that for any $\epsilon > 0$, there exists a random time $T_\epsilon\in\mathcal{L}^1$ such that for any $t\geq T_\epsilon$,
\[
\frac{w_{t,j}}{w_{t,I^*}} \leq \frac{p_{j}^* + \epsilon}{p_{I^*}^*}, \quad\forall j\neq I^*.
\]
\end{lemma}

Lemma~\ref{lem:not over sample} extends \citet[Lemma~13]{qin2023dualdirected} for generalized cost functions, using the exact same proof based on the property~\eqref{eq:over sample} and the strict positivity of the uniform lower bound in Lemma~\ref{lem:proportion uniform lower bound}.

Now we are ready to complete the proof of the strong convergence in Proposition~\ref{prop:TTTS restatement}.
\begin{proof}[Proof of Proposition~\ref{prop:TTTS restatement}]
Combining Lemma~\ref{lem:not over sample} and Corollary~\ref{cor:overall_balance_psi} yields Proposition~\ref{prop:TTTS restatement}.
\end{proof}

%% file: app_Pareto_frontier_new.tex
\section{Proofs for Section~\ref{subsec:frontier}}
\label{app:frontier}

In this appendix, we complete the proofs of the results presented in Section~\ref{subsec:frontier} for a fixed problem instance $\thetabf\in\Theta$. Note that Section~\ref{subsec:frontier} focuses on Gaussian distributions, but most of the results in this appendix apply to one-dimensional exponential family distributions.

\subsection{Length and regret of information-balanced policies}

We first construct policies that define the points on the Pareto frontier.
\begin{definition}[Information-balanced policies]
For any $\beta\in(0,1)$, let $\pi^{(\beta)}$ implement Algorithm~\ref{alg:general-template-exp-family} with any allocation rule such that its empirical allocation $\bm{p}_t$ strongly converges to $\bm{p}^{(\beta)} =\left(p^{(\beta)}_1,\ldots,p^{(\beta)}_k\right)$
with the exploitation rate $p^{(\beta)}_{I^*}=\beta$ and whose exploration rates $\left(p^{(\beta)}_j\right)_{j\neq I^*}$ are identified by information balance condition~\eqref{eq:info-balance-general}.
\end{definition}

The next result derives both the normalized length and normalized regret of policy $\pi^{(\beta)}$, which are
\begin{equation}
\label{eq:NL}
L^{(\beta)}_{\thetabf}\triangleq
\frac{1}{ \min_{\varthetabf\in \overline{\rm Alt}(\thetabf)} \sum_{i\in[k]} p^{(\beta)}_i {\rm KL}(\theta_i,\vartheta_i)  }  
= \frac{1}{D_{\thetabf, I^*,j}\left(\beta, p_j^{(\beta)}\right)}, \quad\forall j\neq I^*,
\end{equation}
and
\begin{equation}
\label{eq:NR}
R^{(\beta)}_{\thetabf} \triangleq L^{(\beta)}_{\thetabf}  \cdot \sum_{j\neq I^*}p^{(\beta)}_j\left(\theta_{I^*} - \theta_j\right)
 = \sum_{j\neq I^*} \frac{p^{(\beta)}_j\left(\theta_{I^*} - \theta_j\right)}{D_{\thetabf, I^*,j}\left(\beta, p_j^{(\beta)}\right)}.
\end{equation}

\begin{lemma}[Normalized length and regret of information-balanced policies]
\label{lem:length-regret-under-pi-beta}
For any $\beta\in (0,1)$, an information-balanced policy $\pi^{(\beta)}$ satisfies
\[
\lim_{n\to\infty} \frac{\mathrm{Length}_{\thetabf}(n,\pi^{(\beta)})}{\ln(n)} =  L^{(\beta)}_{\thetabf} 
\quad\text{and}\quad
\lim_{n\to\infty} \frac{\mathrm{Regret}_{\thetabf}(n,\pi^{(\beta)})}{\ln(n)} =  R^{(\beta)}_{\thetabf}.
\]

\end{lemma}

\begin{proof}
We first analyze the experiment length of $\pi^{(\beta)}$. Observe that the experiment length becomes the cumulative within-experiment costs with $C_i(\thetabf) =1$ for any $i\in[k]$. 
Then Theorem~\ref{thm:efficient-p-general restatement} yields
\[
\limsup_{n\to\infty} \frac{\mathrm{Length}_{\thetabf}(n,\pi^{(\beta)})}{\ln(n)} \leq  \frac{1}{ \min_{\varthetabf\in \overline{\rm Alt}(\thetabf)} \sum_{i\in[k]} p^{(\beta)}_i {\rm KL}(\theta_i,\vartheta_i)  }  =
L^{(\beta)}_{\thetabf}.
\]
On the other hand, Theorem~\ref{thm:efficient-p-general restatement} implies that $\pi^{(\beta)}$ is a consistent policy, and then Proposition~\ref{prop:lower_bound_for_consistent_policy} gives
\[
\liminf_{n\to\infty} \frac{\mathrm{Length}_{\thetabf}(n,\pi^{(\beta)})}{\ln(n)}  \geq  \frac{1}{ \min_{\varthetabf\in \overline{\rm Alt}(\thetabf)} \sum_{i\in[k]} p^{(\beta)}_i {\rm KL}(\theta_i,\vartheta_i)  }  = L^{(\beta)}_{\thetabf}.
\]
Since the upper and lower bounds match, we have
$
\lim_{n\to\infty} \frac{\mathrm{Length}_{\thetabf}(n,\pi^{(\beta)})}{\ln(n)} =
L^{(\beta)}_{\thetabf}
$.
Furthermore, by Lemma~\ref{lem:Lemma3 in garivier2016optimal} and the information balance condition~\eqref{eq:info-balance-general}, we can rewrite
$
L^{(\beta)}_{\thetabf}  
= \frac{1}{D_{\thetabf, I^*,j}\left(\beta, p_j^{(\beta)}\right)}
$
for any $j\neq I^*$.

Now we study the total regret of $\pi^{(\beta)}$. 
Proposition~\ref{prop:bound on post-experimentation cost} ensures the post-experiment regret is $O_{\thetabf}(1)$, so
\[
\lim_{n\to\infty} \frac{\mathrm{Regret}_{\thetabf}(n,\pi^{(\beta)})}{\mathrm{Length}_{\thetabf}(n,\pi^{(\beta)})} 
= \lim_{n\to\infty} \frac{\sum_{j\neq I^*}\E_{\thetabf}^{\pi^{(\beta)}}\left[N_{\tau,j}\right]\left(\theta_{I^*} - \theta_j\right)}{\E_{\thetabf}^{\pi^{(\beta)}}[\tau]}
= \sum_{j\neq I^*}p^{(\beta)}_j\left(\theta_{I^*} - \theta_j\right),
\]
where the second equality follows from Lemma~\ref{lem:expected-by-expected}.
Hence,
\begin{align*}
\lim_{n\to\infty} \frac{\mathrm{Regret}_{\thetabf}(n,\pi^{(\beta)})}{\ln(n)} 
=& \lim_{n\to\infty} \frac{\mathrm{Length}_{\thetabf}(n,\pi^{(\beta)})}{\ln(n)} \lim_{n\to\infty} \frac{\mathrm{Regret}_{\thetabf}(n,\pi^{(\beta)})}{\mathrm{Length}_{\thetabf}(n,\pi^{(\beta)})}\\
=& L^{(\beta)}_{\thetabf}  \cdot \sum_{j\neq I^*}p^{(\beta)}_j\left(\theta_{I^*} - \theta_j\right) \\
=& \sum_{j\neq I^*} \frac{p^{(\beta)}_j\left(\theta_{I^*} - \theta_j\right)}{D_{\thetabf, I^*,j}\left(\beta, p_j^{(\beta)}\right)},
\end{align*}
which completes the proof.
\end{proof}

\subsection{Tracing the Pareto frontier}
Unsurprisingly, our analysis of the Pareto frontier between regret and length will revolve around the study of policies that minimizes a weighted combination of these two metrics. We now define the normalized length and regret of a policy, and its associated normalized cost. Here normalized cost is very similar to our usual objective except that it takes the limit supremum of length and regret separately, rather than taking it over their weighted sum.  This technical difference will not matter for policies under which the normalized length (or cost) attains its limit,  like those in Lemma \ref{lem:length-regret-under-pi-beta}.
\begin{definition}[Normalized length, regret and cost]
	For any policy $\pi \in \Pi$, define 
	\[ 
	{\rm NLength}(\pi) \triangleq \limsup_{n\to \infty} \frac{\mathrm{Length}_{\thetabf}(n,\pi)}{\ln(n)} 
	\quad\text{and}\quad
	{\rm NRegret}(\pi)  \triangleq \limsup_{n\to\infty} \frac{\mathrm{Regret}_{\thetabf}(n,\pi)}{\ln(n)},
	\]
	and for any $c\in(0,\infty)$, define 
	\[
	{\rm NCost}(\pi \mid c) \triangleq c \times {\rm NLength}(\pi) + {\rm NRegret}(\pi).
	\]
\end{definition}

\begin{definition}[Optimal exploitation rate as a function of $c$]
Define a map 
\begin{equation}
\label{eq:beta_c}
c\in(0,\infty)\mapsto \beta_c\in (0,1),
\end{equation} 
where $\beta_c$ represents the exploitation rate. The value of $\beta_c$ along with its corresponding exploration rates $\left(p_j^{(\beta_c)}\right)_{j\neq I^*}$, forms the unique probability vector that satisfies both the information balance condition \eqref{eq:info-balance-general} and the exploitation rate condition \eqref{eq:exploitation-rate-general}.
\end{definition}

\begin{lemma}\label{lem:optimizer-of-ncost} For $c\in(0,\infty)$, let $\pi^{(\beta_c)}$ be Algorithm~\ref{alg:general-template-exp-family} applied with an allocation rule under which ${\bm p}_t$ converges strongly to ${\bm p}^{(\beta_c)}$ for the indefinite-allocation sample paths \eqref{eq:allocaiton-only-sample}. Then, 
	\[
	{\rm NCost}\left(\pi^{(\beta_c)} \mid c\right) = \inf_{\pi \in \Pi} {\rm NCost}(\pi \mid c) = c\times L_{\thetabf}^{(\beta_c)}  + R_{\thetabf}^{(\beta_c)}.	
	\]
	Moreover, for $\tilde{c} \neq c$, $\pi^{(\beta_{\tilde c})}$ is not a minimizer of ${\rm NCost}(\cdot \mid c)$.
\end{lemma}
\begin{proof}
	The formula ${\rm NCost}\left(\pi^{(\beta_c)} \mid c\right)= c\times L_{\thetabf}^{(\beta_c)}  + R_{\thetabf}^{(\beta_c)}$ is an immediate consequence of Lemma~\ref{lem:length-regret-under-pi-beta}.

 The construction of $\pi^{(\beta_c)}$ shows that $\pi^{(\beta_c)}$ is universally efficient (Definition \ref{def:universal-efficiency}). 
 To show the optimality of $\pi^{(\beta_c)}$ for normalized cost, consider any other consistent policy $\pi$. 
 Then,
	\begin{align*}
		{\rm NCost}(\pi \mid c) 
		&=  c \times \limsup_{n\to\infty} \,  \frac{\mathrm{Length}_{\thetabf}(n,\pi)}{\ln(n)} + \limsup_{n\to\infty} \,  \frac{\mathrm{Regret}_{\thetabf}(n,\pi)}{\ln(n)}\\
		&\geq \limsup_{n\to\infty} \,  \frac{c \times \mathrm{Length}_{\thetabf}(n,\pi) + \mathrm{Regret}_{\thetabf}(n,\pi)}{\ln(n)} \\
		&\geq \limsup_{n\to\infty} \,  \frac{c \times \mathrm{Length}_{\thetabf}\left(n,\pi^{(\beta_c)}\right) + \mathrm{Regret}_{\thetabf}\left(n,,\pi^{(\beta_c)}\right)}{\ln(n)}\\
		&=  c \times \lim_{n\to\infty} \,  \frac{\mathrm{Length}_{\thetabf}\left(n,\pi^{(\beta_c)}\right)}{\ln(n)} +\lim_{n\to\infty}  \frac{\mathrm{Regret}_{\thetabf}\left(n,\pi^{(\beta_c)}\right)}{\ln(n)}\\
		&=  c \times \limsup_{n\to\infty} \,  \frac{\mathrm{Length}_{\thetabf}\left(n,\pi^{(\beta_c)}\right)}{\ln(n)} +\limsup_{n\to\infty}  \frac{\mathrm{Regret}_{\thetabf}\left(n,\pi^{(\beta_c)}\right)}{\ln(n)}\\
		&={\rm NCost}_{\thetabf}\left(\pi^{(\beta_c)} \mid c\right).
	\end{align*}
	 The second inequality follows from \eqref{eq:universally efficient rule} in Definition \ref{def:universal-efficiency} and the fact that for two nonnegative sequences 
$\{a_\ell\}_{\ell\in\mathbb{N}_1}$ and $\{b_\ell\}_{\ell\in\mathbb{N}_1}$, $\limsup_{\ell\to\infty}a_\ell b_\ell\leq \limsup_{\ell\to\infty}a_\ell\limsup_{\ell\to\infty}b_\ell$. The second and third equalities use the Lemma \ref{lem:length-regret-under-pi-beta} to relate limits and limit suprema. 
	
	To show the last claim, one can use the uniqueness of the optimal long-run allocation ${\bm p}^{(\beta_c)}$ in Theorem~\ref{thm:efficient-p-gaussian}.
\end{proof}

The next result is intuitive, and suggests that placing comparatively lower weight on length in the objective function (lower $c$) induces an optimized policy to  incur higher length but lower regret.   
\begin{lemma}\label{lem:len-regret-monotonicity}
	The function $c\in (0,\infty) \mapsto L_{\thetabf}^{(\beta_c)}$ is  strictly decreasing.  	The function $c\in (0,\infty) \mapsto R_{\thetabf}^{(\beta_c)}$ is  strictly increasing.
\end{lemma}
\begin{proof}Consider two length-regret cost functions with respective cost parameters $c_1<c_2$. For $j\in \{1,2\}$ set $\beta_j = \beta_{c_j}$ and let $\pi_j=\pi^{(\beta_j)}$ be as in Lemma \ref{lem:optimizer-of-ncost}. 
		
	For arbitrary $c>0$, Lemma \ref{lem:length-regret-under-pi-beta} implies
	\[ 
	{\rm NCost}(\pi_{j} \mid c) = c \times L_{\thetabf}^{(\beta_{j})} + R_{\thetabf}^{(\beta_{j})}.
	\]
	From Lemma \ref{lem:optimizer-of-ncost}, we know that $\pi_{1}$ minimizes  ${\rm NCost}(\cdot \mid c_1)$, but $\pi_{2}$ does not. A symmetric claim holds if the indices are reversed. This yields
	\begin{align*}
	c_1 \times L_{\thetabf}^{(\beta_1)} + R_{\thetabf}^{(\beta_1)}
	&< c_1 \times L_{\thetabf}^{(\beta_2)} + R_{\thetabf}^{(\beta_2)} \quad \implies 	c_1 \times \left(L_{\thetabf}^{(\beta_1)} -L_{\thetabf}^{(\beta_2)}\right) <  R_{\thetabf}^{(\beta_2)} -   R_{\thetabf}^{(\beta_1)}   \\
	c_2 \times L_{\thetabf}^{(\beta_2)} + R_{\thetabf}^{(\beta_2)}
	&< c_2 \times L_{\thetabf}^{(\beta_1)} + R_{\thetabf}^{(\beta_1)} \quad \implies R_{\thetabf}^{(\beta_2)} -   R_{\thetabf}^{(\beta_1)} < 	c_2 \times \left( L_{\thetabf}^{(\beta_1)} - L_{\thetabf}^{(\beta_2)} \right).
	\end{align*}
These inequality imply
\[
\left(c_2 - c_1\right) \times \left( L_{\thetabf}^{(\beta_1)} - L_{\thetabf}^{(\beta_2)} \right) > 0,
\]  
implying that $L_{\thetabf}^{(\beta_1)}> L_{\thetabf}^{(\beta_2)}$, as desired. Similar algebra yields the claim for normalized regret.
\end{proof}

Define the points on the Pareto frontier as
\[
L^*_{\thetabf}(R) \triangleq \inf\{ L' : (L', R') \in \mathcal{F}_{\thetabf}, \, R'\leq R \} \quad \text{and} \quad  R^*_{\thetabf}(L) \triangleq \inf\{ R' : (L', R') \in \mathcal{F}_{\thetabf}, \, L'\leq L \}.
\]
Recall also the definitions
\[
L^*_{\thetabf}  =\inf\{ L' : (L', R') \in \mathcal{F}_{\thetabf}  \} \quad \text{and} \quad R^*_{\thetabf}  = \inf\{ R' : (L', R') \in \mathcal{F}_{\thetabf}\}.
\]
The next result shows that the optimizer of length-regret objectives trace out the Pareto frontier. 
\begin{proposition}[Tracing the Pareto frontier]
\label{prop:frontier}
	For any $c\in (0,\infty)$, 
	\[ 
	L^*_{\thetabf}\left(R_{\thetabf}^{(\beta_c)}\right) = L_{\thetabf}^{(\beta_c)}  \quad \text{and} \quad R_{\thetabf}^*\left(L_{\thetabf}^{(\beta_c)}\right) = R_{\thetabf}^{(\beta_c)}.
	\]
	In addition,
	\[
	R^*_{\thetabf} = \lim_{c\to 0} R^{(\beta_c)}_{\thetabf}   \quad \text{and} \quad L^*_{\thetabf} = \lim_{c\to \infty} L^{(\beta_c)}_{\thetabf}. 
	\]
\end{proposition}
\begin{proof}
	Proceeding by contradiction, suppose $L_{\thetabf}^*\left(R_{\thetabf}^{(\beta_c)}\right) <  L_{\thetabf}^{(\beta_c)}$. Then, by definition, there is a policy $\pi$ obeying
	\[
	{\rm NLength}_{\thetabf}(\pi) < L_{\thetabf}^{(\beta_c)}
	\quad \text{and} \quad
	{\rm NRegret}_{\thetabf}(\pi) \leq R_{\thetabf}^{(\beta_c)}.
	\]
	This policy satisfies,  
	\[
	{\rm NCost}_{\thetabf}(\pi \mid c) < c \times  L_{\thetabf}^{(\beta_c)} + R_{\thetabf}^{(\beta_c)} = {\rm NCost}_{\thetabf}\left(\pi^{(\beta_c)} \mid c\right),
	\]
	contradicting Lemma \ref{lem:optimizer-of-ncost}. The proof that $R_{\thetabf}^*\left(L_{\thetabf}^{(\beta_c)}\right) = R_{\thetabf}^{(\beta_c)}$ is symmetric. 
	
	 Now we prove the second part. If $R^*_{\thetabf} < \lim_{c\to 0} R^{(\beta_c)}_{\thetabf}$, then there exists $(L,R) \in \mathcal{F}_{\thetabf}$ with $R< \lim_{c\to 0} R^{(\beta_c)}_{\thetabf}= \inf_{c>0} R_{\thetabf}^{(\beta_c)}$. 	Proceeding by contradiction, pick some $\tilde{c}>0$ sufficiently small such that 
	\[ 
	\tilde{c}\times L < \inf_{c>0} R_{\thetabf}^{(\beta_{c})}  - R. 
	\]
	
	This implies that the exists a policy $\pi\in \Pi$ which obeys, 
	\begin{align*}
		  \mathrm{NCost}_{\thetabf}(\pi \mid \tilde{c})  
		\leq \tilde{c}\times L + R 
		< \inf_{c>0} R_{\thetabf}^{(\beta_{c})}
		  \leq R_{\thetabf}^{(\beta_{\tilde{c}})} 
		& \leq \tilde{c} \times L_{\thetabf}^{(\beta_{\tilde{c}})} +  R_{\thetabf}^{(\beta_{\tilde{c}})}\\
		  &= \mathrm{NCost}_{\thetabf}\left(\pi^{(\beta_{\tilde{c}})} \mid \tilde{c}\right),
	\end{align*}
	contradicting the optimality of $\pi^{(\beta_{\tilde{c}})}$. A symmetric argument shows $L^*_{\thetabf} = \lim_{c\to \infty} L^{(\beta_c)}_{\thetabf}$. 
\end{proof}

\subsection{Proof of Lemma~\ref{lem:infinite_length}}
The next lemma is an observation about the information-balance and exploitation-rate conditions that define $\beta_c$. We state it without a detailed proof. 
\begin{lemma}\label{lem:beta-decreasing-in-c}
	The function $c\in(0,\infty)\mapsto \beta_c$ is strictly decreasing.
\end{lemma}

Now we prove Lemma \ref{lem:infinite_length}.
\begin{proof}[Proof of Lemma \ref{lem:infinite_length}]
Pick an arbitrary sequence $(R_\ell)_{\ell \in \mathbb{N}}$ with $R_\ell \downarrow R^*_{\thetabf}$. Our goal is to show that $L^*_{\thetabf}(R_\ell) = \infty$. 

Toward this goal, pick a decreasing sequence $\{c_\ell\}_{\ell\in \mathbb{N}}$ with $c_\ell \to 0$.
Lemma \ref{lem:len-regret-monotonicity} and Proposition \ref{prop:frontier} show that  $R_{\thetabf}^{(\beta_c)} \downarrow R^*_{\thetabf}$ as $c\to 0$. Therefore, we can ensure our choice of $c_{\ell}$ satisfies $R_{\thetabf}^{(\beta_{c_\ell})} \leq R_\ell$ for each $\ell$. 
It is immediate from definition that $L^*_{\thetabf}(R)$ is a  non-increasing function of $R$ (since relaxing a constraint is always weakly beneficial). This implies the bound: 
\begin{equation}\label{eq:temporary_seq_bound}
L^*_{\thetabf}(R_\ell) \geq L^*_{\thetabf}\left(R_{\thetabf}^{(\beta_{c_\ell})} \right) = L_{\thetabf}^{(\beta_{c_\ell})}, \quad \forall \ell \in \mathbb{N},
\end{equation}
where the equality uses Proposition \ref{prop:frontier}.

One can show from the definition that $\beta_c \to 1$ as $c\to 0$. Then the expression for $L_{\thetabf}^{(\beta)}$ in \eqref{eq:NL} in terms of Chernoff information, together with the definition of the Chernoff information in \eqref{eq:chernoff-info-minimizer}, implies that $L^{(\beta)}_{\thetabf}\to\infty$ as $\beta\to 1$. 
Hence, by Lemma \ref{lem:beta-decreasing-in-c}, $L_{\thetabf}^{(\beta_c)} \to \infty$ as $c\to 0$. Hence, taking 	$\ell\to \infty$ in \eqref{eq:temporary_seq_bound}  yields the result. 
\end{proof}

\subsection{Proof of Proposition~\ref{prop:Pareto_robustness}}

We are ready to complete the proof of Proposition~\ref{prop:Pareto_robustness}.

\begin{proof}[Proof of Proposition \ref{prop:Pareto_robustness}]

	Consider the optimal allocation ${\bm p}^*$ according to Theorem \ref{thm:efficient-p-gaussian} in a problem where $C_{i}(\thetabf)=1$. That is, within-experiment cost functions only penalize the experiment-length, and 
	\[
	\E_{\bm{\theta}}^{\pi}\left[ \sum_{t=0}^{\tau-1} C_{I_t}(\thetabf)\right] = {\rm Length}_{\thetabf}(n, \pi).
	\] 
	Set $\beta_{\rm BAI} = p^*_{I^*}$.  
 (One can directly verify that $\beta_c\to \beta_{\rm BAI}$ as $c\to \infty$, but we do not use that fact.) The optimality of the sampling probabilities ${\bm p}^* = {\bm p}^{(\beta_{\rm BAI})}$ means that  $L^{(\beta_{\rm BAI})}_{\thetabf}=L^*_{\thetabf}$.  
	
	The properties of ${\bm p}^{(\beta)}$ are studied in \cite{russo2020simple}. By \citet[Lemma 3 (also stated as Theorem 1 Part 3)]{russo2020simple}, we have
	\[
	L^{(\beta)}_{\thetabf} \leq \max\left\{\frac{\beta_{\mathrm{BAI}}}{\beta},\frac{1-\beta_{\mathrm{BAI}}}{1-\beta}\right\} L^{(\beta_{\rm BAI})}_{\thetabf}
	\leq \max\left\{\frac{1}{2\beta},\frac{1}{1-\beta}\right\} L^{(\beta_{\rm BAI})}_{\thetabf} 
	= \frac{1}{1-\beta}L^{(\beta_{\rm BAI})}_{\thetabf},
	\]
	where the second inequality applies $\beta_{\mathrm{BAI}} \leq {1/2}$ for Gaussian distributions with common variance (see \citet[Lemma~18]{qin2022adaptivity}), and the last equality holds for any $\beta\geq 1/3$.
	
	In addition, for Gaussian distributions with common variance $\sigma^2$, $D_{\thetabf, I^*,j}\left(\beta, p_j^{(\beta)}\right) = \frac{(\theta_{I^*} - \theta_j)^2}{2\sigma^2\left(\frac{1}{\beta} + \frac{1}{p_j^{(\beta)}}\right)}$. We can use this to analyze \eqref{eq:NR} as follows,
	\[ 
	R^{(\beta)}_{\thetabf} = \sum_{j\neq I^*} \frac{p^{(\beta)}_j\left(\theta_{I^*} - \theta_j\right)}{D_{\thetabf, I^*,j}\left(\beta, p_j^{(\beta)}\right)} 
	= 2\sigma^2\sum_{j\neq I^*}\frac{\frac{p_j^{(\beta)}}{\beta} + 1}{\theta_{I^*} - \theta_j} \leq 2\sigma^2\sum_{j\neq I^*}\frac{\frac{1}{\beta}}{\theta_{I^*} - \theta_j},
	\]
	where the inequality uses replaces $p_{j}^{(\beta)} \leq 1-\beta$, since $p_{I^*}^{(\beta)} = \beta$.
	The minimal regret can be written as 
	\[
	R^*_{\thetabf} = \lim_{c\to 0} R^{(\beta_c)}_{\thetabf} = \lim_{\beta \to 1} \, 2\sigma^2\sum_{j\neq I^*}\frac{\frac{p_j^{(\beta)}}{\beta} + 1}{\theta_{I^*} - \theta_j}  = 2\sigma^2\sum_{j\neq I^*}  \frac{1}{\theta_{I^*} - \theta_j}.
	\]
	The first equality above is Proposition \ref{prop:frontier}, the second uses that $\beta_c \to 1$ as $c\to 0$ (which can be verified by inspecting the formula), and the last equality uses that $p_j^{(\beta)}\to 0$ as $\beta \to 1$ for each $j$. 
	
	These formulas show that $R_{\thetabf}^{(\beta)}\leq \frac{1}{\beta}  R^{*}_{\thetabf}$, completing the proof.
\end{proof}

%% file: main.bbl
\begin{thebibliography}{73}
\providecommand{\natexlab}[1]{#1}
\providecommand{\url}[1]{\texttt{#1}}
\expandafter\ifx\csname urlstyle\endcsname\relax
  \providecommand{\doi}[1]{doi: #1}\else
  \providecommand{\doi}{doi: \begingroup \urlstyle{rm}\Url}\fi

\bibitem[Adusumilli(2022)]{adusumilli2022neyman}
Karun Adusumilli.
\newblock Neyman allocation is minimax optimal for best arm identification with
  two arms.
\newblock \emph{arXiv preprint arXiv:2204.05527}, 2022.

\bibitem[Adusumilli(2023)]{adusumilli2023sample}
Karun Adusumilli.
\newblock How to sample and when to stop sampling: The generalized wald problem
  and minimax policies, 2023.

\bibitem[Agrawal et~al.(1989{\natexlab{a}})Agrawal, Teneketzis, and
  Anantharam]{anantharam1989asymptotically2}
Rajeev Agrawal, Demosthenis Teneketzis, and Venkatachalam Anantharam.
\newblock Asymptotically efficient adaptive allocation schemes for controlled
  {M}arkov chains: finite parameter space.
\newblock \emph{IEEE Transactions on Automatic Control}, 34\penalty0
  (12):\penalty0 1249--1259, 1989{\natexlab{a}}.

\bibitem[Agrawal et~al.(1989{\natexlab{b}})Agrawal, Teneketzis, and
  Anantharam]{rajeev1989asymptotically1}
Rajeev Agrawal, Demosthenis Teneketzis, and Venkatachalam Anantharam.
\newblock Asymptotically efficient adaptive allocation schemes for controlled
  i.i.d. processes: finite parameter space.
\newblock \emph{IEEE Transactions on Automatic Control}, 34\penalty0
  (3):\penalty0 258--267, 1989{\natexlab{b}}.

\bibitem[Amadio(2020)]{Amadio2020StitchFix}
Brian Amadio.
\newblock Multi-armed bandits and the stitch fix experimentation platform.
\newblock \url{https://multithreaded.stitchfix.com/blog/2020/08/05/bandits/},
  2020.
\newblock Accessed: June 18, 2024.

\bibitem[Ariu et~al.(2021)Ariu, Kato, Komiyama, McAlinn, and
  Qin]{ariu2021policy}
Kaito Ariu, Masahiro Kato, Junpei Komiyama, Kenichiro McAlinn, and Chao Qin.
\newblock Policy choice and best arm identification: Asymptotic analysis of
  exploration sampling.
\newblock \emph{arXiv preprint arXiv:2109.08229}, 2021.

\bibitem[Arrow et~al.(1949)Arrow, Blackwell, and Girshick]{arrow1949bayes}
Kenneth~J Arrow, David Blackwell, and Meyer~A Girshick.
\newblock Bayes and minimax solutions of sequential decision problems.
\newblock \emph{Econometrica, Journal of the Econometric Society}, pages
  213--244, 1949.

\bibitem[Athey et~al.(2022)Athey, Byambadalai, Hadad, Krishnamurthy, Leung, and
  Williams]{athey2022contextual}
Susan Athey, Undral Byambadalai, Vitor Hadad, Sanath~Kumar Krishnamurthy,
  Weiwen Leung, and Joseph~Jay Williams.
\newblock Contextual bandits in a survey experiment on charitable giving:
  Within-experiment outcomes versus policy learning, 2022.

\bibitem[Bhat et~al.(2020)Bhat, Farias, Moallemi, and Sinha]{bhat2020near}
Nikhil Bhat, Vivek~F Farias, Ciamac~C Moallemi, and Deeksha Sinha.
\newblock Near-optimal {A-B} testing.
\newblock \emph{Management Science}, 66\penalty0 (10):\penalty0 4477--4495,
  2020.

\bibitem[Boyd and Vandenberghe(2004)]{boyd2004convex}
Stephen Boyd and Lieven Vandenberghe.
\newblock \emph{Convex optimization}.
\newblock Cambridge university press, 2004.

\bibitem[Bubeck et~al.(2009)Bubeck, Munos, and Stoltz]{bubeck2009pure}
S{\'e}bastien Bubeck, R{\'e}mi Munos, and Gilles Stoltz.
\newblock Pure exploration in multi-armed bandits problems.
\newblock In \emph{International conference on Algorithmic learning theory},
  pages 23--37. Springer, 2009.

\bibitem[Bubeck et~al.(2011)Bubeck, Munos, and Stoltz]{BUBECK20111832}
Sébastien Bubeck, Rémi Munos, and Gilles Stoltz.
\newblock Pure exploration in finitely-armed and continuous-armed bandits.
\newblock \emph{Theoretical Computer Science}, 412\penalty0 (19):\penalty0
  1832--1852, 2011.
\newblock Algorithmic Learning Theory (ALT 2009).

\bibitem[Bui et~al.(2011)Bui, Johari, and Mannor]{bui2011committing}
Loc Bui, Ramesh Johari, and Shie Mannor.
\newblock Committing bandits.
\newblock In J.~Shawe-Taylor, R.~Zemel, P.~Bartlett, F.~Pereira, and K.Q.
  Weinberger, editors, \emph{Advances in Neural Information Processing
  Systems}, volume~24. Curran Associates, Inc., 2011.

\bibitem[Capp{\'e} et~al.(2013)Capp{\'e}, Garivier, Maillard, Munos, and
  Stoltz]{cappe2013kullback}
Olivier Capp{\'e}, Aur{\'e}lien Garivier, Odalric-Ambrym Maillard, R{\'e}mi
  Munos, and Gilles Stoltz.
\newblock {K}ullback-{L}eibler upper confidence bounds for optimal sequential
  allocation.
\newblock \emph{The Annals of Statistics}, pages 1516--1541, 2013.

\bibitem[Caria et~al.(2023)Caria, Gordon, Kasy, Quinn, Shami, and
  Teytelboym]{caria2023adaptive}
A~Stefano Caria, Grant Gordon, Maximilian Kasy, Simon Quinn, Soha~Osman Shami,
  and Alexander Teytelboym.
\newblock {An Adaptive Targeted Field Experiment: Job Search Assistance for
  Refugees in Jordan}.
\newblock \emph{Journal of the European Economic Association}, 2023.

\bibitem[Chan and Lai(2006)]{chan2006sequential}
Hock~Peng Chan and Tze~Leung Lai.
\newblock Sequential generalized likelihood ratios and adaptive treatment
  allocation for optimal sequential selection.
\newblock \emph{Sequential Analysis}, 25\penalty0 (2):\penalty0 179--201, 2006.

\bibitem[Chapelle and Li(2011)]{chapelle2011empirical}
Olivier Chapelle and Lihong Li.
\newblock An empirical evaluation of thompson sampling.
\newblock \emph{Advances in neural information processing systems},
  24:\penalty0 2249--2257, 2011.

\bibitem[Chen et~al.(2000)Chen, Lin, Y{\"u}cesan, and
  Chick]{chen2000simulation}
Chun-Hung Chen, Jianwu Lin, Enver Y{\"u}cesan, and Stephen~E. Chick.
\newblock Simulation budget allocation for further enhancing the efficiency of
  ordinal optimization.
\newblock \emph{Discrete Event Dynamic Systems}, 10\penalty0 (3):\penalty0
  251--270, 2000.

\bibitem[Chen and Ryzhov(2023)]{chen2023balancing}
Ye~Chen and Ilya~O Ryzhov.
\newblock Balancing optimal large deviations in sequential selection.
\newblock \emph{Management Science}, 69\penalty0 (6):\penalty0 3457--3473,
  2023.

\bibitem[Chernoff(1959)]{chernoff1959sequential}
Herman Chernoff.
\newblock Sequential design of experiments.
\newblock \emph{Annals of Mathematical Statistics}, 30\penalty0 (3):\penalty0
  755--770, 1959.

\bibitem[Chick and Frazier(2012)]{chick2012sequential}
Stephen~E. Chick and Peter Frazier.
\newblock Sequential sampling with economics of selection procedures.
\newblock \emph{Management Science}, 58\penalty0 (3):\penalty0 550--569, 2012.

\bibitem[Chick and Gans(2009)]{chick2009economic}
Stephen~E. Chick and Noah Gans.
\newblock Economic analysis of simulation selection problems.
\newblock \emph{Management Science}, 55\penalty0 (3):\penalty0 421--437, 2009.

\bibitem[Chick and Inoue(2001)]{chick2001new}
Stephen~E. Chick and Koichiro Inoue.
\newblock New two-stage and sequential procedures for selecting the best
  simulated system.
\newblock \emph{Operations Research}, 49\penalty0 (5):\penalty0 732--743, 2001.

\bibitem[Chick et~al.(2021)Chick, Gans, and Yapar]{chick2021bayesian}
Stephen~E. Chick, Noah Gans, and {\"O}zge Yapar.
\newblock Bayesian sequential learning for clinical trials of multiple
  correlated medical interventions.
\newblock \emph{Management Science}, 2021.

\bibitem[Cover and Thomas(2006)]{cover2006elements}
Thomas~M. Cover and Joy~A. Thomas.
\newblock \emph{Elements of Information Theory}.
\newblock Wiley, 2006.

\bibitem[Dai et~al.(2023)Dai, Gradu, and Harshaw]{dai2023clipogd}
Jessica Dai, Paula Gradu, and Christopher Harshaw.
\newblock {CLIP}-{OGD}: An experimental design for adaptive {N}eyman allocation
  in sequential experiments.
\newblock In \emph{Thirty-seventh Conference on Neural Information Processing
  Systems}, 2023.

\bibitem[Degenne(2023)]{degenne2023existence}
R{\'e}my Degenne.
\newblock On the existence of a complexity in fixed budget bandit
  identification.
\newblock In \emph{Proceedings of Thirty Sixth Conference on Learning Theory}.
  PMLR, 2023.

\bibitem[Degenne and Koolen(2019)]{degenne2019MultipleCorrectAnswers}
R\'{e}my Degenne and Wouter~M Koolen.
\newblock Pure exploration with multiple correct answers.
\newblock In \emph{Advances in Neural Information Processing Systems},
  volume~32, 2019.

\bibitem[Degenne et~al.(2019{\natexlab{a}})Degenne, Koolen, and
  M\'{e}nard]{degenne2019Non-Asymptotic}
R\'{e}my Degenne, Wouter~M Koolen, and Pierre M\'{e}nard.
\newblock Non-asymptotic pure exploration by solving games.
\newblock In \emph{Advances in Neural Information Processing Systems},
  volume~32, 2019{\natexlab{a}}.

\bibitem[Degenne et~al.(2019{\natexlab{b}})Degenne, Nedelec, Calauzenes, and
  Perchet]{degenne2019bridging}
R\'emy Degenne, Thomas Nedelec, Clement Calauzenes, and Vianney Perchet.
\newblock Bridging the gap between regret minimization and best arm
  identification, with application to {A/B} tests.
\newblock In \emph{Proceedings of the Twenty-Second International Conference on
  Artificial Intelligence and Statistics}, volume~89, pages 1988--1996,
  2019{\natexlab{b}}.

\bibitem[Degenne et~al.(2020)Degenne, Shao, and Koolen]{degenne2020structure}
R{\'e}my Degenne, Han Shao, and Wouter Koolen.
\newblock Structure adaptive algorithms for stochastic bandits.
\newblock In \emph{International Conference on Machine Learning}, pages
  2443--2452. PMLR, 2020.

\bibitem[Efron(2022)]{efron2022exponential}
Bradley Efron.
\newblock \emph{Exponential families in theory and practice}.
\newblock Cambridge University Press, 2022.

\bibitem[Erraqabi et~al.(2017)Erraqabi, Lazaric, Valko, Brunskill, and
  Liu]{erraqabi17a}
Akram Erraqabi, Alessandro Lazaric, Michal Valko, Emma Brunskill, and Yun-En
  Liu.
\newblock {Trading off Rewards and Errors in Multi-Armed Bandits}.
\newblock In \emph{Proceedings of the 20th International Conference on
  Artificial Intelligence and Statistics}, pages 709--717. PMLR, 20--22 Apr
  2017.

\bibitem[Fudenberg et~al.(2018)Fudenberg, Strack, and
  Strzalecki]{Fudenberg2018}
Drew Fudenberg, Philipp Strack, and Tomasz Strzalecki.
\newblock Speed, accuracy, and the optimal timing of choices.
\newblock \emph{American Economic Review}, 108\penalty0 (12):\penalty0
  3651--84, December 2018.

\bibitem[Garivier and Kaufmann(2016)]{garivier2016optimal}
Aur{\'e}lien Garivier and Emilie Kaufmann.
\newblock Optimal best arm identification with fixed confidence.
\newblock In \emph{Conference on Learning Theory}, pages 998--1027. PMLR, 2016.

\bibitem[Glynn and Juneja(2004)]{glynn2004large}
Peter Glynn and Sandeep Juneja.
\newblock A large deviations perspective on ordinal optimization.
\newblock In \emph{Proceedings of the 2004 Winter Simulation Conference,
  2004.}, volume~1. IEEE, 2004.

\bibitem[Graves and Lai(1997)]{graves1997asymptotically}
Todd~L. Graves and Tze~Leung Lai.
\newblock Asymptotically efficient adaptive choice of control laws incontrolled
  {M}arkov chains.
\newblock \emph{SIAM journal on control and optimization}, 35\penalty0
  (3):\penalty0 715--743, 1997.

\bibitem[Gray(2011)]{gray2011entropy}
Robert~M. Gray.
\newblock \emph{Entropy and information theory}.
\newblock Springer Science \& Business Media, 2011.

\bibitem[Hong et~al.(2021)Hong, Fan, and Luo]{hong2021review}
L~Jeff Hong, Weiwei Fan, and Jun Luo.
\newblock Review on ranking and selection: A new perspective.
\newblock \emph{Frontiers of Engineering Management}, 8\penalty0 (3):\penalty0
  321--343, 2021.

\bibitem[Jamieson et~al.(2014)Jamieson, Malloy, Nowak, and
  Bubeck]{jamieson2014lilUCB}
Kevin Jamieson, Matthew Malloy, Robert Nowak, and Sébastien Bubeck.
\newblock {lil' UCB}: An optimal exploration algorithm for multi-armed bandits.
\newblock In \emph{Proceedings of The 27th Conference on Learning Theory},
  pages 423--439, Barcelona, Spain, 13--15 Jun 2014.

\bibitem[Jourdan et~al.(2022)Jourdan, Degenne, Baudry, de~Heide, and
  Kaufmann]{jourdan2022top}
Marc Jourdan, R{\'e}my Degenne, Dorian Baudry, Rianne de~Heide, and Emilie
  Kaufmann.
\newblock Top two algorithms revisited.
\newblock In \emph{Advances in Neural Information Processing Systems}, 2022.

\bibitem[Kanarios et~al.(2024)Kanarios, Zhang, and Ying]{kanarios2024cost}
Kellen Kanarios, Qining Zhang, and Lei Ying.
\newblock Cost aware best arm identification.
\newblock \emph{arXiv preprint arXiv:2402.16710}, 2024.

\bibitem[Kasy and Sautmann(2021)]{kasy2021adaptive}
Maximilian Kasy and Anja Sautmann.
\newblock Adaptive treatment assignment in experiments for policy choice.
\newblock \emph{Econometrica}, 89\penalty0 (1):\penalty0 113--132, 2021.

\bibitem[Kaufmann and Koolen(2021)]{Kaufmann2021martingale}
Emilie Kaufmann and Wouter~M. Koolen.
\newblock Mixture martingales revisited with applications to sequential tests
  and confidence intervals.
\newblock \emph{Journal of Machine Learning Research}, 22\penalty0
  (246):\penalty0 1--44, 2021.

\bibitem[Kaufmann et~al.(2016)Kaufmann, Capp{\'e}, and
  Garivier]{kaufmann2016complexity}
Emilie Kaufmann, Olivier Capp{\'e}, and Aur{\'e}lien Garivier.
\newblock On the complexity of best-arm identification in multi-armed bandit
  models.
\newblock \emph{Journal of Machine Learning Research}, 17\penalty0
  (1):\penalty0 1--42, 2016.

\bibitem[Krishnamurthy et~al.(2023)Krishnamurthy, Zhan, Athey, and
  Brunskill]{krishnamurthy2023proportional}
Sanath~Kumar Krishnamurthy, Ruohan Zhan, Susan Athey, and Emma Brunskill.
\newblock Proportional response: Contextual bandits for simple and cumulative
  regret minimization.
\newblock In \emph{Thirty-seventh Conference on Neural Information Processing
  Systems}, 2023.

\bibitem[Lai et~al.(1980)Lai, Levin, Robbins, and Siegmund]{lai1980sequential}
TL~Lai, Bruce Levin, Herbert Robbins, and David Siegmund.
\newblock Sequential medical trials.
\newblock \emph{Proceedings of the National Academy of Sciences}, 77\penalty0
  (6):\penalty0 3135--3138, 1980.

\bibitem[Lai and Robbins(1985)]{lai1985asymptotically}
Tze~Leung Lai and Herbert Robbins.
\newblock Asymptotically efficient adaptive allocation rules.
\newblock \emph{Advances in applied mathematics}, 6\penalty0 (1):\penalty0
  4--22, 1985.

\bibitem[Lattimore and Szepesv{\'a}ri(2020)]{lattimore2020bandit}
Tor Lattimore and Csaba Szepesv{\'a}ri.
\newblock \emph{Bandit algorithms}.
\newblock Cambridge University Press, 2020.

\bibitem[Li et~al.(2022)Li, Nogas, Song, Kumar, Durand, Rafferty, Deliu,
  Villar, and Williams]{li2022algorithms}
Tong Li, Jacob Nogas, Haochen Song, Harsh Kumar, Audrey Durand, Anna Rafferty,
  Nina Deliu, Sofia~S. Villar, and Joseph~J. Williams.
\newblock Algorithms for adaptive experiments that trade-off statistical
  analysis with reward: Combining uniform random assignment and reward
  maximization, 2022.

\bibitem[Liang et~al.(2022)Liang, Mu, and Syrgkanis]{Liang2022}
Annie Liang, Xiaosheng Mu, and Vasilis Syrgkanis.
\newblock Dynamically aggregating diverse information.
\newblock \emph{Econometrica}, 90\penalty0 (1):\penalty0 47--80, 2022.

\bibitem[Liu et~al.(2014)Liu, Mandel, Brunskill, and Popovic]{LiuMBP14}
Yun{-}En Liu, Travis Mandel, Emma Brunskill, and Zoran Popovic.
\newblock Trading off scientific knowledge and user learning with multi-armed
  bandits.
\newblock In \emph{Proceedings of the 7th International Conference on
  Educational Data Mining, {EDM} 2014, London, UK, July 4-7, 2014}, pages
  161--168. International Educational Data Mining Society {(IEDMS)}, 2014.

\bibitem[Morris and Strack(2019)]{morris2019wald}
Stephen Morris and Philipp Strack.
\newblock The wald problem and the relation of sequential sampling and ex-ante
  information costs.
\newblock \emph{Available at SSRN 2991567}, 2019.

\bibitem[Polyanskiy and Wu(2023+)]{PolyanskiyWu}
Yury Polyanskiy and Yihong Wu.
\newblock \emph{Information Theory: From Coding to Learning}.
\newblock Cambridge university press, 2023+.

\bibitem[Qin(2022)]{qin2022open}
Chao Qin.
\newblock Open problem: Optimal best arm identification with fixed-budget.
\newblock In \emph{Proceedings of Thirty Fifth Conference on Learning Theory},
  pages 5650--5654. PMLR, 02--05 Jul 2022.

\bibitem[Qin and Russo(2022)]{qin2022adaptivity}
Chao Qin and Daniel Russo.
\newblock Adaptivity and confounding in multi-armed bandit experiments.
\newblock \emph{arXiv preprint arXiv:2202.09036v3}, 2022.

\bibitem[Qin and Russo(2023)]{qin2023adaptive}
Chao Qin and Daniel Russo.
\newblock Adaptive experimentation in the presence of exogenous nonstationary
  variation, 2023.

\bibitem[Qin and You(2023)]{qin2023dualdirected}
Chao Qin and Wei You.
\newblock Dual-directed algorithm design for efficient pure exploration, 2023.

\bibitem[Qin et~al.(2017)Qin, Klabjan, and Russo]{qin2017improving}
Chao Qin, Diego Klabjan, and Daniel Russo.
\newblock Improving the expected improvement algorithm.
\newblock \emph{Advances in Neural Information Processing Systems},
  30:\penalty0 5382--5392, 2017.

\bibitem[Rosenzweig and Offer-Westort(2022)]{Rosenzweig2022Conversations}
L.~R. Rosenzweig and M.~Offer-Westort.
\newblock Conversations with a concern-addressing chatbot increase {COVID}-19
  vaccination intentions among social media users in {K}enya and {N}igeria.,
  2022.

\bibitem[Russo(2016)]{russo2016simple}
Daniel Russo.
\newblock Simple bayesian algorithms for best arm identification.
\newblock In \emph{Conference on Learning Theory}, pages 1417--1418. PMLR,
  2016.

\bibitem[Russo(2020)]{russo2020simple}
Daniel Russo.
\newblock Simple bayesian algorithms for best-arm identification.
\newblock \emph{Operations Research}, 68\penalty0 (6):\penalty0 1625--1647,
  2020.

\bibitem[Russo and Van~Roy(2018)]{russo2018learning}
Daniel Russo and Benjamin Van~Roy.
\newblock Learning to optimize via information-directed sampling.
\newblock \emph{Operations Research}, 66\penalty0 (1):\penalty0 230--252, 2018.

\bibitem[Scott(2013)]{scott2013google}
Steven Scott.
\newblock Multi-armed bandit experiments.
\newblock
  \url{https://analytics.googleblog.com/2013/01/multi-armed-bandit-experiments.html},
  2013.
\newblock Accessed: February 26, 2024.

\bibitem[Scott(2010)]{scott2010modern}
Steven~L. Scott.
\newblock A modern bayesian look at the multi-armed bandit.
\newblock \emph{Applied Stochastic Models in Business and Industry},
  26\penalty0 (6):\penalty0 639--658, 2010.

\bibitem[Shang et~al.(2020)Shang, Heide, Menard, Kaufmann, and
  Valko]{shang2020fixed}
Xuedong Shang, Rianne Heide, Pierre Menard, Emilie Kaufmann, and Michal Valko.
\newblock Fixed-confidence guarantees for bayesian best-arm identification.
\newblock In \emph{International Conference on Artificial Intelligence and
  Statistics}, pages 1823--1832. PMLR, 2020.

\bibitem[Simchi-Levi and Wang(2023)]{simchi-levi2023MABExperimentalDesign}
David Simchi-Levi and Chonghuan Wang.
\newblock Multi-armed bandit experimental design: Online decision-making and
  adaptive inference.
\newblock In \emph{Proceedings of The 26th International Conference on
  Artificial Intelligence and Statistics}, pages 3086--3097. PMLR, 25--27 Apr
  2023.

\bibitem[Thompson(1933)]{thompson1933likelihood}
William~R Thompson.
\newblock On the likelihood that one unknown probability exceeds another in
  view of the evidence of two samples.
\newblock \emph{Biometrika}, 25\penalty0 (3/4):\penalty0 285--294, 1933.

\bibitem[Villar et~al.(2015)Villar, Bowden, and Wason]{villar2015multi}
Sof{\'\i}a~S Villar, Jack Bowden, and James Wason.
\newblock Multi-armed bandit models for the optimal design of clinical trials:
  benefits and challenges.
\newblock \emph{Statistical science: a review journal of the Institute of
  Mathematical Statistics}, 30\penalty0 (2):\penalty0 199, 2015.

\bibitem[Wald(1947)]{Wald:1947}
Abraham Wald.
\newblock \emph{Sequential Analysis}.
\newblock John Wiley and Sons, 1st edition, 1947.

\bibitem[Zhang and Ying(2023)]{zhang2023fast}
Qining Zhang and Lei Ying.
\newblock Fast and regret optimal best arm identification: Fundamental limits
  and low-complexity algorithms.
\newblock In \emph{Thirty-seventh Conference on Neural Information Processing
  Systems}, 2023.

\bibitem[Zhao(2023)]{zhao2023adaptive}
Jinglong Zhao.
\newblock Adaptive neyman allocation, 2023.

\bibitem[Zhong et~al.(2023)Zhong, Cheung, and Tan]{zhong2023achieving}
Zixin Zhong, Wang~Chi Cheung, and Vincent Tan.
\newblock Achieving the pareto frontier of regret minimization and best arm
  identification in multi-armed bandits.
\newblock \emph{Transactions on Machine Learning Research}, 2023.

\end{thebibliography}
